\documentclass[hidelinks,onefignum,onetabnum]{siamart250211}

\newcommand{\regretent}{\texttt{Regret-Min}$+$\texttt{Entropy}\xspace}

\newcommand{\regretlhalf}{\texttt{Regret-Min}$+\ell_{1/2}$\xspace}

\newcommand{\regregretent}{\texttt{Regularized-Regret-Min}$+$\texttt{Entropy}\xspace}

\newcommand{\regregretlhalf}{\texttt{Regularized-Regret-Min}$+\ell_{1/2}$\xspace}

\newcommand{\alphaacc}{\alpha_{\text{acc}}^*}
\newcommand{\alphaobj}{\alpha_{\text{obj}}^*}

\renewcommand{\b}[1]{{\bf #1}}

\newcommand{\X}{\b{X}}
\newcommand{\PT}{\b{P^\top}}

\newcommand{\BS}{\b{S}}

\newcommand{\XX}{\b{\tilde{X}}}
\newcommand{\XXT}{\b{\tilde{X}^\top}}

\newcommand{\A}{\b{A}}
\newcommand{\Aint}{\b{\tilde{A}}}

\newcommand{\F}{\b{F}}
\newcommand{\Y}{\b{Y}}

\newcommand{\U}{\b{U}}

\newcommand{\xx}{\tilde{\b x}}
\newcommand{\xxt}{\tilde{\b x}^\top}
\newcommand{\x}{{\b x}}
\newcommand{\y}{{\b y}}

\newcommand{\xt}{{\b x}^\top}
\newcommand{\pistar}{\pi_i^*}
\newcommand{\xnorm}{\|\x_i \|_2^2}
\newcommand{\xxnorm}{\|\xx_i \|_2^2}

\newcommand{\sumin}{\sum_{i=1}^n}

\newcommand{\w}{\b{w}}

\newcommand{\what}{\widehat{\w}}
\newcommand{\yhat}{\widehat{y}}

\newcommand{\bpi}{\bm \pi}
\newcommand{\bwi}{\bm \omega}
\newcommand{\bepsilon}{\bm \epsilon}
\newcommand{\btheta}{\bm \theta}
\newcommand{\emptheta}{\widehat{\bm \theta}}

\newcommand{\blambda}{\bm \Lambda}

\newcommand{\bH}{\b H}

\DeclareMathOperator{\E}{\mathbb{E}}

\DeclareMathOperator{\Tr}{Tr}

\newcommand{\simplex}{\Delta_{d}}
\DeclareMathOperator{\sigmas}{\b\Sigma}

\DeclareMathOperator*{\argmin}{arg\,min}
\DeclareMathOperator*{\argmax}{arg\,max}

\newcommand{\entropyfactor}{\frac{\alpha \xnorm}{1 - \exp(-\alpha \xnorm) }}

\definecolor{mygreen}{RGB}{0,153,0}
\definecolor{light-gray}{gray}{0.93}
\definecolor{mid-gray}{gray}{0.88}

\DeclareMathOperator{\bSigma}{\b\Sigma}
\DeclareMathOperator{\B}{\b B}

\newcommand{\V}{\b V}
\newcommand{\Z}{\b Z}
\newcommand{\bLambda}{\b \Lambda}
\newcommand{\tD}{\widetilde{\b D}}

\newcommand{\optsigma}{\bSigma_{\diamond}}

\makeatletter
\newcommand*{\rom}[1]{\expandafter\@slowromancap\romannumeral #1@}
\makeatother

\usepackage{lipsum}
\usepackage{amsfonts}
\usepackage{graphicx}
\usepackage{epstopdf}
\usepackage{algorithmic}
\ifpdf
  \DeclareGraphicsExtensions{.eps,.pdf,.png,.jpg}
\else
  \DeclareGraphicsExtensions{.eps}
\fi

\usepackage{enumitem}
\setlist[enumerate]{leftmargin=.5in}
\setlist[itemize]{leftmargin=.5in}

\newsiamremark{remark}{Remark}
\newsiamremark{hypothesis}{Hypothesis}
\crefname{hypothesis}{Hypothesis}{Hypotheses}
\newsiamthm{claim}{Claim}
\newsiamremark{fact}{Fact}
\crefname{fact}{Fact}{Facts}

\headers{Extensions of regret-minimization for optimal design}{Youguang Chen and George Biros}

\title{Extensions of the regret-minimization algorithm for optimal design}

\author{Youguang Chen
\and George Biros\thanks{Oden Institute of Computational Engineering and Science, The University of Texas at Austin, 
  (\email{youguang@utexas.edu}, \email{gbiros@acm.org}).}}

\usepackage{amsopn}
\DeclareMathOperator{\diag}{diag}

\usepackage{url}            
\usepackage{booktabs}   
\usepackage{tikz} 
\usepackage{multirow}
\usepackage{bm}
\usepackage{amsmath}

\usepackage{array}
\usepackage{comment}
\newcolumntype{C}{>{\centering\arraybackslash}p{3.1cm}}%
\newcolumntype{D}{>{\centering\arraybackslash}p{1.0cm}}%
\newcolumntype{F}{>{\centering\arraybackslash}p{1.0cm}}%
\newcolumntype{P}{>{\centering\arraybackslash}p{9cm}}%
\newcolumntype{E}{>{\centering\arraybackslash}p{1.5cm}}%
\newcolumntype{G}{>{\centering\arraybackslash}p{1.6cm}}%
\newtheorem{assumption}{Assumption}
\usepackage{amssymb}
\usepackage{enumitem}
\usepackage{colortbl}
\usepackage{xspace}  

\usepackage{thmtools,thm-restate}

\ifpdf
\hypersetup{
  pdftitle={An Example Article},
  pdfauthor={D. Doe, P. T. Frank, and J. E. Smith}
}
\fi

\begin{document}
\sloppy

\maketitle

\begingroup
\renewcommand\thefootnote{}
\footnotetext{Accepted for publication in SIAM Journal on Matrix Analysis and Applications.}
\addtocounter{footnote}{-1}
\endgroup

\begin{abstract}
We consider the problem of selecting a subset of points from a dataset of $n$ unlabeled examples for labeling, with the goal of training a multiclass classifier. To address this, we build upon the regret minimization framework introduced by Allen-Zhu et al. in "Near-optimal design of experiments via regret minimization" (ICML, 2017). We propose an alternative regularization scheme within this framework, which leads to a new sample selection objective along with a provable sample complexity bound that guarantees a $(1+\epsilon)$-approximate solution. Additionally, we extend the regret minimization approach to handle experimental design in the ridge regression setting. We evaluate the selected samples using logistic regression and compare performance against several state-of-the-art methods. Our empirical results on MNIST, CIFAR-10, and a 50-class subset of ImageNet demonstrate that our method consistently outperforms competing approaches across most scenarios.
\end{abstract}




\section{Introduction}

{\color{black}
Supervised learning has driven major advances in areas such as computer vision, natural language processing, and healthcare. These successes, however, depend on access to large labeled datasets, and in many domains labeling is the primary bottleneck. In medical imaging, for example, expert annotations from radiologists are costly and often require multiple labels per image to assess inter-rater variability~\cite{menze2014multimodal}. Similar challenges arise in scientific data collection and other specialized fields, where obtaining reliable labels demands scarce expert effort. In such cases, we cannot afford to label all available data, and must instead decide \textit{which samples should be labeled first.} Two natural scenarios arise in this context:
\begin{itemize}
    \item \textbf{Active learning:} a model is trained in rounds, and at each step, the learner adaptively queries the most informative samples. This interactive approach has been widely studied and can substantially reduce labeling costs.
    \item \textbf{One-shot selection from a large unlabeled pool:} in other cases, one must commit upfront to labeling a small subset of data, without the ability to adapt later. This situation arises when annotation must be batched, when retraining is too costly, or when an initial seed set is required to bootstrap active learning.
\end{itemize}

This paper focuses on the second scenario: selecting a representative subset from a large pool of unlabeled data. While less flexible than active learning, this problem is both practically important and theoretically tractable. Moreover, the quality of the initial subset can strongly influence the effectiveness of downstream active/semi-supervised learning. {\color{black}For example, FixMatch~\cite{sohn2020fixmatch}, a state-of-the-art semi-supervised learning algorithm, achieves test accuracy ranging from 40\% to 80\% in CIFAR-10~\cite{cifar10} when trained with just one labeled sample per class, depending solely on which samples are selected (see Figure 7 in~\cite{sohn2020fixmatch}).} This stark variability highlights the need for principled strategies to guide subset selection.

To study this question systematically, we turn to statistical learning theory. {\color{black}For multiclass logistic regression, the excess risk of the maximum likelihood estimator is both upper and lower bounded by the Fisher Information Ratio (FIR), as shown in Theorem 3 of~\cite{firal}. The FIR depends on both the selected samples and the classifier parameters. In \Cref{sec:regression}, we relax this dependency by replacing the FIR with the V-optimal design objective (defined in \cref{eq:v-design}), which depends only on the selected samples. This observation motivates our approach: treating sample selection as an instance of optimal experimental design, especially in settings where labels and well-trained classifiers are not yet available.}
}

{\color{black}Building on this motivation, we formalize the optimal experimental design problem in \cref{def:design}}, where the objective $f$ in \cref{eq:design} is derived from the statistical efficiency of linear regression models. Interpretations of these optimality criteria are discussed further in \Cref{sec:linear-regression} and \Cref{sec:design-meaning}.

\begin{definition}[Experimental design problem]\label{def:design}
Let $\X \in \mathbb{R}^{n \times d}$ denote a pool of $n$ data points and let $k \geq d$ be the number of points to be selected. Define $\X_S \in \mathbb{R}^{k \times d}$ as the matrix consisting of the $k$ selected points indexed by $S$. Let $f: \mathbb{R}^{d \times d} \rightarrow \mathbb{R}$ be a function that evaluates the quality or optimality of $\X_S$. The goal of experimental design is to choose a subset $S^*$ such that
\begin{align}\label{eq:design}
    S^* = \argmin_{S\subseteq[n], |S| = k} f(\X_S^\top \X_S).
\end{align}
\end{definition}

This combinatorial optimization problem is computationally challenging: for many commonly used objectives $f$, solving \cref{eq:design} exactly is NP-hard~\cite{CIVRIL20094801,vcerny2012two}. To overcome this, Zhu et al.~\cite{design} proposed a tractable two-step algorithm—known as \textbf{Regret-Min} approach—that provides near-optimal solutions for a wide class of functions $f$ satisfying the following assumptions: 
\begin{assumption}\label{assume:f}
{\color{black}Let $\mathbb{S}_{++}^d$ denote the cone of $d \times d$ symmetric positive definite matrices.  
We assume that the optimality criterion $f:\mathbb{S}_{++}^d \to \mathbb{R}$ satisfies:  }
\begin{enumerate}
    \item \textbf{Convexity:} for all $\A,\B \in \mathbb{S}_{++}^d$ and $t \in [0,1]$,  $f(t\A + (1-t)\B) \leq t f(\A) + (1-t) f(\B)$.
    \item \textbf{Monotonicity:} if $\A,\B \in \mathbb{S}_{++}^d$ with $\A \preceq \B$, then  
$f(\A) \geq f(\B)$.
    \item \textbf{Reciprocal sub-linearity:} for any $\A \in \mathbb{S}_{++}^d$ and $t \in (0,1)$, $ f(t\A) \leq t^{-1} f(\A)$.
\end{enumerate}
\end{assumption}


The Regret-Min algorithm consists of two stages. First, it relaxes the discrete optimization problem into a continuous one. Then, it sparsifies the resulting solution using a greedy sample selection strategy based on the Follow-the-Regularized-Leader (FTRL) regret minimization framework.

{\color{black}
The choice of regularizer is important for FTRL in regret minimization. Two prominent regularizers have been studied extensively: the $\ell_{1/2}$-regularizer, defined as $w(\A) = -2\Tr(\A^{1/2})$, and the entropy regularizer $w(\A) = \Tr(\A \log \A - \A)$, where $\A$ is symmetric positive definite. The original Regret-Min method~\cite{design} employs the $\ell_{1/2}$-regularizer, while the entropy regularizer is widely used in the matrix multiplicative weights method~\cite{arora2012multiplicative}.  In the context of regret minimization, the $\ell_{1/2}$-regularizer achieves a smaller regret width at the expense of a larger diameter term~\cite{Zhu-2015}, making it preferable when the width term dominates the regret bound.

However, sample selection in Regret-Min fundamentally differs from regret minimization, despite both utilizing the FTRL framework. In regret minimization, we have \textit{no control} over the loss matrix $\F_t$ (formally defined later in \Cref{sec:rounding})  at each step $t$---it is determined by the environment and revealed after we select the action matrix $\A_t$. The objective is to choose $\A_t$ to minimize the regret upper bound across arbitrary loss sequences. Conversely, in sample selection for Regret-Min, we \textit{do control} $\F_t$ through our sample choices. Here, the goal is to select $\F_t$ to maximize $\lambda_{\min}\left(\sum_{t=1}^k \F_t\right)$, the minimum eigenvalue of the cumulative loss matrix sum.

This fundamental difference implies that the theoretical advantages of the $\ell_{1/2}$-regularizer in regret minimization do not necessarily transfer to the sample selection task. This observation naturally raises the question: can we instead employ the entropy regularizer for sample selection in Regret-Min? And if so, how does it compare to the $\ell_{1/2}$-regularizer, both theoretically and empirically?
}

In many practical settings—such as when features are correlated, the sample size is small, or overfitting is a concern—ridge regression is preferred over ordinary least squares. In this case, the optimal design objective in \cref{eq:design} becomes $f(\X_S^\top \X_S + \lambda \b I)$, where $\lambda>0 $ is a regularization parameter. We explore whether the regret minimization framework can be extended to handle this ridge regression scenario, and whether comparable performance guarantees can be established.

We evaluate our method on synthetic datasets as well as multiclass image classification datasets. In the classification experiments, we consider the regime where the number of selected samples $k$ is only a small multiple of the number of classes. We assess both regularizers for Regret-Min and further compare Regret-Min with a variety of other sample selection approaches.
\subsection{Our contributions}
Our contributions can be summarized as follows:

\begin{itemize}
\item \textbf{Relaxing FIR bounds with V-optimal design (\cref{prop:risk-log}):} We show that the excess risk of multiclass logistic regression can be bounded in terms of the V-optimal design objective in \Cref{sec:regression}.
\item \textbf{Entropy-based regret minimization for optimal design (\cref{algo:rounding}, \cref{thm:design-complexity}):} We incorporate the unnormalized negentropy (entropy-regularizer) into the Regret-Min framework. Our algorithm achieves an \(\epsilon\)-approximate guarantee with sample complexity \(\widetilde{\mathcal{O}}(d/\epsilon^2)\), matching the guarantee of the \(\ell_{1/2}\)-regularizer. We also derive a tighter, sample-dependent bound that improves to \(\widetilde{\mathcal{O}}(d/\epsilon)\) in favorable cases.
\item \textbf{Extension to ridge regression (\cref{algo:rounding-regularize}, \cref{thm:reg-design-complexity}):} We adapt the regret minimization framework to handle regularized optimal design problems for both entropy and \(\ell_{1/2}\) regularizers. We prove that the sample complexity remains \(\widetilde{\mathcal{O}}(d/\epsilon^2)\), matching that of the unregularized case.  Notably, adapting the framework and establishing these guarantees in the presence of regularization requires overcoming several non-trivial technical challenges.
\item \textbf{Empirical validation (\Cref{sec:experiments}):} We benchmark Regret-Min against baseline methods (listed in \cref{fig:diagram}) on synthetic and real-world datasets (MNIST, CIFAR-10, ImageNet-50). Our experiments demonstrate consistently better performance of Regret-Min and show that the entropy-regularizer achieves more reliable objective-classification accuracy alignment than the $\ell_{1/2}$-regularizer.
\end{itemize}

\begin{figure}
    \centering
    \footnotesize
\begin{tikzpicture}
\node[] (a) at (0,0) {Methods for selecting representative samples};
\node[] (b) at (-5,-1) {Uniform};
\draw[] (a) -- (b);
\node[] (c) at (-3, -1) {K-Means};
\draw[] (a) -- (c);
\node[] (d) at (-1.5, -1) {RRQR};
\draw[] (a) -- (d);
\node[] (e) at (0, -1) {MMD};
\draw[] (a) -- (e);
\node[] (f) at (3, -1) {Optimal design};
\draw[] (a) -- (f);
\node[] (e) at (4, -1.8) {Greedy};
\draw[] (f) -- (e);
\node[] (e2) at (0, -1.8) {Relaxation-based};
\draw[] (f) -- (e2);
\node[] (e3) at (-2.5, -2.5) {Max-weights};
\draw[] (e2) -- (e3);
\node[] (e4) at (0, -2.5) {Weighted-sampling};
\draw[] (e2) -- (e4);
\node[] (e5) at (2.5, -2.5) {Regret-Min};
\draw[] (e2) -- (e5);
\end{tikzpicture}    
\caption{Methods for selecting representative samples without label information.}
    \label{fig:diagram}
\end{figure}

\subsection{Notation}

Let $\X \in \mathbb{R}^{n \times d}$ denote a pool of $n$ points in $d$ dimensions, and let $k$ be the sample selection budget. {\color{black}We write $[n]:= \{1,2,\cdots,n\}$. We write $f\lesssim g$ or $f = \mathcal{O}(g)$ to represent that $f(\cdot) \leq C g(\cdot)$ for any admissible arguments of $f(\cdot)$ and $g(\cdot)$ and some constant $C>0$;  analogously for $f\gtrsim g$. We use $\widetilde{\mathcal{O}}(\cdot)$ to suppress logarithmic factors.}
For matrices $\A, \B$, their inner product is denoted by $\langle \A, \B \rangle$. Denote by $\mathbb{S}_{+}^d$ the cone of $d \times d$ symmetric positive semidefinite (PSD) matrices, and by $\mathbb{S}_{++}^d$ the cone of symmetric positive definite (PD) matrices. For symmetric $\A, \B \in \mathbb{R}^{d \times d}$, we use $\A \preceq \B$ to mean $\b v^\top(\B-\A)\b v \geq 0$ for all $\b v \in \mathbb{R}^d$. The symbol $\otimes$ denotes the Kronecker product between two matrices.

{\color{black} We follow the formalism of sub-Gaussian random variables as in \cite{vershynin2018high}.  
For a sub-Gaussian random variable $\xi \in \mathbb{R}$, the $\psi_2$-norm $\|\xi\|_{\psi_2}$ can be defined in several equivalent ways (see Proposition 2.6.1 in \cite{vershynin2018high}).  
In this paper, we adopt the definition $\|\xi\|_{\psi_2} := \inf \left\{ K > 0 : \; \mathbb{E}\exp(\xi^2/K^2)\leq 2 \right\}$.
This extends naturally to sub-Gaussian random vectors $\x \in \mathbb{R}^d$ by $\|\x\|_{\psi_2} := \sup_{\|\boldsymbol{\theta}\|_2 \leq 1} \|\langle \x, \boldsymbol{\theta}\rangle\|_{\psi_2}$, 
i.e., the maximal $\psi_2$-norm over all one-dimensional marginals of $\x$.}

\subsection{Organization of the Paper}
The remainder of the paper is organized as follows. In \Cref{sec:regression}, we present non-asymptotic results that connect the V-optimal design objective to the excess risk in linear and logistic regression models. These results motivate our use of experimental design as a guiding principle for sample selection, though they are not directly tied to the algorithmic contributions of the paper. Readers primarily interested in the algorithmic aspects may wish to skip ahead to \Cref{sec:opt-design}, which introduces our regret minimization framework for solving the experimental design problem using either the $\ell_{1/2}$ or entropy regularizer. We extend this approach to the experimental design in ridge regression setting in \Cref{sec:reg}. Finally, \Cref{sec:experiments} presents our empirical results: we first compare the two regularizers on synthetic data, and then evaluate our method on real-world multiclass classification tasks, comparing it with other sample selection strategies listed in \cref{fig:diagram}.

\section{Optimal Experimental design with linear and logistic regression models}\label{sec:regression}

In this section, we present a finite-sample analysis linking optimal experimental design to the excess risk (generalization error) of linear and logistic regression models. Under sub-Gaussian assumptions, we establish non-asymptotic high-probability bounds (\cref{prop:risk-linear,prop:risk-log}), demonstrating that the V-optimal design objective provides both upper and lower bounds on the excess risk for these models. We conclude the section with a discussion of \cref{prop:risk-log}, summarized in two remarks. The proofs of this section are presented in \Cref{sec:appendix-pf-regression}.

We assume that the sample matrix $\X\in\mathbb{R}^{n\times d}$ consists of $n$ i.i.d. samples drawn from a sub-Gaussian distribution $p(\x)$ (see \cref{assume:px}). 
\begin{assumption}\label{assume:px}
    Let $p(\x)$ be a sub-Gaussian distribution for $\x \in{\color{black}\mathbb{R}^d}$. We also assume that the covariance matrix of $p(\x)$ denoted by $\V_p:= \E[\x \x^\top]$ is positive definite. 
\end{assumption}
Let $S\subset [n]$ denote the index set  of $k>d$ selected samples, and let $\X_S\in\mathbb{R}^{k\times d}$ be the corresponding submatrix. We further assume that both $\X^\top \X$ and $\X_S^\top \X_S$ are invertible. The V-optimal design objective is defined by
\begin{align}\label{eq:v-design}
    f_V(\X_S) :=\left \langle \left(\frac{1}{k}\X_S^\top \X_S\right)^{-1}, \frac{1}{n}\X^\top\X \right\rangle = \frac{k}{n}\left \langle\left(\X_S^\top \X_S\right)^{-1}, \X^\top \X \right\rangle,
\end{align}
where $\frac{1}{k} \X_S^\top \X_S$ is the average covariance matrix of selected samples, $\frac{1}{n} \X^\top \X$ is the average covariance matrix of all samples.

\subsection{Linear regression}\label{sec:linear-regression}
For linear regression model, we assume that given a sample $\x$, the observation $\mathbb{R}\ni y= \w_*^\top \x + \epsilon$, where $\epsilon$ is the noise with $\E[\epsilon]=0$ and variance $\E[\epsilon^2]=\sigma^2$. Given samples $\X_S$ and observations $\y_S =\X_S \w_*  + \bepsilon_S$, the estimate of the weights is defined by
\begin{align}\label{eq:ls-what}
    \what \in \argmin_{\w \in\mathbb{R}^d} \|\y_S - \X_S \w \|_2^2 = (\X_S^\top \X_S)^{-1} \X_S^\top \y_S.
\end{align}
The excess risk is defined as the expected error of the prediction using estimated weight $\what$ compared to the ground truth weight $\w$, i.e.
\begin{align}\label{eq:ls-lp-def}
    R_p(\what) := \E_{\epsilon_S, \x}[(\w_*^\top \x-\what^\top \x)^2].
\end{align}

We prove that the V-optimal design objective both upper and lower bounds the excess risks in \cref{prop:risk-linear}. Consequently, selecting samples $\X_S$ that achieve a low V-optimal objective value can effectively reduce excess risk.

\begin{proposition}\label{prop:risk-linear}
    Assume that \cref{assume:px} holds, let $K>0$ be the constant s.t. $\| \V_p^{-1} \x_i\|_{\psi_2} = K$. Then for any $\gamma, \delta \in(0,1)$, if 
    \begin{align}\label{eq:ls-nbound}
        n {\color{black}\gtrsim \max\left\{\frac{4 d}{\gamma^2},\frac{4\log(2/\delta)}{ \gamma^2} \right\},}
    \end{align}
    with probability at least $1-\delta$, we have
    \begin{align}\label{eq:risk-linear}
       \frac{\sigma^2}{(1+\gamma)k} f_V(\X_S)\leq R_p(\what) \leq \frac{\sigma^2}{(1-\gamma)k} f_V(\X_S).
    \end{align}
\end{proposition}


\subsection{Multi-class logistic regression} 
\begin{figure}[t]
    \centering
\begin{tikzpicture}
    \node[] at (0,0)  {\includegraphics[width=5cm]{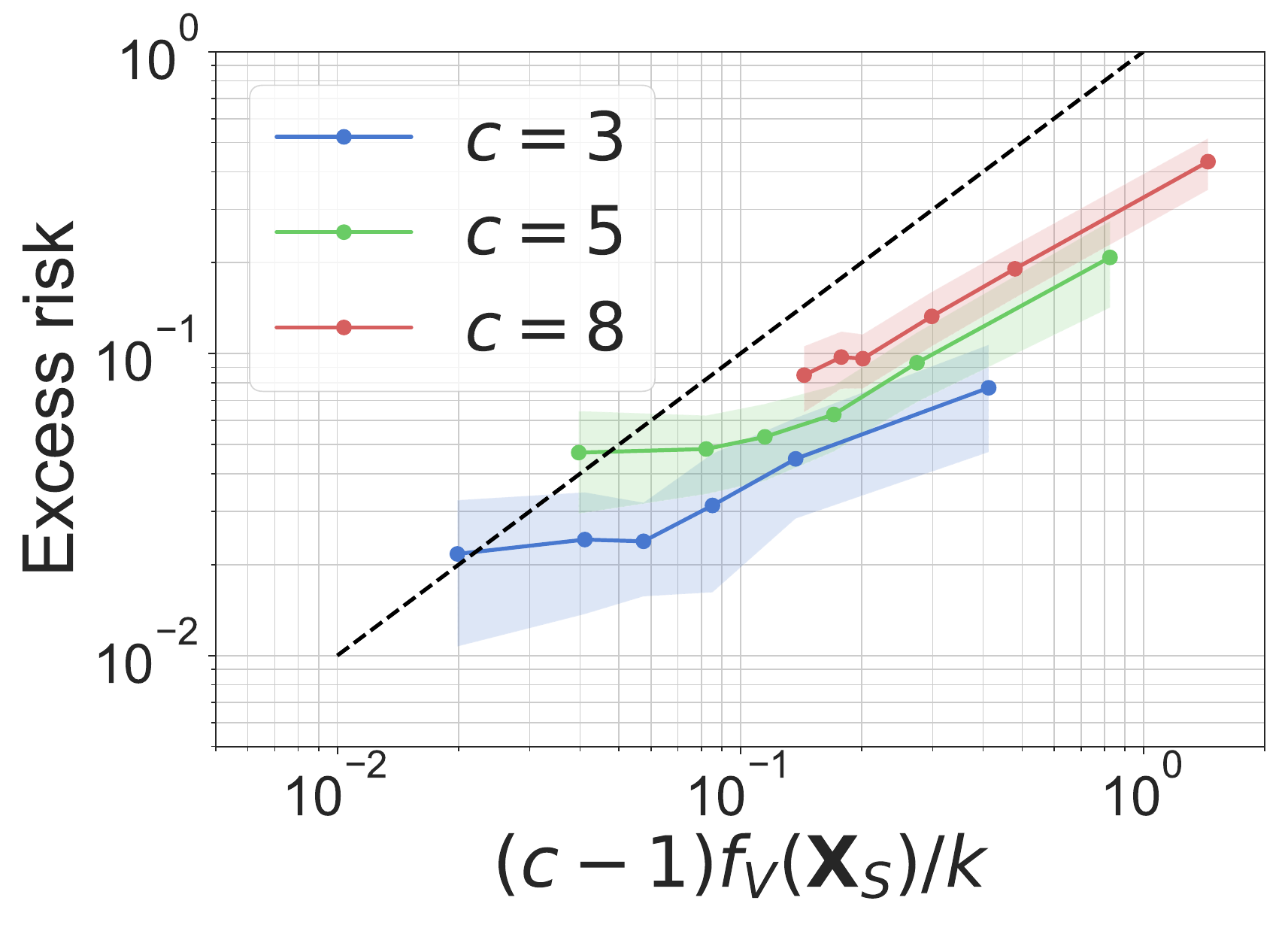}};
   \node[] at (6,0)  {\includegraphics[width=5cm]{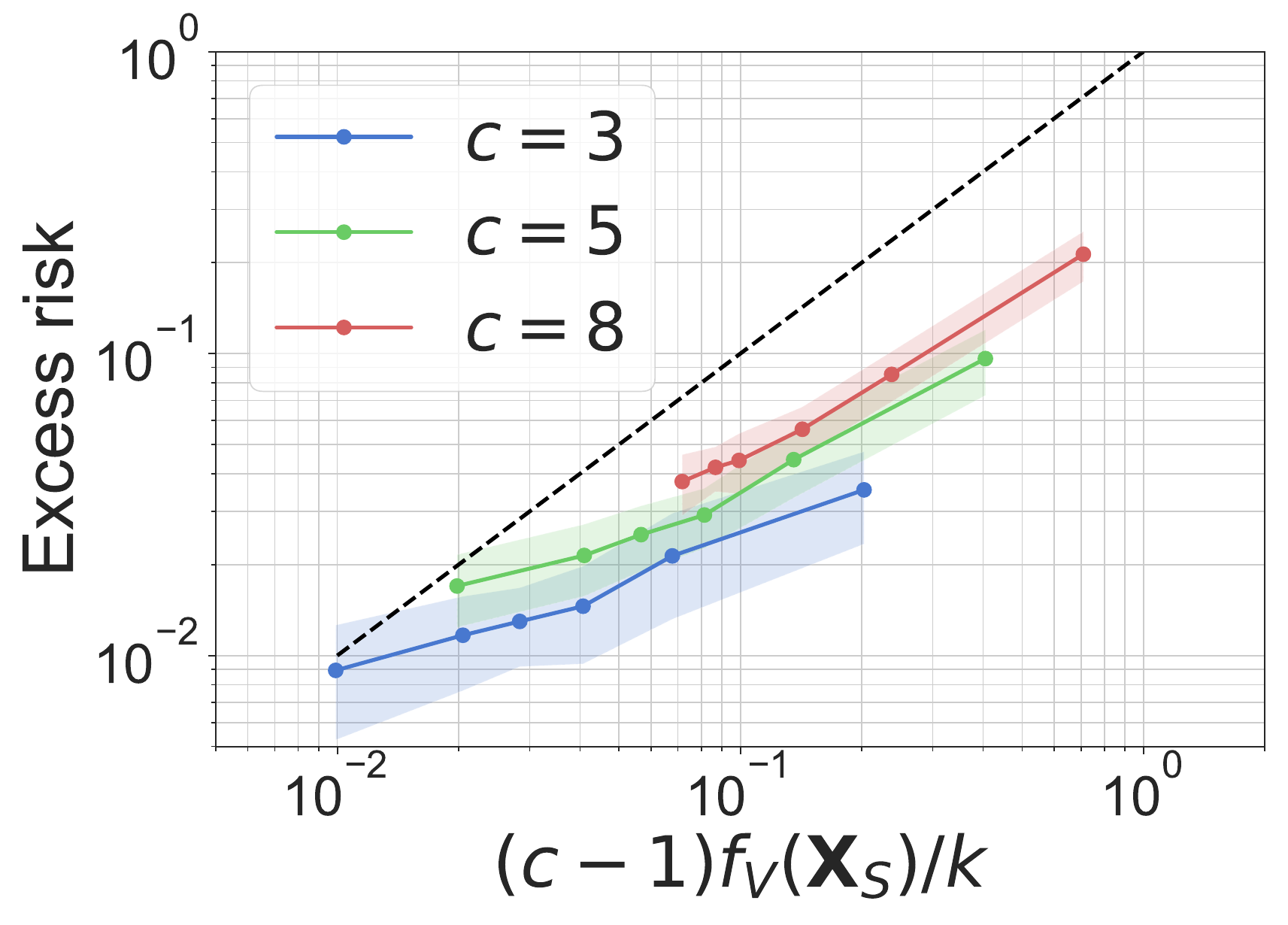}};
   \node[] at (.5,2.) {\small$k=400$};
   \node[] at (6.5,2.) {\small $k=800$};
\end{tikzpicture}
    \caption{Excess risk of logistic regression on synthetic Gaussian data with input dimension 8. Each point in the plots represents a selected subset of $k = |S|$ samples. The parameter $c$ in each plot indicates the number of classes. The quantity $f_V(\X_S)$ denotes the V-optimality objective as defined in \cref{eq:v-design}. {\color{black}Black dashed lines are included as references for the linear relationship.} }
    \label{fig:risk-synthetic}
\end{figure}

We denote a labeled sample as a pair $(\x,y)$, where $\x\in\mathbb{R}^d$ is a data point, $y \in \{1,2,\cdots, c \}$ is its label, and $c$ is the number of classes. Let $\btheta\in\mathbb{R}^{(c-1) \times d}$ be the parameters of a $c$-class logistic regression classifier. Given $\x$ and $\btheta$, the likelihood of the label $y$ is defined by
\begin{align}\label{eq:setup-conditional}
    p(y|\x,\btheta) = \begin{cases}
   \frac{\exp(\btheta_y^\top \x)}{1 + \sum_{l\in[c-1]} \exp(\btheta_l^\top \x)} ,\qquad y \in [c-1]\\
    \frac{1}{1 + \sum_{l\in[c-1]} \exp(\btheta_l^\top \x)},\qquad y = c.
    \end{cases}
\end{align}

We assume there is a ground truth $\btheta_*$ such that true label $y\sim p(y|\x, \btheta)$.  We use the negative log-likelihood as the loss function $\ell_{(\x,y)}(\btheta) := -\log p(y|\x,\btheta)$. The expected loss at $\btheta$ is defined by
\begin{align}\label{eq:setup-generalization-error}
    L_p(\btheta) := \E_{\x\sim p(\x)} \E_{y\sim  p(y|\x,\btheta_*)}[\ell_{(\x,y)}(\btheta)].
\end{align}

 Let $\emptheta$ be the empirical risk minimizer (ERM) for data points in $\X_S$ ($\emptheta$ is also denoted as M-estimator):
 \begin{align}\label{eq:ERM-q}
    \emptheta\in \argmin_{\theta}  \frac{1}{k}\sum_{\x\in\X_S}\ell_{(x,y)}(\btheta)=: Q_k(\btheta), \qquad \forall \x\in \X_S, \quad y\sim p(y|\x, \btheta_*).
 \end{align}
The excess risk of $\emptheta$ is defined as the error of the expected loss compared to the ground truth $\theta_*$, i.e. $L_p(\emptheta) - L_p(\btheta_*)$. We define the following Hessians w.r.t $\btheta$:
\begin{align}
    \bH_p:= \nabla^2 L_p(\btheta_*), \qquad \bH_k(\btheta):= \nabla^2 \E_{y}[Q_k(\btheta)].
\end{align}
To obtain non-asymptotic bounds on the excess risk, we further assume in \cref{assume:hessian} that the sample Hessian $\bH_k$ is positive definite for all $\btheta$ in a neighborhood of $\btheta_*$. With \cref{assume:px} and \cref{assume:hessian}, we derive high probability bounds for the excess risk in \cref{prop:risk-log}. 

\begin{assumption}\label{assume:hessian}
There is a radius $r\gtrsim 1$ such that for any $\theta \in \mathcal{B}_{r}(\theta_*) = \{\theta: \|\theta-\theta_* \|_{2,\infty} \leq r \}$, $\bH_k(\theta)$ is positive definite, where $\|\cdot\|_{2,\infty}$ denotes the maximum row norm of a matrix.
\end{assumption}

\begin{proposition}\label{prop:risk-log}
    Assume that \cref{assume:px} and \cref{assume:hessian} hold. Denote $\widetilde{c}:=c-1$. Let $\rho_1^+, \rho_1^-, \rho_2^+, \rho_2^->1$ be constants such that $ \rho_1^-\bH_p\preceq \b I_{\widetilde{c}} \otimes \V_p \preceq \rho_1 ^+\bH_p$ and $\rho_2^- \bH_k(\btheta_*)\preceq\b I_{\widetilde{c}} \otimes \left( \frac{1}{k} \X_S^\top \X_S\right) \preceq \rho_2^+ \bH_k(\btheta_*) $. Then for any $\gamma, \delta \in(0,1)$, if 
    \begin{align}\label{eq:complexity-risklog}
    k\gtrsim \widetilde{c} d\log(ed/\delta)\quad \mathrm{and}\quad n \gtrsim  d/\gamma^2,
    \end{align}
    we have with probability at least $1-\delta$,
    \begin{samepage}
    \begin{align}\label{eq:risk-log}
  \frac{e^{-\alpha} + \alpha -1}{\alpha^2}\frac{\rho_2^-}{ \rho_1^+(1+\gamma) }\frac{\widetilde{c}f_V(\X_S)}{k}  
  &\lesssim \E[L_p(\emptheta)] - L_p(\btheta_*) \nonumber\\
  &\lesssim   \frac{e^{\alpha} - \alpha -1}{ \alpha^2} \frac{ \rho_2^+}{\rho_1^-(1-\gamma)}\frac{\widetilde{c}f_V(\X_S)}{k},
    \end{align}
    \end{samepage}
where $\alpha=\mathcal{O}\left(\sqrt{\rho_1\widetilde{c} d\log(e/\delta)/k}\right)$ and $\E$ is expectation over $\{y\sim p(y|\x, \btheta_*)\}_{\x \in \X_S}$.
\end{proposition}

\begin{remark}\label{remark:risk-log}
The bounds in \cref{prop:risk-log} are derived based on Theorem 3 from \cite{firal}, which shows that the excess risk in logistic regression is both upper and lower bounded by the Fisher Information Ratio (FIR), given by $\langle \bH_k(\btheta_*)^{-1}, \bH_p(\theta_*)\rangle$.  The FIR depends on the true parameter $\btheta_*$, the data distribution $p(x)$ and the selected sample subset $\X_S$. In contrast to Theorem 3 in \cite{firal}, our bounds in \cref{eq:risk-log} are \textit{looser} in that they approximate the FIR using the constants $\rho_2^-$, $\rho_2^+$, and the V-optimal design objective  $f_V(\X_S)$. An alternative interpretation arises when  $\btheta_*=\b 0$—that is, when labels are uniformly distributed across all classes for every input. This scenario corresponds to the setting where no label information is available. Under this condition, and assuming the full dataset $\X$ is used to empirically approximate the distribution, the FIR simplifies exactly to the V-optimal design objective. 
\end{remark}

We further illustrate the relationship between excess risk and V-optimal design objective $f_V(\X_S)$ using a Gaussian synthetic dataset. As shown in \cref{fig:risk-synthetic}, the results indicate a strong correlation between the excess risk and $f_V(\X_S)$.

\begin{remark}
From \cref{eq:risk-log}, we observe that the V-optimal design objective $f_V(\X_S)$ appears in both the upper and lower bounds of the excess risk. \textit{This motivates selecting the subset $\X_S$ by minimizing $f_V(\X_S)$ in classification tasks where label information is unavailable.} In fact, many optimal design criteria (see \cref{table:design-criteria}) focus on selecting samples that yield well-conditioned sample covariance matrices $\X_S^\top\X_S$. A key advantage of the regret minimization approach discussed in \Cref{sec:opt-design} is its applicability to a broad class of design objectives that satisfy \cref{assume:f}. The effectiveness of the optimal design criteria may vary depending on the specific application and dataset characteristics, which we further explore in our experiments (see \Cref{sec:experiments}).
\end{remark}

\section{Optimal experimental design via regret minimization}\label{sec:opt-design}

\begin{table}[t]
\small
    \centering
\caption{Objectives of (A-, D-, E-, V-, G-) optimal design. $\X \in\mathbb{R}^{n\times d}$ is the pool of design points and $\X_S\in\mathbb{R}^{n\times k}$ is the  $k$ selected points. Assume that $\bSigma:= \X_S^\top\X_S$ is positive definite matrix. }
  \label{table:design-criteria}
    \begin{tabular}{ccccc}
     \toprule
       A  &D &E &V & G\\\midrule
$\frac{1}{d} \Tr (\sigmas^{-1})$    &$-\frac{1}{d}( \log \det \sigmas)$  & $\|\sigmas^{-1} \|_2$  & $\frac{1}{n}\langle \bSigma^{-1}, \X^\top \X\rangle$ &$\max \text{diag}( {{ \X} \sigmas^{-1} \X^\top})$\\
\bottomrule
    \end{tabular}
\end{table}

Suppose the sample selection budget is $k \geq \color{black}{d}$. The goal of optimal experimental design is to select a subset of points that minimizes a given optimality criterion $f:\mathbb{S}_{+}^d\rightarrow \mathbb{R}$ over the covariance matrix of the selected subset, i.e.,
   \begin{align}\label{eq:ip}
        {\b s}^* = \argmin_{{\b s} = (s_1, \dots, s_n)} f\left(\sum_{i=1}^n s_i \x_i \xt_i \right), \qquad s.t.\quad \|\b s\|_1= k, \quad\b{s} \in \{0,1,\dots, k\}^n.
   \end{align}

{\color{black}For simplicity, we consider the with-replacement setting throughout our analysis.} The optimality criterion $f$ in \cref{eq:ip} typically represents a measure of statistical efficiency. We denote $f^*$ as the optimal value of the objective, i.e. $f^* \triangleq f(\sum_{i\in[n]} s_i^* \x_i \x_i^\top)$. We assume that $f$ satisfies the properties outlined in  \cref{assume:f}, which are met by many commonly used optimality criteria.  \Cref{table:design-criteria} summarizes the definitions of several popular criteria, and \Cref{sec:design-meaning} provides a brief review of their statistical interpretation.

{\color{black}We adopt a two-step strategy, following the framework of~\cite{design}, to obtain a near-optimal solution to problem~\cref{eq:ip}. In the first step, we solve a relaxed version of~\cref{eq:ip}; in the second step, we convert this relaxed solution into a valid one via regret minimization. While the prior work~\cite{design} employs the $\ell_{1/2}$-regularizer in the rounding step, our analysis extends this approach to the entropy-regularizer. For clarity, we refer to Regret-Min with the entropy-regularizer as \regretent, and Regret-Min with the $\ell_{1/2}$-regularizer as \regretlhalf.

All corresponding proofs of the results in this section are deferred to \Cref{sec:appendix-pf-design}.
}

\subsection{Relaxed Problem}\label{sec:relax}

We first relax the constraint of \cref{eq:ip} on ${\b s}$ from integer values into real values:
    \begin{align}\label{eq:lp}
        \bpi^\diamond = \argmin_{\bpi = (\pi_1, \pi_2, \dots, \pi_n)} f\left(\sum_{i=1}^n \pi_i x_i x_i^\top\right),\qquad s.t.\quad \| \bpi \|_1= k, \quad\bpi \geq \b 0.
    \end{align}
Since the constraint in~\cref{eq:ip} is a subset of the constraint in~\cref{eq:lp}, we have the following lemma.
\begin{lemma}\label{lm:ip-lp-relation}
$f^\diamond \triangleq f\left(\sum_{i=1}^n \pi_i^\diamond \x_i \xt_i \right) \leq f(\sum_{i=1}^n s_i^* \x_i \xt_i) \triangleq f^*$.
\end{lemma}

Since both objective and constraint of the relaxed problem \cref{eq:lp} are convex, it can be easily solved by conventional convex optimization algorithms. Note that the constraint of~\cref{eq:lp} is a scaling of the probability simplex, if we define $\bwi = \bpi/k$, \cref{eq:lp} can be reformulated as
\begin{align}\label{eq:lp-2}
    \bwi^* = \argmin_{\bwi = (\omega_1, \omega_2, \dots, \omega_n)} f\left(\sum_{i=1}^n k \omega_i x_i x_i^\top\right),\qquad s.t.\quad \| \bwi \|_1= 1, \quad\bwi \geq \b 0.
\end{align}
We can efficiently solve such convex optimization problem by entropic mirror descent~\cite{beck2003mirror}). In particular, starting with $\bwi^1$ as an interior point in the probability simplex, let $\gamma_t$ be the step size at step $t$, the update of $\bwi$ is given by
\begin{align}
    \omega_i^{t+1} = \frac{\omega_i^{t} \exp(-\gamma_t f^\prime_i(\bwi^t))}{\sum_{j=1}^n \omega_j^{t} \exp(-\gamma_t f^\prime_j(\bwi^t))} ,
\end{align}
where  $f^\prime(\bwi^t) = ({\color{black}f^\prime_1(\bwi^t)},\cdots, f^\prime_n(\bwi^t))^\top\in \partial f(\bwi^t)$. The {\color{black}gradient}/subgradient of the design criteria  can be easily obtained by~\cref{table:design-criteria} by chain rule. The rate of convergence using entropic mirror descent is $\mathcal{O}\left(\sqrt{\frac{\log n}{T}}\right)$ (Theorem 5.1 in \cite{beck2003mirror}).


\subsection{Rounding via regret minimization}\label{sec:rounding}
After solving the relaxed problem~\cref{eq:lp}, we need to round the solution $\bpi^\diamond$ into an integer solution to~\cref{eq:ip}. Now we introduce how to use regret minimization approach to achieve this goal with performance guarantee. First, we make some definitions for the ease of discussion. 
\begin{align}\label{eq:transforms}
    \bSigma_\diamond\triangleq \sum_{i\in[n]} \pi_i^\diamond \x_i \x_i^\top , \qquad \xx_i \triangleq \bSigma_{\diamond}^{-1/2} \x_i,\qquad \XX  \triangleq \X \bSigma_{\diamond}^{-1/2},
\end{align}
where $\bSigma_\diamond$ is the covariance matrix with the weights as the optimal relaxed solution, $\xx_i$ is the denoised data point, $\XX\in \mathbb{R}^{n\times d}$ is the denoised sample matrix stacked by the denoised data points as rows.
\paragraph{Goal of rounding} Let $\b s$ be a feasible solution to~\cref{eq:ip}, $\b S \triangleq \mathrm{diag}(\b s)$, then $\X^\top \b S \X $ is the covariance matrix of the selected samples, i.e. $\X^\top \b S \X = \sum_{i\in[n]} s_i \x_i \x_i^\top$. We have the following observation. 

\begin{proposition}\label{prop:goal}
    Given $\tau \in (0,1)$, if $ \lambda_{\min}\left(\XX^\top \b S \XX\right) \geq \tau$, then
    \begin{align}
        f\left(\X^\top \b S \X\right) \leq \tau^{-1} f^*.
    \end{align}
\end{proposition}
\begin{proof}
    $\lambda_{\min}\left(\XX^\top \b S \XX\right) \geq \tau$ is equivalent to $\XX^\top \b S \XX\succeq \tau \b I$.
    
    By~\cref{eq:transforms}, we have $\bSigma_\diamond^{-1/2} \left(\X^\top \b S \X\right) \bSigma_\diamond^{-1/2}\succeq \tau \b I$, and thus $\X^\top \b S \X\succeq \tau \bSigma_\diamond$. Then
    \begin{align}
          f\left(\X^\top \b S \X\right) &\leq f\left(\tau \bSigma_\diamond\right) \tag*{(item 2 of \cref{assume:f}) }\nonumber \\
          &\leq \tau^{-1} f\left(\bSigma_\diamond\right) = \tau^{-1} f^\diamond
          \tag*{(item 3 of \cref{assume:f}) } \nonumber\\
          &\leq \tau^{-1} f^*      \tag*{(\cref{lm:ip-lp-relation}) } 
    \end{align}
\end{proof}

From \cref{prop:goal}, a larger value of $\tau$ might indicate that $f\left(\X^\top \b S \X\right)$ is closer to the optimal value $f^*$. Thus our rounding goal is to select points such that $\lambda_{\min}\left(\XX^\top \b S \XX\right)$ is maximized.

\paragraph{Regret minimization}
We now introduce the regret minimization problem of the adversarial linear bandit problem. Consider a game of $k$ rounds. At each step $t\in[k]$:
\begin{itemize}
    \item the player chooses an action matrix $\A_t\in\simplex$, where 
    \begin{align}\label{eq:matrix-simplex}
        \simplex \triangleq {\color{black}\left \{\A\in \mathbb{S}_{+}^{d}: \Tr(\A) = 1\right\}.}
    \end{align}
    \item then the environment reveals a loss matrix $\F_t \in  \mathbb{S}^d_{+}$.
    \item a loss $\langle \A_t, \F_t \rangle$ is incurred for round $t$.
\end{itemize}
Note that the loss matrix $\F_t$ is always determined \textbf{after} the action $\A_t$ is chosen at round $t$. After $k$ rounds, the total regret of such a game is defined by
    \begin{align}\label{eq:regret}
        R :=\sum_{t=1}^k \left\langle\A_t , \F_t \right\rangle -  \inf_{\U \in \simplex} \left\langle \U, \sum_{t=1}^k \F_t \right\rangle.
       \end{align}
The goal is to find $\A_t$ such that the total regret can be minimized.


\paragraph{Choosing the action matrix $\A_t$ via FTRL}  
The Follow-the-Regularized-Leader (FTRL) method~\cite{shalev2012online} provides a standard approach for selecting the action matrix $\A$ within the regret minimization framework introduced earlier. The procedure is parameterized by a regularizer $w: \simplex \to \mathbb{R}$ together with a learning rate $\alpha > 0$. In our analysis, we focus on two specific choices of regularizers:  
\begin{itemize}
    \item \textbf{Entropy-regularizer:} $w(\A) = \langle \A, \log \A - \b I \rangle$.  
    \item \textbf{$\ell_{1/2}$-regularizer:} $w(\A) = -2 \Tr(\A^{1/2})$.
\end{itemize}
{\color{black}While prior work~\cite{design} developed a sample selection algorithm based on the $\ell_{1/2}$-regularizer, here we extend this framework to the entropy-regularizer.  
}

FTRL  chooses $\A_1 = \argmin_{\A \in \simplex} w(\A)$ for the first round. In subsequent rounds $t\geq 2$, the actions are determined by
\begin{align}
        \A_{t} = \argmin_{\A \in \simplex} \left\{ \alpha \sum_{s=1}^{t-1}\left\langle \A,\F_s \right\rangle + w(\A)\right\}.\label{eq:ftrl}
\end{align}

We obtain closed-form expressions of $\A_t$ under both the entropy and $\ell_{1/2}$ regularizers in~\cref{prop:expressions}. {\color{black}Using these action matrices determined by FTRL, we then derive the regret bounds stated in~\cref{prop:regret-bound}.}

\begin{proposition}[Closed forms of $\A_t$ by FTRL]\label{prop:expressions}
 FTRL chooses action $\A_1 = \frac{1}{d} {\b I}$ for the first round, For the subsequent rounds $t\geq 2$, the actions are
\begin{align}
    \A_t &= \exp\left(-\alpha\sum_{s=1}^{t-1} \F_s - \nu_t {\b I}\right),\qquad (\text{entropy-regularizer})\label{eq:At-ent}\\
    \A_t &= \left( \alpha\sum_{s=1}^{t-1} \F_s + \nu_t {\b I} \right)^{-2}, \qquad\qquad (\ell_{1/2}\text{-regularizer; \color{black}\cite[Eq.~(10)]{design}}) \label{eq:At-l12}
\end{align}
where $\nu_t$ is the unique constant that ensures $\A_t \in \simplex$.
\end{proposition}

{\color{black}
\begin{proposition}[Regret bound of FTRL]\label{prop:regret-bound}
Let $\alpha > 0$, and suppose that $\{\A_t\}_{t \in [k]}$ are selected by FTRL with either the entropy regularizer or the $\ell_{1/2}$-regularizer. Then the regret defined in~\cref{eq:regret} satisfies
\begin{align}
    R &\leq \frac{\log d}{\alpha} 
        + \sum_{t=1}^k \langle \A_t, \F_t \rangle 
        - \frac{1}{\alpha} \sum_{t=1}^k \Bigl(1 - \langle \A_t, \exp(-\alpha \F_t) \rangle \Bigr),
        \quad\text{(entropy regularizer)} \label{eq:regret-entropy}\\[1ex]
    R &\leq \frac{2\sqrt{d}}{\alpha} 
        + \sum_{t=1}^k \langle \A_t, \F_t \rangle 
        - \frac{1}{\alpha} \sum_{t=1}^k 
          \Tr\!\left[\A_t^{1/2} - \bigl(\A_t^{-1/2} + \alpha \F_t \bigr)^{-1}\right],
         \quad\text{($\ell_{1/2}$ regularizer)}. \label{eq:regret-l12}
\end{align}
\end{proposition}
}

\paragraph{Sample selection}
{\color{black}
As discussed at the beginning of \Cref{sec:rounding}, our goal in sample selection is to maximize 
$\lambda_{\min}\!\left(\XX^\top \b S \XX\right)$. At each step $t$, we restrict the loss matrix to be of the form 
$\F_t \in \{\xx_i \xx_i^\top\}_{i \in [n]}$. Defining $\F := \sum_{t=1}^k \F_t$, we observe that 
$\lambda_{\min}\!\left(\XX^\top \b S \XX\right) = \lambda_{\min}(\F)$.  

By \Cref{lm:mtd-eigen-min}, the regret in \cref{eq:regret} admits an equivalent form for regret:
\begin{align}\label{eq:Regret-def-min}
    R = \sum_{t=1}^k \left\langle \A_t , \F_t \right\rangle - \lambda_{\min}(\F).
\end{align}
Substituting \eqref{eq:Regret-def-min} into the regret bounds in \cref{eq:regret-entropy,eq:regret-l12} yields lower bounds for $\lambda_{\min}(\F)$. Further derivations lead to \Cref{thm:min-eigen}: inequality~\eqref{eq:ent-bound} gives the bound for the entropy regularizer, while inequality~\eqref{eq:l12-bound} gives the bound for the $\ell_{1/2}$ regularizer. Complete proofs are deferred to \Cref{sec:proof-min-eigen}.  

Since the objective now is to maximize $\lambda_{\min}(\F)$, at each step $t$ we select a sample $i_t \in [n]$ that maximizes the corresponding lower bound. At this stage, the action matrices $\{\A_s\}_{s \in [t]}$ and past loss matrices $\{\F_s\}_{s \in [t-1]}$ are already determined. Substituting $\F_t = \xx_{i_t}\xx_{i_t}^\top$ into \eqref{eq:ent-bound} and \eqref{eq:l12-bound} yields the sample selection objectives in \eqref{eq:ent-sample} and \eqref{eq:l12-sample} for the entropy- and $\ell_{1/2}$-regularizer, respectively.
}
\begin{restatable}{theorem}{thmmineigen}{\normalfont(Lower bound of $\lambda_{\min}\left(\sum_{t\in[k]} \F_t\right)$ by FTRL).}
\label{thm:min-eigen}
Suppose that at each round $t$, the loss matrix is rank-one, $\F_t = x x^\top$ for some nonzero $x \in \mathbb{R}^d$, and define $\F = \sum_{t=1}^k \F_t$. Then
    \begin{align}
      & \lambda_{\min}(\F)
      \geq -\frac{\log d}{\alpha}+ \frac{1}{\alpha} \sum_{t=1}^k \left[1-\exp\left(-\alpha \Tr\left(\F_t\right)\right)\right] \frac{\left\langle \A_t, \F_t \right\rangle}{\Tr(\F_t)},\quad (\text{entropy-regularizer})\label{eq:ent-bound}\\
   & \lambda_{\min}(\F)
      \geq -\frac{2 \sqrt{d}}{\alpha} +  \sum_{t=1}^k \frac{\left\langle\A_t, \F_t \right\rangle}{1 + \left\langle \A_t^{1/2}, \alpha \F_t\right\rangle} .\quad (\text{$\ell_{1/2}$-regularizer; \color{black}\cite[Eq.~(11)]{design}})\label{eq:l12-bound}
    \end{align}
\end{restatable}

\begin{definition}[Sample selection objective]\label{def:opt-sample}
Given sample size $k\geq d$, regularizer type, and learning rate $\alpha>0$, we choose $\A_t$ by FTRL at each round ( \cref{prop:expressions}), and choose $\F_t = \xx_{i_t}\xx_{i_t}^\top$, where $i_t$ is the sample index determined by the following: 
\begin{align}
        i_t &=\argmax_{i \in [n]} \left[1 - \exp(-\alpha \| \xx_i\|_2^2)\right] \frac{\xxt_i \A_t \xx_i}{\| \xx_i\|_2^2},\quad (\text{entropy-regularizer}) \label{eq:ent-sample} \\
 i_t &=\argmax_{i \in [n]} \frac{\xxt_i \A_t \xx_i}{1 + \alpha \xxt_i \A_t^{1/2} \xx_i}.\qquad(\text{$\ell_{1/2}$-regularizer; \color{black}\cite[Algorithm 1]{design}})  \label{eq:l12-sample}
    \end{align}
\end{definition}


\begin{figure}[t]
    \centering
\footnotesize
\begin{tikzpicture}
\node[inner sep=0pt] (a1) at (0,0) {\includegraphics[width=6cm]{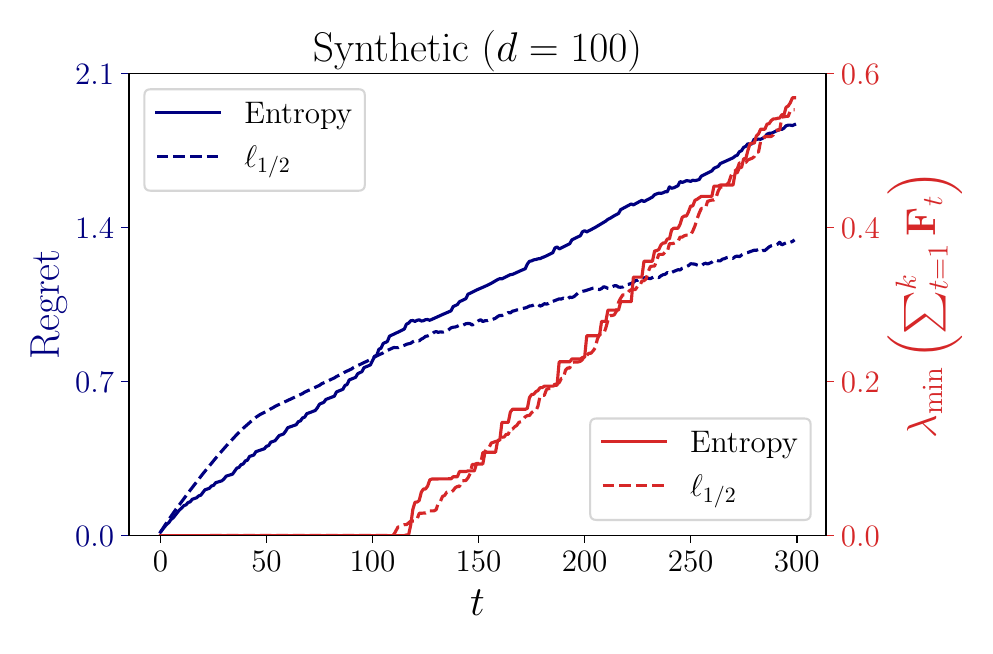}};
\node[inner sep=0pt] (a2) at (6,0) {\includegraphics[width=6cm]{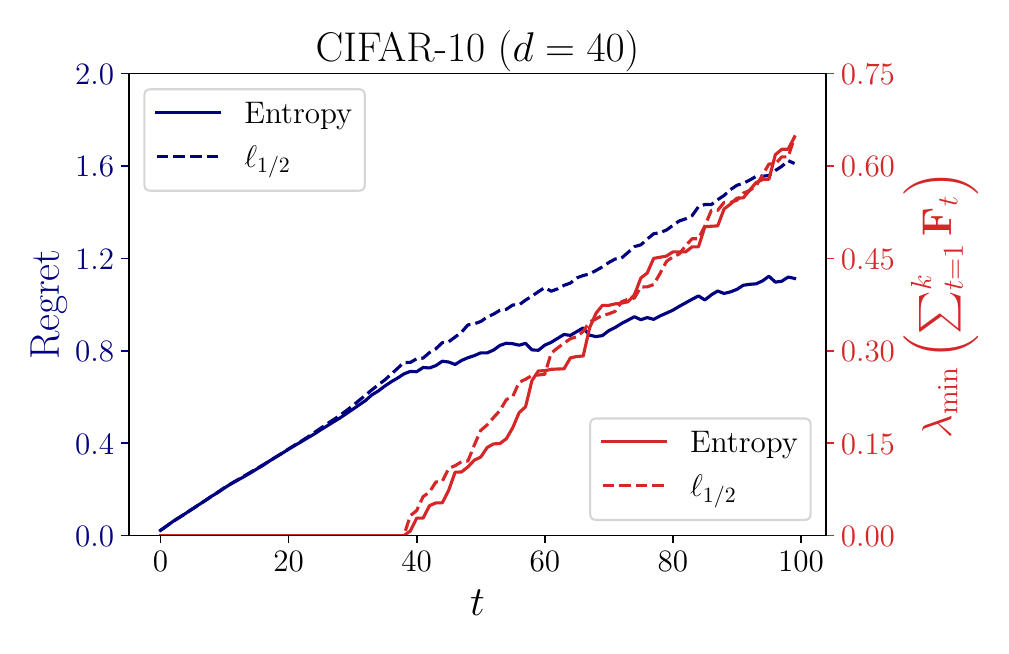}};
\end{tikzpicture}
\caption{\color{black}Regret (blue lines) and $\lambda_{\min}\left(\sum_{t=1}^k \F_t\right)$ (red lines) for \regretent (solid lines) and \regretlhalf (dashed lines). Left: synthetic data with sample selection budget $k=300$; experiment details are provided in \cref{sec:synthetic-experiments}. Right: CIFAR-10 with 40-dimensional PCA-reduced features and sample selection budget $k=100$; experiment details are provided in \cref{sec:realworld-experiments}. Further discussion is provided in \cref{remark:difference}.}
\label{fig:regret}
\end{figure}

{\color{black}
To clarify the relationship between regret minimization and sample selection in Regret-Min, we provide the following remark.

\begin{remark}\label{remark:difference}
We emphasize that sample selection in Regret-Min and regret minimization are fundamentally different tasks. In regret minimization, we do \textit{not} have control over the loss matrix $\F_t$: the loss matrix is determined by the environment and is revealed \textit{after} we choose the action matrix $\A_t$ at each step $t$. The goal is to optimize the selection of $\A_t$ to minimize the regret upper bound across arbitrary loss sequences. 

In contrast, sample selection for Regret-Min operates under different conditions: here we \textit{do} control $\F_t$ through the samples we select (as defined in \cref{def:opt-sample}), and the goal is to maximize $\lambda_{\min}\left(\sum_{t=1}^k \F_t\right)$ by choosing appropriate $\F_t$.

This distinction is illustrated in \cref{fig:regret}, which shows both the regret $R$ and $\lambda_{\min}\left(\sum_{t=1}^k \F_t\right)$ during sample selection for two experiments from \cref{sec:experiments}. Notably, while the regret behavior differs between methods and datasets---\regretent exhibits higher regret than \regretlhalf after $t \geq 100$ on synthetic data but lower regret on CIFAR-10, with differences enlarging over time---both methods achieve very similar values of $\lambda_{\min}\left(\sum_{t=1}^k \F_t\right)$ throughout the selection process.
\end{remark}
}

\subsection{Algorithm and complexity}\label{sec:algorithm}

We summarize the complete algorithm in \cref{algo:rounding}. For each time step of solving the relaxed problem (Lines 1 in \cref{algo:rounding}), computing $\bSigma$ has a complexity of $\mathcal{O}(nd^2)$, its eigendecomposition requires $\mathcal{O}(d^3)$. The complexity of computing the gradient or subgradient depends on the choice of the optimality function $f$. For example, in the case of A-design, the complexity is $\mathcal{O}(nd^2)$ once the eigendecomposition of $\bSigma$ is available. 
 For each step of solving the rounding problem (Lines 5-19 in \cref{algo:rounding}),  choosing sample $i_t$ (line 6 or 8) has a complexity of  $\mathcal{O}(nd^2)$, updating action matrix $\A_t$ requires $\mathcal{O}(d^3)$. Thus the total complexity of the  algorithm is $\mathcal{O}(kd^3 + knd^2).$

\begin{algorithm}[t]
\caption{{\color{black}\texttt{Regret-Min}$+$\texttt{Entropy}/$\ell_{1/2}$}}
\label{algo:rounding}
 \hspace*{\algorithmicindent} \textbf{Input:} \nolinebreak
$\X \in \mathbb{R}^{n\times d}$, budget $k \geq d$,  learning rate $\alpha>0$, design objective $f$, regularizer $w\in\{\mathrm{``entropy"}, ``\ell_{1/2}" \}$ \\
 \hspace*{\algorithmicindent} \textbf{Output:} \nolinebreak
sample indices set $S$
\begin{algorithmic}[1]
\STATE $\bpi^*\gets$ solution of \cref{eq:lp} with certain design objective $f$
\STATE $\XX\gets \X (\X^\top \mathrm{diag}(\bpi^*)\X)^{-1/2}$
\STATE $\A_1 = \frac{1}{d} \b I$
\FOR {$t = 1$ to $k$}
    \IF{$w$ is ``entropy''}
        \STATE $i_t \gets \argmax_{i \in [n]\backslash S} \left[1 - \exp(-\alpha \| \xx_i\|_2^2)\right] \frac{\xxt_i \A_t \xx_i}{\| \xx_i\|_2^2}.$
    \ELSE
        \STATE $ i_t \gets \argmax_{i \in [n]\backslash S} \frac{\xxt_i \A_t \xx_i}{1 + \alpha \xxt_i \A_t^{1/2} \xx_i}.$
    \ENDIF
    \STATE $S \gets S \cup \{i_t\}$
    \STATE eigendecomposition: $\alpha \sum_{i\in S}\xx_i \xxt_i = \b V \b \Lambda \b V^\top$
    \IF{$w$ is ``entropy''}
        \STATE $\nu_t \gets \log [ \sum_{i=1}^d \exp(-\b \lambda_i) ]$\footnotemark
        \STATE $\b \Lambda^\prime \gets \exp(-\b\Lambda -\nu_t \b I)$
    \ELSE
        \STATE update $\nu_t$ such that $\sum_{i=1}^d (\lambda_i + \nu_t)^{-2} = 1$
        \STATE $\b \Lambda^\prime \gets [\b\Lambda +\nu_t \b I]^{-2}$   
    \ENDIF
    \STATE $\A_t \gets \b V \b \Lambda^\prime \b V^\top$
\ENDFOR
\end{algorithmic}
\end{algorithm}
\footnotetext{For stable computation, $\nu_t \gets -\lambda_{\min} + \log \!\left(\sum_{i=1}^d \exp\!\big(-(\lambda_i - \lambda_{\min})\big)\right)$,where $\lambda_{\min} =\min_i \lambda_i$.}

\subsection{Performance guarantee}\label{sec:performance}
We present the performance guarantee of \cref{algo:rounding} in \cref{thm:design-complexity} by providing the sample complexity such that the selected samples can generate near-optimal results. The complete proof is given in \Cref{sec:pf-complexity}.

\begin{restatable}{theorem}{thmcomplexity}\label{thm:design-complexity}
Given sample matrix $\X \in\mathbb{R}^{n\times d}$ and sample selection budget $k\geq d$. We select samples by \cref{algo:rounding} with learning rate $\alpha$ and either entropy or  $ \ell_{1/2}$ regularizer $w$  with objective $f$ satisfying \cref{assume:f}. Let $\epsilon\in(0,1)$, suppose that one of the following condition is satisfied:

\begin{enumerate}[label=(\alph*)]
    \item $w$ is entropy-regularizer, $\alpha = (4\log d)/\epsilon$ and $k\geq 16d\log d/\epsilon^2$.
    \item $w$ is entropy-regularizer, $\alpha =(2 \log d)/\epsilon$ and  $k\geq C_1 (d\log d)/\epsilon$, where
    \begin{align}
        C_1 = 2\left(1- \exp\left(-\alpha\min_{t\in[k]}\|\x_{i_t}\|_{\bSigma_{\diamond}^{-1}}^2\right)\right)^{-1},\label{eq:const-entropy}
    \end{align}
   where $\bSigma_\diamond$ is defined in \cref{eq:transforms} and $\{i_t\}_{t\in[k]}$ are indices of selected samples.
    \item {\color{black}(\cite[Theorem 1.1]{design})} $w$ is $\ell_{1/2}$-regularizer, $\alpha = 8\sqrt{d}/\epsilon$, and $k\geq 32 d/\epsilon^2$.
\end{enumerate}
Then the samples selected by \cref{algo:rounding} achieve near-optimal performance, i.e.
\begin{align}
    f\left(\X_S^\top \X_S\right) \leq (1+\epsilon) f^*,
\end{align}
where $S$ is the sample indices set, $\X_S\in\mathbb{R}^{k\times d}$ is the selected sample matrix, and $f^*$ is the optimal value of design objective in \cref{eq:ip}.
\end{restatable}

{\color{black}
\begin{remark}
For \regretent, \cref{thm:design-complexity} provides two alternative sample complexity bounds. Condition (a) gives a deterministic bound of $k \geq 16 d \log d / \epsilon^2$, while condition (b) achieves the tighter bound $k \geq C_1 d \log d / \epsilon$ at the cost of introducing the data-dependent constant $C_1$ defined in \cref{eq:const-entropy}.

To evaluate the practical implications of condition (b), we analyze $C_1$ values across 38 experiments from \Cref{sec:experiments}, covering both synthetic and real-world classification datasets. As shown in \cref{fig:C1-discuss}, we observe that $C_1$ typically falls within the range $2.5 \leq C_1 \leq 3.5$.

The selection behavior of \regretent provides insight into why these $C_1$ values arise. According to the sample selection objective in \cref{eq:ent-sample}, among samples with similar values of $\xx^\top \A_t \xx/\|\xx\|^2$, \regretent prefers the sample with the larger $\|\xx\|^2$, where $\|\xx\|^2 = \|x\|_{\bSigma_{\diamond}^{-1}}^2$ by the definition in \cref{eq:transforms}. This preference is validated experimentally: the lower two plots in \cref{fig:C1-discuss} demonstrate that the $\|\xx\|^2$ values of samples selected by \regretent consistently rank within the top 20\% of the entire dataset, which explains why $C_1$ remains relatively small in practice.

For comparison, \regretlhalf achieves the complexity bound $k \geq 32 d / \epsilon^2$ under condition (c). Condition (b) offers two advantages over condition (c). First, the constant $C_1$ is substantially smaller than 32---for instance, with $C_1 \leq 3.2$, condition (b) yields a strictly smaller bound than condition (c) for any $\epsilon \in (0,1)$ when $d < 20,000$. Second, the dependence on precision is linear in $1/\epsilon$ rather than quadratic in $1/\epsilon^2$. This linear dependence makes condition (b) particularly favorable in high-precision regimes where $\epsilon$ is small. 
\end{remark}
}

\begin{figure}[t]
    \centering
\footnotesize
\begin{tikzpicture}
\node[inner sep=0pt] (a1) at (0,0) {\includegraphics[width=5.2cm]{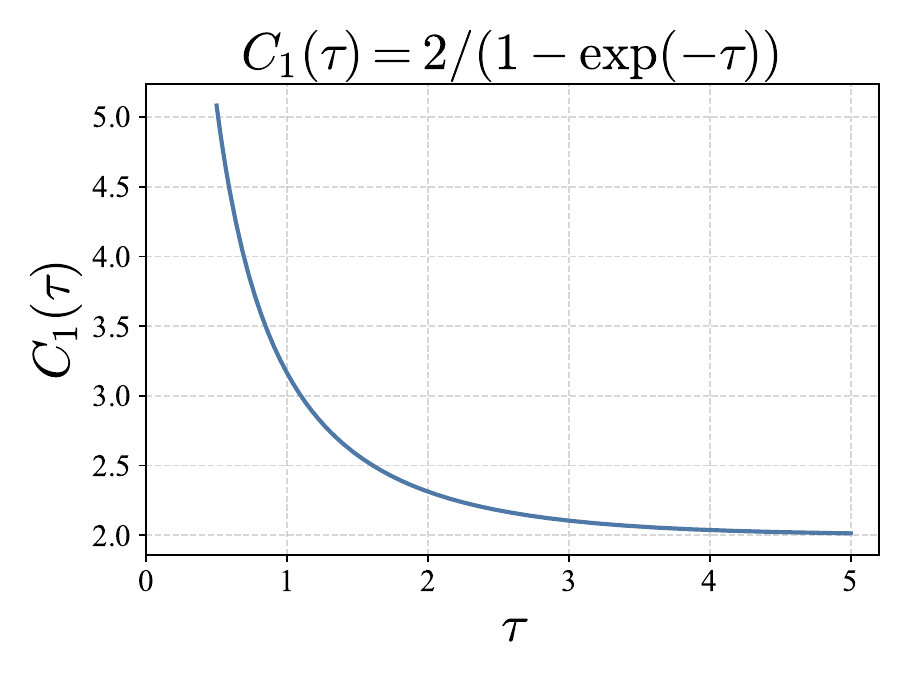}};
\node[inner sep=0pt] (a2) at (5.5,0) {\includegraphics[width=6cm]{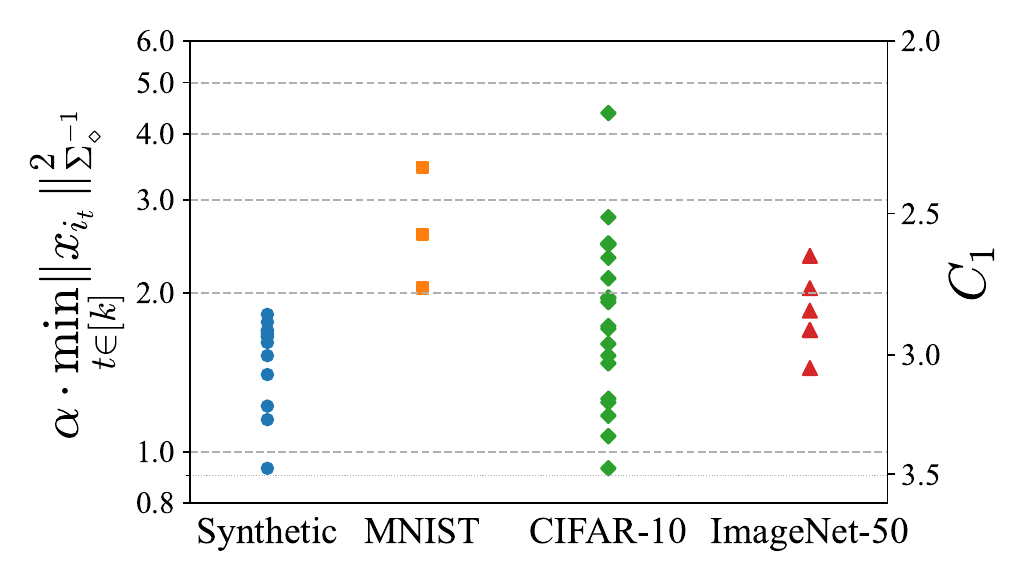}};
\node[inner sep=0pt] (a3) at (0,-3.55) {\includegraphics[width=5.5cm]{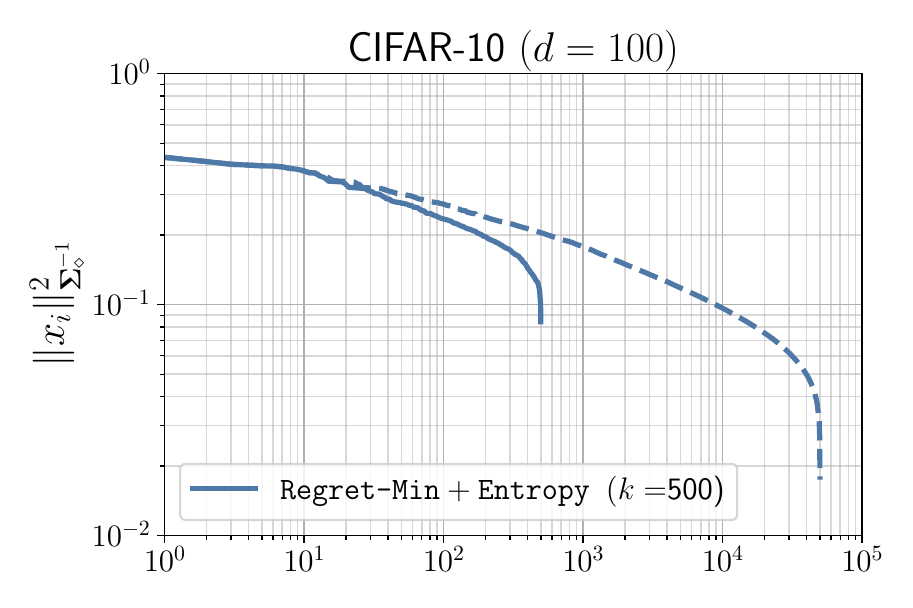}};
\node[inner sep=0pt] (a4) at (5.3,-3.55) {\includegraphics[width=5.5cm]{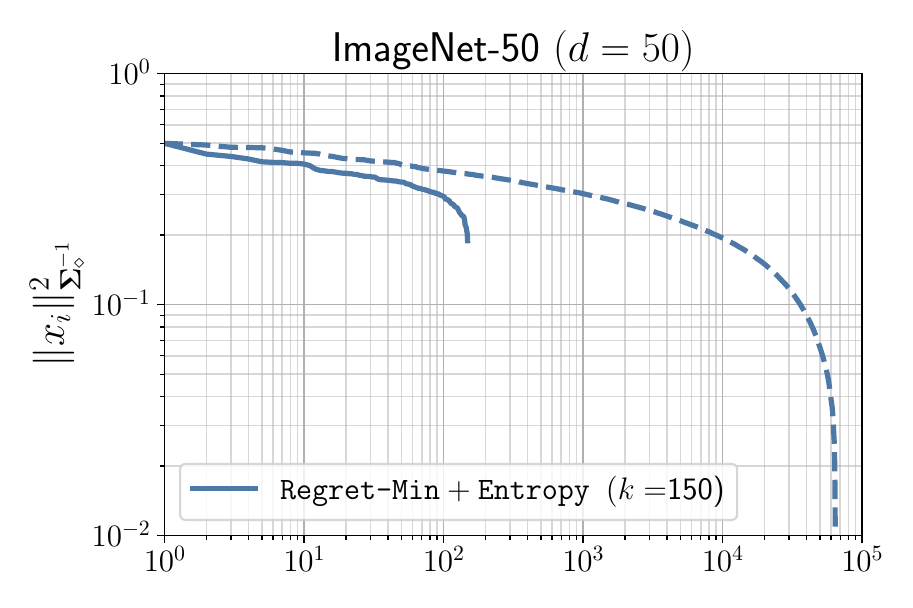}};
\end{tikzpicture}
\caption{\color{black}\regretent: Upper left: Plot of the function $C_1(\tau) = 2/(1-\exp(-\tau))$. Upper right: Values of $\alpha\min_{t\in[k]}\|x_{i_t}\|_{\bSigma_\diamond^{-1}}^2$ and the corresponding $C_1$ values (as defined in~\cref{eq:const-entropy}) for the experiments in \Cref{sec:experiments} using \regretent. Points correspond to the synthetic dataset (upper row of \cref{fig:design-synthetic}) and to the MNIST, CIFAR-10, and ImageNet-50 experiments in \cref{table:mnist,table:cifar10-logistic,table:image50-100}, respectively. Lower left (right): Values of $\|x\|_{\bSigma_{\diamond}^{-1}}^2$ for CIFAR-10 (ImageNet-50). Dashed lines represent the full dataset, while solid lines indicate samples selected by \regretent\ with $k=500$ ($k=150$).}
\label{fig:C1-discuss}
\end{figure}

\section{Regularized optimal experimental design} \label{sec:reg}

In many instances where regularized least squares regression (i.e., ridge regression) is used, the associated optimal experimental design problem takes the following form:
\begin{align}\label{eq:reg-ip}
    {\b s}^* = \argmin_{{\b s} = (s_1, \dots, s_n)} f\left(\sum_{i=1}^n s_i \x_i \xt_i + \lambda \b I \right), \qquad s.t.\quad \|\b s\|_1= k, \quad\b{s} \in \{0,1,\dots, k\}^n.
\end{align}
We plan to address this problem using a two-step approach similar to how we solve the unregularized case. First, we solve the relaxed optimization problem:
\begin{align}\label{eq:reg-lp}
     \bpi^\diamond = \argmin_{\bpi = (\pi_1, \pi_2, \dots, \pi_n)} f\left(\sum_{i=1}^n \pi_i x_i x_i^\top + \lambda \b I\right),\qquad s.t.\quad \| \bpi \|_1= k, \quad\bpi \geq \b 0.
\end{align}
After solving the relaxed problem, we aim to convert the relaxed solution into an integer one using the regret minimization approach introduced in~\Cref{sec:rounding}, with some necessary modifications.

\paragraph{Rounding objective}
We denote $\bSigma_\diamond$ as the regularized weighted covariance, using the weights $\pi^\diamond$ obtained from the solution of the relaxed problem. Additionally, we denote $\xx_i$ as the denoised feature for point $i$ and $\XX$ as the corresponding denoised matrix of all points. 
\begin{align}\label{eq:reg-transforms}
    \bSigma_\diamond\triangleq \sum_{i\in[n]} \pi_i^\diamond \x_i \x_i^\top + \lambda \b I , \qquad \xx_i \triangleq \bSigma_{\diamond}^{-1/2} \x_i,\qquad \XX  \triangleq \X \bSigma_{\diamond}^{-1/2}.
\end{align}

We can establish a sufficient condition for a solution with near-optimal performance, as stated in \cref{prop:reg-goal}. Therefore, our rounding objective is to select samples to maximize $\lambda_{\min}\left(\XX^\top \b S \XX + \lambda \b \bSigma_{\diamond}^{-1}\right) $. We continue to use the FTRL algorithm from \cref{prop:expressions} to determine the action matrices $\A_t$ at each step $t$. However, in the regularized case, the selection of loss matrices must be modified.

\begin{proposition}\label{prop:reg-goal}
Given $\tau \in (0,1)$, if $ \lambda_{\min}\left(\XX^\top \b S \XX + \lambda \b \bSigma_{\diamond}^{-1}\right) \geq \tau$, then
    \begin{align}
        f\left(\X^\top \b S \X + \lambda\b I\right) \leq \tau^{-1} f^*.
    \end{align}
\end{proposition}

\paragraph{Sample selection}
In the unregularized case, we select $\F_t$ as rank-1 matrices. Now, we define $\F_t = \xx_{i_t} \xx_{i_t}^\top + \frac{\lambda}{k} \bSigma_{\diamond}^{-1}$ for some point index $i_t$ to ensure that after $k$ steps, the sum  satisfies $\sum_{t\in[k]} \F_t = \XX^\top \b S \XX + \lambda \bSigma_{\diamond}^{-1}$, where $\b S$ is the diagonal matrix with each diagonal indicating if the sample is selected or not. Similar to \cref{thm:min-eigen}, we can get lower bound for $\lambda_{\min}\left(\sum_{t\in[k]}\F_t\right)$ for each regularizer in \cref{thm:reg-min-eigen}. The proof is detailed in \Cref{sec:pf-reg-min-eigen}.

\begin{restatable}{theorem}{thmlowerboundreg}{\normalfont{(Lower bound of $\lambda_{\min}\left(\sum_{t\in[k]} \F_t\right)$ by FTRL for regularized case)}}.\label{thm:reg-min-eigen}
    Suppose that the loss matrix $\F_t  \in \{\xx_{i} \xx_{i}^\top + \frac{\lambda}{k} \bSigma_{\diamond}^{-1}\}_{i\in[n]}$. Denote $\F = \sum_{k\in[t]} \F_t$, then
\begin{itemize}[leftmargin=*]
    \item for the entropy-regularizer:
    \begin{align}
      & \lambda_{\min}(\F)
      \geq -\frac{\log d}{\alpha} + \frac{1}{\alpha} \sum_{t=1}^k \left\{\Tr\left(\A_t-\B_t\right) + \left[1-\exp\left(-\alpha \|\xx_{i_t}\|_2^2\right)\right] \frac{\xx_{i_t}^\top\B_t \xx_{i_t}}{\| \xx_{i_t}\|_2^2}\right\},\label{eq:reg-ent-bound}\\
         &\log \B_t = \log \A_t- \frac{\alpha\lambda}{k} \bSigma_{\diamond}^{-1} \label{eq:bt-ent}.
    \end{align}
  \item for the $\ell_{1/2}$-regularizer:
  \begin{align}
     & \lambda_{\min}(\F)
      \geq -\frac{2 \sqrt{d}}{\alpha} + \frac{1}{\alpha}\sum_{t=1}^k \left\{ \Tr\left(\A_t^{1/2}-\B_t^{1/2}\right) +  \frac{\alpha\xx_{i_t}^\top \B_t\xx_{i_t}}{1 +\alpha\xx_{i_t}^\top\B_t^{1/2}\xx_{i_t}} \right\},\label{eq:reg-l12-bound}\\
     &  \B_t^{-1/2} = \A_t^{-1/2} + \frac{\alpha \lambda}{k} \bSigma_{\diamond}^{-1} \label{eq:bt-l12}.
  \end{align}
\end{itemize}
\end{restatable}

Based on \cref{thm:reg-min-eigen}, we can select point at each step $t$ to maximize the lower bound of $\lambda_{\min}\left(\sum_{t\in[k]} \F_t\right) $ in \cref{eq:reg-ent-bound} for the entropy-regularizer and \cref{eq:reg-l12-bound} for the $\ell_{1/2}$-regularizer. We conclude the selection strategy in \cref{def:reg-opt-sample}. Our optimal design algorithm for regularized case only needs to replace the lines 5-8 of \cref{algo:rounding} into the formula defined in \cref{def:reg-opt-sample}. Note that we only need to update $\B_t$ once at the beginning of each step $t$. The algorithm is summarized in \cref{algo:rounding-regularize}.  {\color{black}Similar to \cref{algo:rounding},  we refer to Regularized-Regret-Min with the entropy-regularizer as \regregretent, and Regularized-Regret-Min with the $\ell_{1/2}$-regularizer as \regregretlhalf.}

 \begin{definition} [Sample selection for the regularized case]\label{def:reg-opt-sample}
Given sample size $k$, regularizer type, and learning rate $\alpha>0$, we choose $\A_t$ by FTRL at each round ( \cref{prop:expressions}), and choose $\F_t = \xx_{i_t}\xx_{i_t}^\top+\frac{\lambda}{k}\bSigma_\diamond^{-1}$, where $i_t$ is the sample index determined by the following rule: 
\begin{align}
        i_t &=\argmax_{i \in [n]} \left[1 - \exp\left(-\alpha \| \xx_i\|_2^2\right)\right] \frac{\xxt_i \B_t \xx_i}{\| \xx_i\|_2^2},\qquad (\text{entropy-regularizer}) \label{eq:reg-ent-sample} \\
 i_t &=\argmax_{i \in [n]} \frac{\xxt_i \B_t \xx_i}{1 + \alpha \xxt_i \B_t^{1/2} \xx_i},\qquad\qquad\qquad \qquad (\text{$\ell_{1/2}$-regularizer})  \label{eq:reg-l12-sample}
    \end{align}
where $\B_t$ is determined by \cref{eq:bt-ent} for entropy-regularizer and by \cref{eq:bt-l12} for $\ell_{1/2}$-regularizer.
\end{definition}

\paragraph{Performance guarantee}
Compared to the unregularized optimal design, establishing a performance guarantee for the regularized case is significantly more challenging. We begin by deriving bounds for the lower bounds in \cref{eq:reg-ent-bound,eq:reg-l12-bound} at each step $t\in[k]$, as presented in \cref{prop:guarantee-entropy-reg,prop:guarantee-l12-reg}. Finally, we summarize the sample complexity bound to achieve a near-optimal performance guarantee in \cref{thm:reg-design-complexity}. The proofs of \cref{prop:guarantee-entropy-reg,prop:guarantee-l12-reg,thm:reg-design-complexity} are presented in \Cref{sec:appendix-pf-regdesign}.


\begin{proposition}\label{prop:guarantee-entropy-reg}
    For the entropy-regularizer, if $k \geq \alpha d$, at each round $t\in[k]$, the following holds: 
\begin{align}\label{eq:guarantee-entropy-reg}
 \max_{i\in[n]} \frac{1}{\alpha}  \left\{\Tr\left(\A_t-\B_t\right) + \left[1-\exp\left(-\alpha \|\xx_{i}\|_2^2\right)\right] \frac{\xx_{i}^\top\B_t \xx_{i}}{\| \xx_{i}\|_2^2}\right\} \geq \frac{1}{k +\alpha d}.
\end{align}
\end{proposition}

\begin{proposition}\label{prop:guarantee-l12-reg}
For the $\ell_{1/2}$-regularizer, at each round $t\in[k]$, the following holds:
    \begin{align}\label{eq:guarantee-l12-reg}
        \max_{i\in[n]} \frac{1}{\alpha} \left\{\Tr\left(\A_t^{1/2}-\B_t^{1/2}\right) +  \frac{\alpha\xx_{i}^\top \B_t\xx_{i}}{1 +\alpha\xx_{i}^\top\B_t^{1/2}\xx_{i}}\right\} \geq \frac{1 - \frac{\alpha}{2k}}{k + \alpha\sqrt{d}}.
    \end{align}
\end{proposition}

\begin{algorithm}[t]
\caption{{\color{black}\texttt{Regularized-Regret-Min}$+$\texttt{Entropy}/$\ell_{1/2}$}}
\label{algo:rounding-regularize}
 \hspace*{\algorithmicindent} \textbf{Input:} \nolinebreak
$\X \in \mathbb{R}^{n\times d}$, budget $k \geq d$,  learning rate $\alpha>0$, regularization parameter $\lambda>0$, design objective $f$, regularizer $w\in\{\mathrm{``entropy"}, ``\ell_{1/2}" \}$ \\
 \hspace*{\algorithmicindent} \textbf{Output:} \nolinebreak
sample indices set $S$
\begin{algorithmic}[1]
\STATE $\bpi^*\gets$ solution of \cref{eq:reg-lp} with certain design objective $f$
\STATE $\bSigma_{\diamond} \gets \X^\top \mathrm{diag}(\bpi^*)\X + \lambda \b I$, $\XX\gets \X \bSigma_{\diamond}^{-1/2}$
\STATE $\A_1 = \frac{1}{d} \b I$
\FOR {$t = 1$ to $k$}
    \IF{$w$ is ``entropy''}
        \STATE $\B_t \gets \exp\left[\log \A_t- \frac{\alpha\lambda}{k} \bSigma_{\diamond}^{-1}\right]$
        \STATE $i_t \gets \argmax_{i \in [n]\backslash S} \left[1 - \exp(-\alpha \| \xx_i\|_2^2)\right] \frac{\xxt_i \B_t \xx_i}{\| \xx_i\|_2^2}$
    \ELSE
        \STATE $\B_t \gets \left[\A_t^{-1/2} + \frac{\alpha \lambda}{k} \bSigma_{\diamond}^{-1}\right]^{-2}$
        \STATE $ i_t \gets \argmax_{i \in [n]\backslash S} \frac{\xxt_i \B_t \xx_i}{1 + \alpha \xxt_i \B_t^{1/2} \xx_i}$
    \ENDIF
    \STATE Repeat Lines 10-19 of \cref{algo:rounding}.
    \STATE $\A_t \gets \b V \b \Lambda^\prime \b V^\top$
\ENDFOR
\end{algorithmic}
\end{algorithm}

\begin{restatable}{theorem}{thmcomplexityreg}\label{thm:reg-design-complexity}
    Given sample matrix $\X \in\mathbb{R}^{n\times d}$ and sample selection budget $k\geq d$. We select samples by \cref{algo:rounding-regularize} with learning rate $\alpha$ and entropy or $\ell_{1/2}$ regularizer $w$ to solve the \textbf{regularized} experimental optimal design problem~\cref{eq:reg-ip} with objective $f$ satisfying \cref{assume:f}. Let $\epsilon\in(0,1)$, suppose that one of the following condition is satisfied:

\begin{enumerate}[label=(\alph*)]
    \item $w$ is entropy-regularizer, $\alpha = 4\log d/\epsilon$, and $k\geq 16d\log d/\epsilon^2$.
    \item $w$ is $\ell_{1/2}$-regularizer, $\alpha = 8\sqrt{d}/\epsilon$, and $k\geq 32 d/\epsilon^2 + 16\sqrt{d}/\epsilon^2$.
\end{enumerate}
Then the samples selected by \cref{algo:rounding-regularize} achieve near-optimal performance, i.e.
\begin{align}
    f\left(\X_S^\top \X_S+\lambda \b I\right) \leq (1+\epsilon) f^*,
\end{align}
where $S$ is the sample indices set, $\X_S\in\mathbb{R}^{k\times d}$ is the selected sample matrix, and $f^*$ is the optimal value of design objective in \cref{eq:reg-ip}.
\end{restatable}

\section{Numerical experiments}\label{sec:experiments}
\definecolor{Gray}{gray}{0.9}
\definecolor{LightCyan}{rgb}{0.88,1,1}

\begin{algorithm}[t]
\caption{\texttt{MMD-critic}}
\label{algo:mmd}
 \hspace*{\algorithmicindent} \textbf{Input:} \nolinebreak
sample indices set $S=\emptyset$, budget $k$
\begin{algorithmic}[1]
\WHILE {$|S| < k$}
    \FOR{$i\in[n]\backslash S$}
       \STATE $f_i = J(S \cup \{i\}) - J(S)$ 
       \ENDFOR
    \STATE $S \gets S \cup \{\argmax_{i\in [n] \backslash S} f_i\} $
\ENDWHILE
\end{algorithmic}
\end{algorithm}

\begin{algorithm}[t]
\caption{Greedy A-optimal design}
\label{algo:greedy-A}
 \hspace*{\algorithmicindent} \textbf{Input:} \nolinebreak $\X\in\mathbb{R}^{n\times d}$, budget $k$, initial subset $S_0\subseteq [n]$ ($|S_0| > k$)
\begin{algorithmic}[1]
\STATE $S \gets S_0$
\WHILE {$|S| > k$}
   \STATE $i\gets \argmin_{i \in S}\Tr\big[ \big( \X_S^\top \X_S - \x_i \x_i^
\top\big)^{-1}\big]$
    \STATE $S \gets S \setminus\{ i\}$
\ENDWHILE
\end{algorithmic}
\end{algorithm}

In our experiments, we aim to address the following questions: (i) How does the performance of \regretent compare to the original \regretlhalf for solving the optimal design problem? (ii) How effective is our extension of the regret minimization framework in tackling the optimal design problem under ridge regression? (iii) How well does the regret minimization approach select representative samples from unlabeled data for classification tasks, compared to baseline methods?

\subsection{\color{black}Synthetic Dataset}\label{sec:synthetic-experiments}
{\color{black}
We consider a pool of $2,000$ data points represented by the matrix $\X \in \mathbb{R}^{2,000 \times 100}$, following the experimental setup of \cite{allen2021near}. Specifically, $\X$ is structured as
\begin{align}
\X = \begin{bmatrix}
\X_1 & \mathbf{0}\\
\mathbf{0} & \X_2
\end{bmatrix},
\end{align}
where $\X_1 \in \mathbb{R}^{1,000 \times 50}$ is a Gaussian random matrix rescaled so that the eigenvalues of $\X_1^\top \X_1$ decay quadratically, i.e., $\lambda_i(\X_1^\top \X_1) \propto i^{-2}$, and $\X_2 \in \mathbb{R}^{1,000 \times 50}$ is a Gaussian random matrix rescaled so that the eigenvalues of $\X_2^\top \X_2$ decay linearly, i.e., $\lambda_i(\X_2^\top \X_2) \propto i^{-1}$.  

We evaluate both the unregularized design problem \cref{eq:ip} and the regularized design problem \cref{eq:reg-ip} under multiple optimality criteria (A-, D-, E-, and V-optimality). The selection sizes $k$ considered are:
\begin{itemize}
    \item \textbf{Unregularized problem:} $k \in \{120, 200, 300\}$.
    \item \textbf{Regularized problem:} $k \in \{60, 80, 100, 120, 200, 300\}$ for A-, D-, and V-optimality, and $k \in \{120, 200, 300\}$ for E-optimality. For E-optimality, if $k < d$, then $\|(\X_S^\top \X_S + \lambda \b I)^{-1}\|_2 = 1/\lambda$, so we restrict to $k > d$.
\end{itemize}

In addition to Regret-Min, we include two baselines: (i) uniform sampling, and (ii) the swapping algorithm (Algorithm 1 in \cite{allen2021near}). The swapping method is also based on regret minimization, achieving a $(1+\epsilon)$-near-optimal performance guarantee with sample complexity $\mathcal{O}(d/\epsilon^2)$. Similar to Regret-Min, it first solves the relaxed problem; however, in the rounding phase, it starts from an arbitrary selection of $k$ points and then iteratively refines the design by swapping a chosen point with an unselected one based on regret-driven objectives.

\cref{fig:design-synthetic} summarizes the experimental results on the synthetic dataset, comparing Uniform, Swapping, and (Regularized-)Regret-Min under various optimality criteria. The plots report relative objective values with respect to the relaxed optimum $f^\diamond$. We observe that the relative objectives consistently decrease as $k/d$ increases, indicating that larger subsets yield designs closer to the optimum. (Regularized-)Regret-Min consistently achieves the smallest relative gaps across all criteria, significantly outperforming Uniform and Swapping. While Swapping improves over Uniform and narrows the gap at larger $k/d$, it is less competitive for smaller sample sizes. Comparing the upper and lower rows of \cref{fig:design-synthetic}, we see that regularization further reduces the relative gap, especially for E- and V-optimality. Finally, \texttt{(Regularized-)Regret-Min}$+$\texttt{Entropy} performs almost identically to \texttt{(Regularized-)Regret-Min}$+\ell_{1/2}$.
}

\begin{figure}[t]
    \centering
\footnotesize
\begin{tikzpicture}
\node[inner sep=0pt] (a1) at (0,0) {\includegraphics[width=12cm]{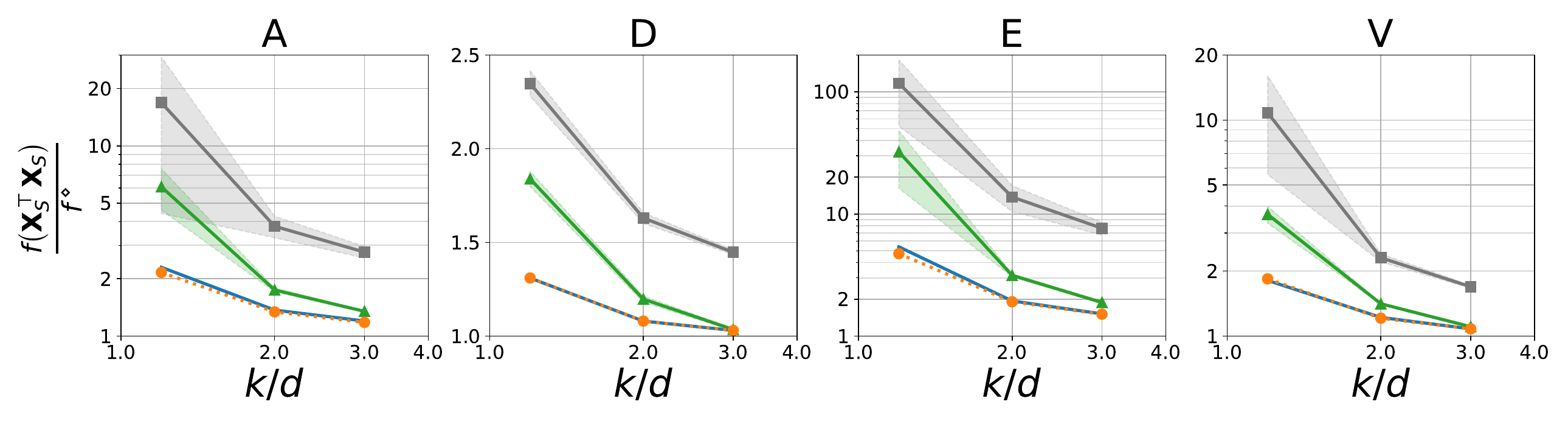}};
\node[inner sep=0pt] (a2) at (0,-3.5) {\includegraphics[width=12cm]{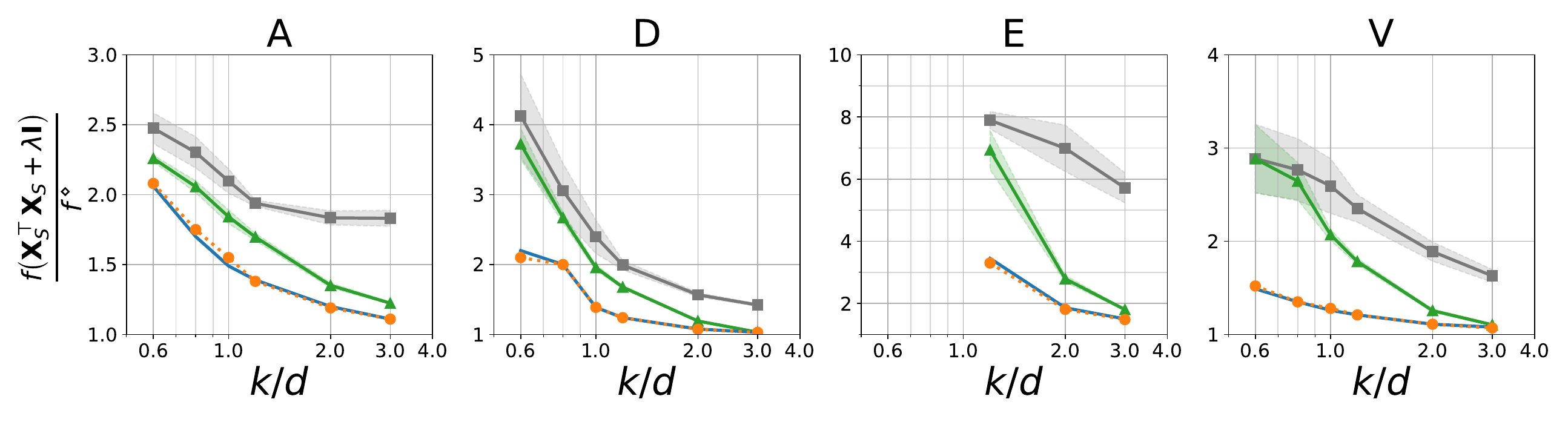}};
\node[inner sep=0pt] (a3) at (0.2,-1.75) {\includegraphics[width=11.5cm]{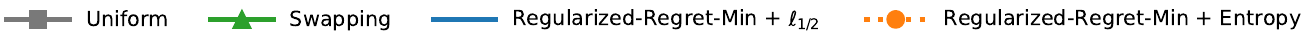}};
\node[inner sep=0pt] (a0) at (0,1.75) {\includegraphics[width=9cm]{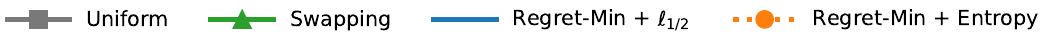}};
\draw[] (-6.6, -1.5) -- (6,-1.5);
\draw[] (-6., -5) -- (-6.,1.95);
\draw (-6.6,-5) rectangle (6,1.95);
\node[rotate=90,anchor=north] at (-6.6, 0) {\textbf{Problem~\cref{eq:ip}}};
\node[rotate=90,anchor=north] at (-6.6, -3.3) {\textbf{Problem~\cref{eq:reg-ip}}};
\end{tikzpicture}
\caption{\color{black}Results on a synthetic dataset ($n=2,000$, $d=100$). For ``Uniform" and ``Swapping", we report the mean and standard deviation over 10 trials. The top row shows results for the optimal design problem \cref{eq:ip}, and the bottom row shows results for the regularized optimal design problem \cref{eq:reg-ip}. Here, $f_\diamond$ denotes the objective value obtained from the {\color{black}{relaxed}} problem. }
\label{fig:design-synthetic}
\end{figure}

\color{black}

\subsection{Real-world classification datasets}\label{sec:realworld-experiments}
\subsubsection{Experimental setting}
\paragraph{Datasets} We conduct experiments on three image classification datasets: MNIST~\cite{deng-2012mnist}, CIFAR-10~\cite{cifar10}, and ImageNet~\cite{deng2009imagenet}. To enable effective sample selection and evaluation, we apply unsupervised preprocessing to obtain suitable feature representations. Importantly, no label information is used during preprocessing. Details for each dataset are as follows:

\begin{itemize}
    \item \textbf{MNIST:} The original dataset contains 60,000 samples, each represented as a 784-dimensional vector. We reduce the dimensionality to 20 using spectral embedding via the normalized graph Laplacian (see \cref{algo:laplacian}), with 256 nearest neighbors used in the construction of the graph. This results in a feature matrix $\X \in \mathbb{R}^{60,000 \times 20}$.
    \item \textbf{CIFAR-10:} We preprocess the 50,000 training samples using SimCLR~\cite{simclr}, a contrastive learning framework that learns visual representations by maximizing agreement between augmented views of the same image in a latent space. The resulting feature vectors have a dimension of 512. We then apply either principal component analysis (PCA) or spectral embedding (via normalized graph Laplacian, \cref{algo:laplacian}) to further reduce the dimension of features  obtained by SimCLR.
\item \textbf{ImageNet-50:} We use DINOv2~\cite{oquab2023dinov2}, a state-of-the-art self-supervised model, to extract 2,048-dimensional feature embeddings from the ImageNet dataset. These embeddings are then reduced to 50 dimensions using PCA. From the full ImageNet dataset, we randomly select 50 classes to construct the ImageNet-50 subset, which contains 64,273 samples. This ImageNet-50 dataset is \textit{class-balanced}. Additionally, we create a \textit{class-imbalanced} version of ImageNet-50 by sampling a varying number of examples per class, with total number of 5,000 samples.
\end{itemize}

\paragraph{Evaluation metric}
{\color{black}To evaluate the representativeness of the selected samples, we use logistic regression from scikit-learn~\cite{scikit-learn} as the classifier, with the detailed settings provided in \Cref{sec:new-experiments}. }
Specifically, we train the logistic regression model on the selected samples with their labels, apply the trained model to predict labels for the entire unlabeled training set, and report the average prediction accuracy.

\paragraph{Methods} We compare the performance of the following  sampling methods:

\begin{enumerate}
    \item Uniform: uniformly selects random samples from the entire pool of  data.
    \item K-Means:  this method uses K-Means clustering method to group all samples into $k$ clusters. Then, it selects the sample closest to the centroid of each group.
    \item RRQR: the rank-revealing QR algorithm is implemented in \texttt{Scipy.linalg.qr} with pivoting. 
    \item MMD: MMD-critic is a method proposed in \cite{kim2016examples}, which can be used to select prototypes by the maximum mean discrepancy (MMD). We use the \texttt{MMD-critic} implementation from \url{https://github.com/BeenKim/MMD-critic}. The method employs the radial basis function (RBF) kernel defined as $k(x_i, x_j) = \exp\big(- \frac{1}{2}(\|x_i - x_j\|_2/\gamma)^2\big)$. It constructs the dense kernel matrix $\mathbf{K}\in\mathbb{R}^{n\times n}$ with pairwise kernels, where $\mathbf{K}_{i,j} = k(x_i, x_j)$. The value of $\gamma$ is chosen as the median of all pairwise Euclidean distances (i.e. the ``median trick"). The algorithm is presented in~\cref{algo:mmd}, where
    \[{\color{black}J(S)}=\frac{2}{n |S|} \sum_{i\in[n], j\in S}k(x_i, x_j) - \frac{1}{|S|^2} \sum_{i, j \in S} k(x_i, x_j).\]
    \item Max-weights: select the  largest $k$ samples according to $\bpi^\diamond$, where $\bpi^\diamond$ is the solution of relaxed problem~\Cref{eq:lp}.
    \item Weighted-sampling: randomly select $k$ samples according to distribution $\bpi^\diamond/k$, where $\bpi^\diamond$ is the solution of relaxed problem~\Cref{eq:lp}.
    \item Greedy: a {\color{black}greedy} removal method for A-optimal design is proposed in~\cite{Avron13}. The algorithm is listed in \cref{algo:greedy-A}. We apply this method with initial set $S_0\subseteq [n]$ selected randomly with size $|S_0| = 10 d$.
    \item Regret-Min: without notice, we report the results obtained by applying Algorithm \ref{algo:rounding} with the \textbf{entropy-regularizer} to solve \textbf{A-optimal} design problem. It performs a grid search for the learning rate $\alpha$ that \textit{minimizes} the design objective function, i.e. $f(\X_S^\top \X_S)$, where matrix $\X_S \in\mathbb{R}^{k\times d}$ represents the selected samples.
\end{enumerate}
For randomized methods such as Uniform, K-Means, Weighted Sampling, and Greedy, we report the mean and standard deviation over 20 independent trials for each experiment. For the relaxed solver used in Max-Weights, Weighted Sampling, and Regret-Min across all optimal design problems, mirror descent iterations are terminated once the relative change in the objective falls below $1.0 \times 10^{-5}$.

\subsubsection{Results}
\paragraph{Entropy-regularizer vs. $\ell_{1/2}$-regularizer}
{\color{black}
We evaluate the performance of entropy and $\ell_{1/2}$ regularizers in our Regret-Min algorithms across a variety of datasets and settings. Specifically, we compare both the standard Regret-Min (\cref{algo:rounding}) and the Regularized-Regret-Min (\cref{algo:rounding-regularize}) on CIFAR-10, ImageNet-50, and their randomly sampled subsets. Experiments are conducted under different sample budgets $k$, with learning rates $\alpha$ ranging from 1 to 100. For each configuration, we record both the relative objective value $f(\X_S^\top \X_S)/f^{\diamond}$—where $\X_S$ denotes the selected samples and $f^{\diamond}$ the optimum of the relaxed problem—and the prediction accuracy of logistic regression.  

Results for the full CIFAR-10 and ImageNet-50 datasets are shown in \cref{fig:reg-compare,fig:reg-compare-regularized}, while those for the subset experiments appear in \cref{fig:subset-cifar10-compare,fig:subset-imagenet-compare} in \Cref{sec:new-experiments}. For ease of discussion, we denote by
\begin{itemize}
\item $\alpha_{\text{obj}}^*$: the learning rate that minimizes the \textit{design objective} (marked by red dots in the plots).
\item $\alpha_{\text{acc}}^*$: the learning rate that maximizes \textit{logistic regression prediction accuracy} (marked by blue dots in the plots).
\end{itemize}

For the objective value $f$ (red curves in the plots), we find that $\alpha_{\text{obj}}^*$ typically lies in $(1,10)$ for the entropy-regularizer, while for the $\ell_{1/2}$-regularizer it usually falls in $(10,100)$, except in the regularized CIFAR-10 experiments shown in \cref{fig:reg-compare-regularized}, where $\alpha_{\text{obj}}^* \in (1,10)$. Despite these differences in optimal learning rates, both regularizers achieve nearly identical objective function values at their respective optima. The regularizers exhibit markedly different sensitivity profiles with respect to learning rate: the entropy regularizer demonstrates sharp, well-defined optimization landscapes with clear minima—objective values decrease sharply as $\alpha$ approaches $\alphaobj$. In contrast, the $\ell_{1/2}$-regularizer produces characteristically flatter response surfaces with broader basins of near-optimal performance, allowing multiple learning rate values to achieve comparable objective function performance. This effect is more pronounced in the ImageNet-50 subset experiments (see \cref{fig:subset-imagenet-compare}).

As for the profile of accuracy (blue curves in the plots), both regularizers achieve comparable maximum accuracy levels, with differences typically under 0.5\% across the experiments shown in \cref{fig:reg-compare,fig:reg-compare-regularized,fig:subset-cifar10-compare}. In the ImageNet-50 subset experiments (\cref{fig:subset-imagenet-compare}), the entropy-regularizer attains accuracy values that are at most ~1\% higher in a few subsets. These differences are small and not statistically significant; they primarily reflect variance in the downstream logistic-regression classifier rather than inherent advantages of either regularizer.

In general, accuracy exhibits high sensitivity to learning rate selection for both regularizers. For the entropy-regularizer, the optimal accuracy learning rate $\alphaacc$  falls within the narrow range $(1,10)$ for almost all experimental conditions. In contrast, the $\ell_{1/2}$-regularizer demonstrates considerably more variability, with $\alphaacc$ values spanning the broader range $(1,100)$. This variability is particularly pronounced for the $\ell_{1/2}$-regularizer when comparing different subsets of  CIFAR-10 in \cref{fig:subset-cifar10-compare}: $\alphaacc$ ranges from $(3,5)$ for subsets \rom{2}--\rom{4}, while $\alphaacc$ ranges from $(20,40)$ for subsets \rom{1} and \rom{5}.

Regarding the consistency between the optimal learning rate for the objective function $\alphaobj$ and the optimal learning rate for classification accuracy $\alphaacc$, the entropy-regularizer demonstrates superiority over the $\ell_{1/2}$-regularizer. This conclusion is supported by the following key observations:
\begin{itemize}
    \item \textbf{Entropy-regularizer}: The values of $\alphaobj$ and $\alphaacc$ coincide in most experiments. Even when they differ, the accuracy gap between $\alpha_{\text{obj}}^*$ and $\alpha_{\text{acc}}^*$ is minimal—for example, less than $0.2\%$ in subsets \rom{1} and \rom{5} of \cref{fig:subset-cifar10-compare} and in subset \rom{1} of \cref{fig:subset-imagenet-compare}.
    \item \textbf{$\ell_{1/2}$-regularizer:} Except for the CIFAR-10 experiments in \cref{fig:reg-compare-regularized}, the optimal learning rates $\alpha_{\text{obj}}^*$ and $\alpha_{\text{acc}}^*$ \textbf{never coincide} across all remaining experiments. More importantly, the resulting accuracy gaps between $\alpha_{\text{obj}}^*$ and $\alpha_{\text{acc}}^*$  are often significant—exceeding 1\% in all cases shown in \cref{fig:subset-cifar10-compare,fig:subset-imagenet-compare} as well as in the ImageNet-50 results of \cref{fig:reg-compare}.
\end{itemize} 

Overall, our experiments show that the entropy and $\ell_{1/2}$ regularizers achieve comparable classification accuracy across all datasets. The main difference lies not in accuracy levels but in stability: the entropy-regularizer consistently yields tighter alignment between the optimal learning rate for the objective function and that for downstream accuracy, whereas the $\ell_{1/2}$-regularizer exhibits higher sensitivity to learning-rate choice. This stability makes the entropy regularizer a more reliable option in practice, particularly in scenarios where the goal is to select a small number of representative samples from large pools of unlabeled data.
}

\begin{figure}[tbp]
\centering
  \footnotesize
\begin{tikzpicture}
\node[inner sep=0pt] (a1) at (0,0) {\includegraphics[width=5.5cm]{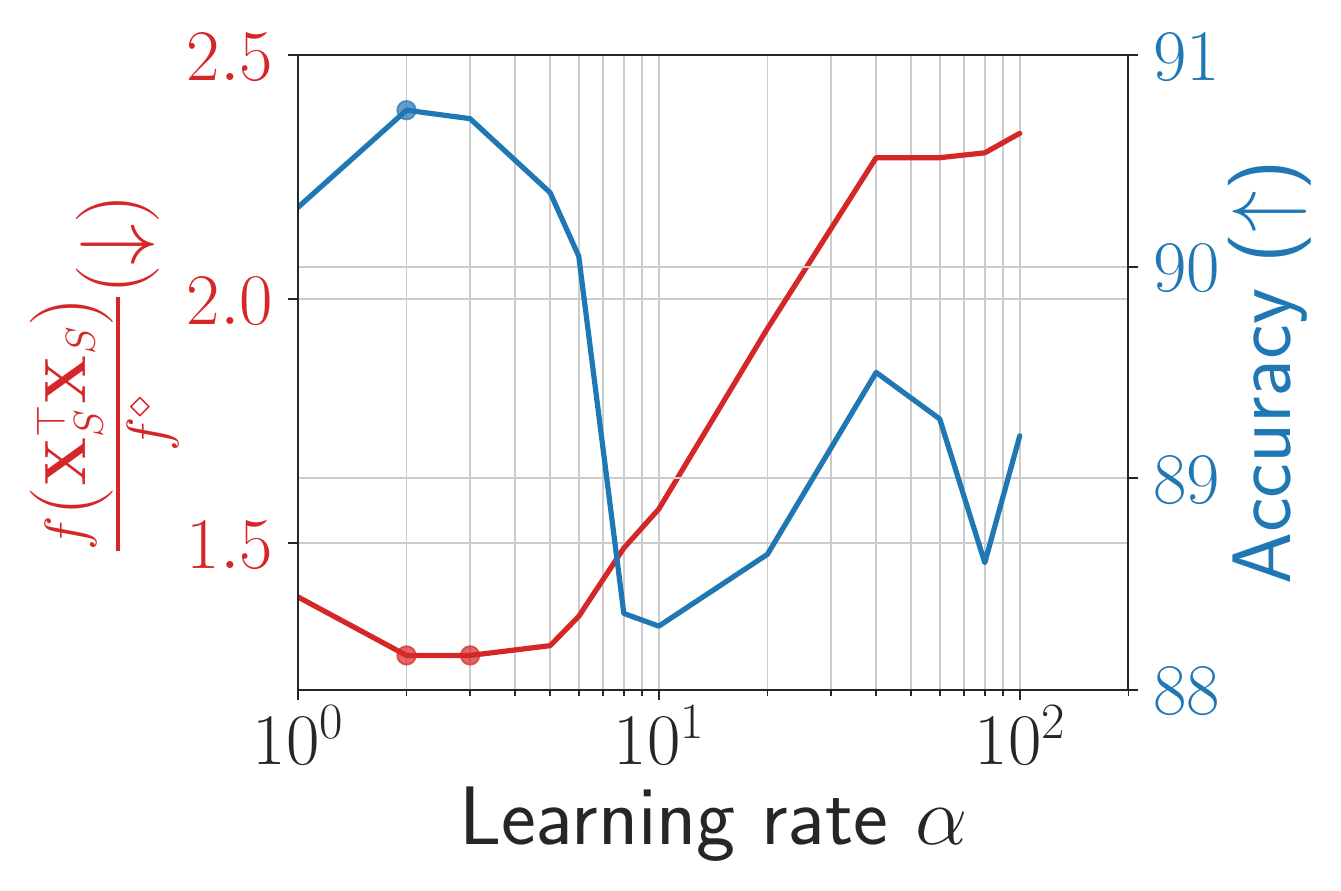}};
\node[inner sep=0pt] (b1) at (6,0) {\includegraphics[width=5.5cm]{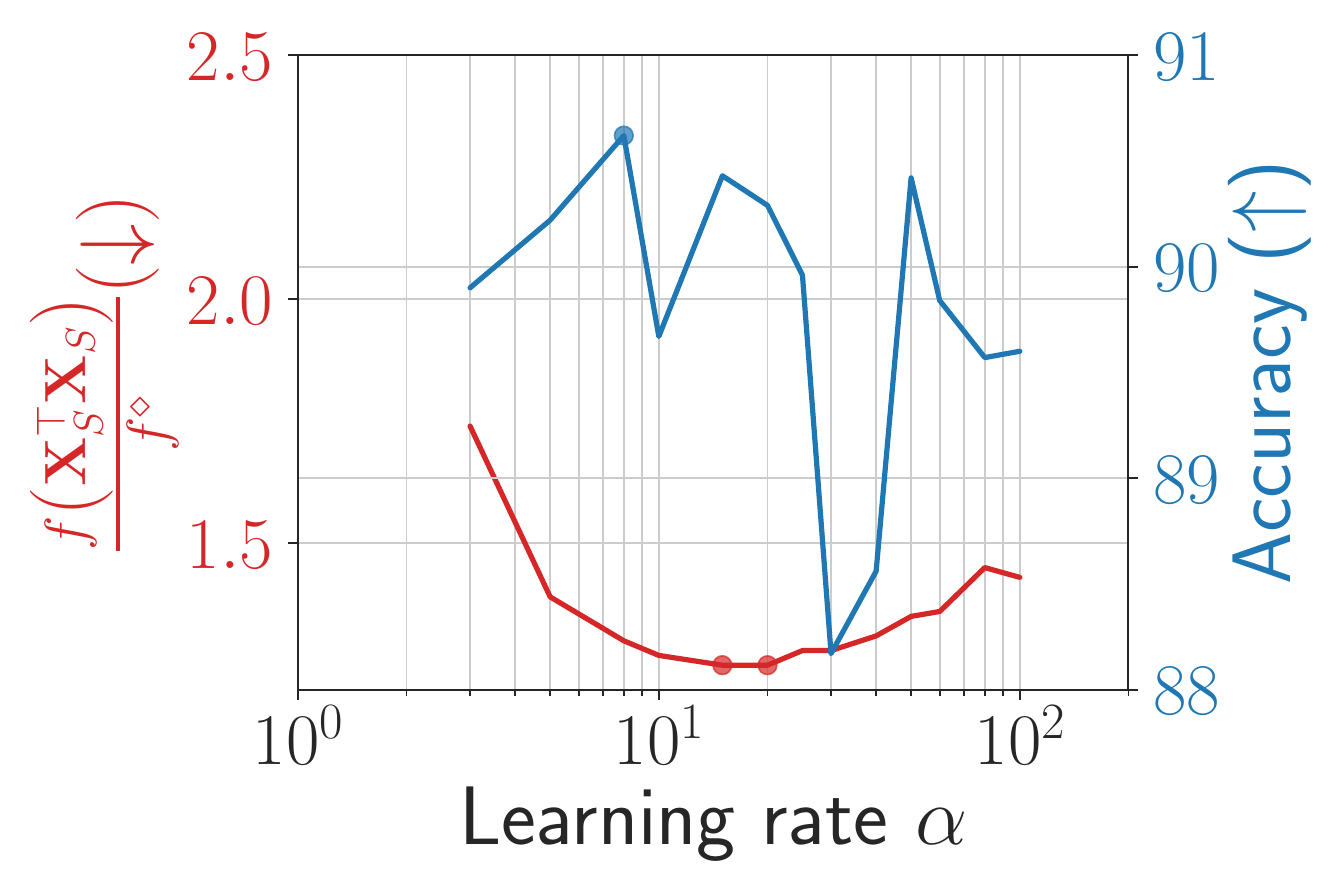}};
\node[inner sep=0pt] (a2) at (0,-3.6) {\includegraphics[width=5.5cm]{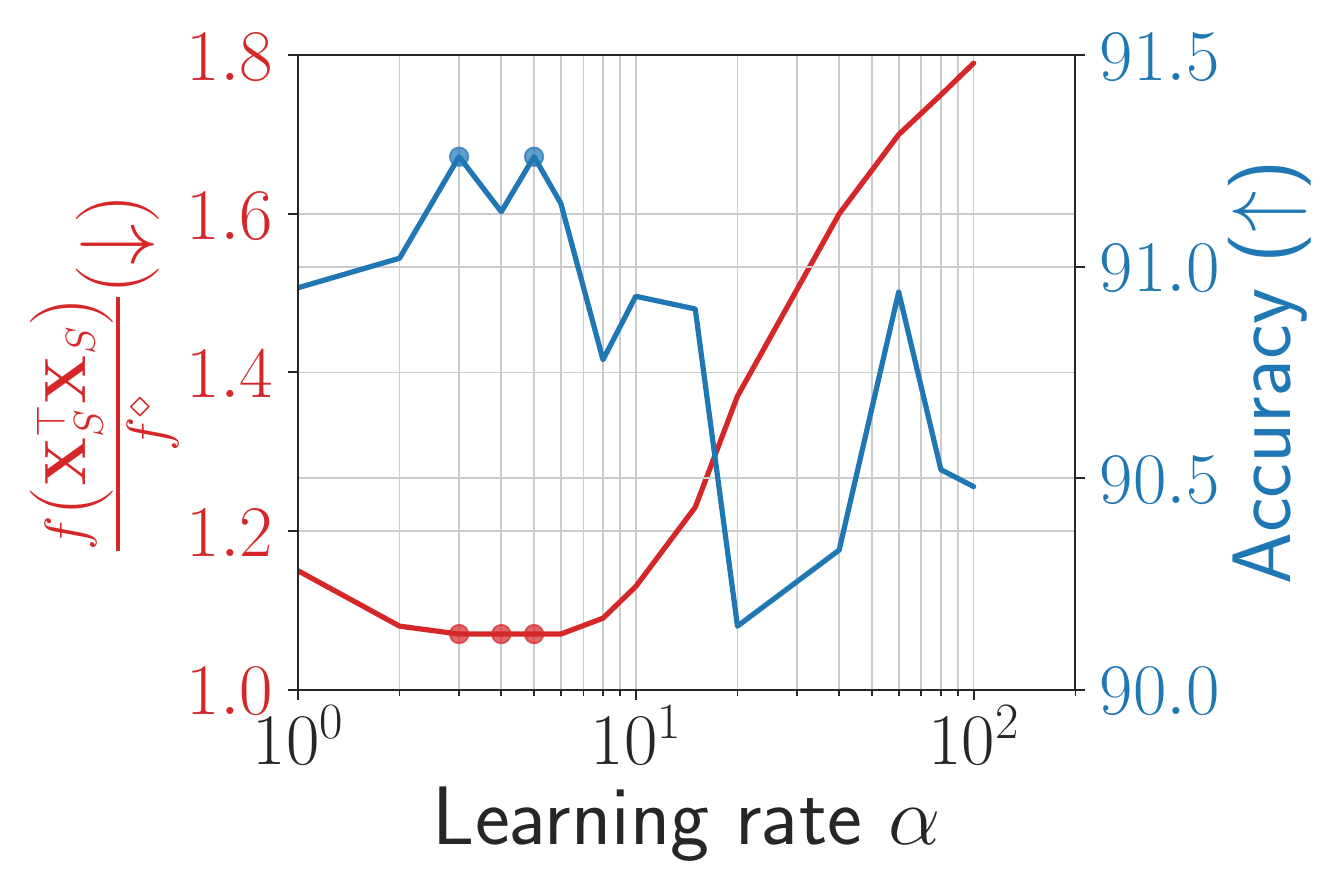}};
\node[inner sep=0pt] (b2) at (6,-3.6) {\includegraphics[width=5.5cm]{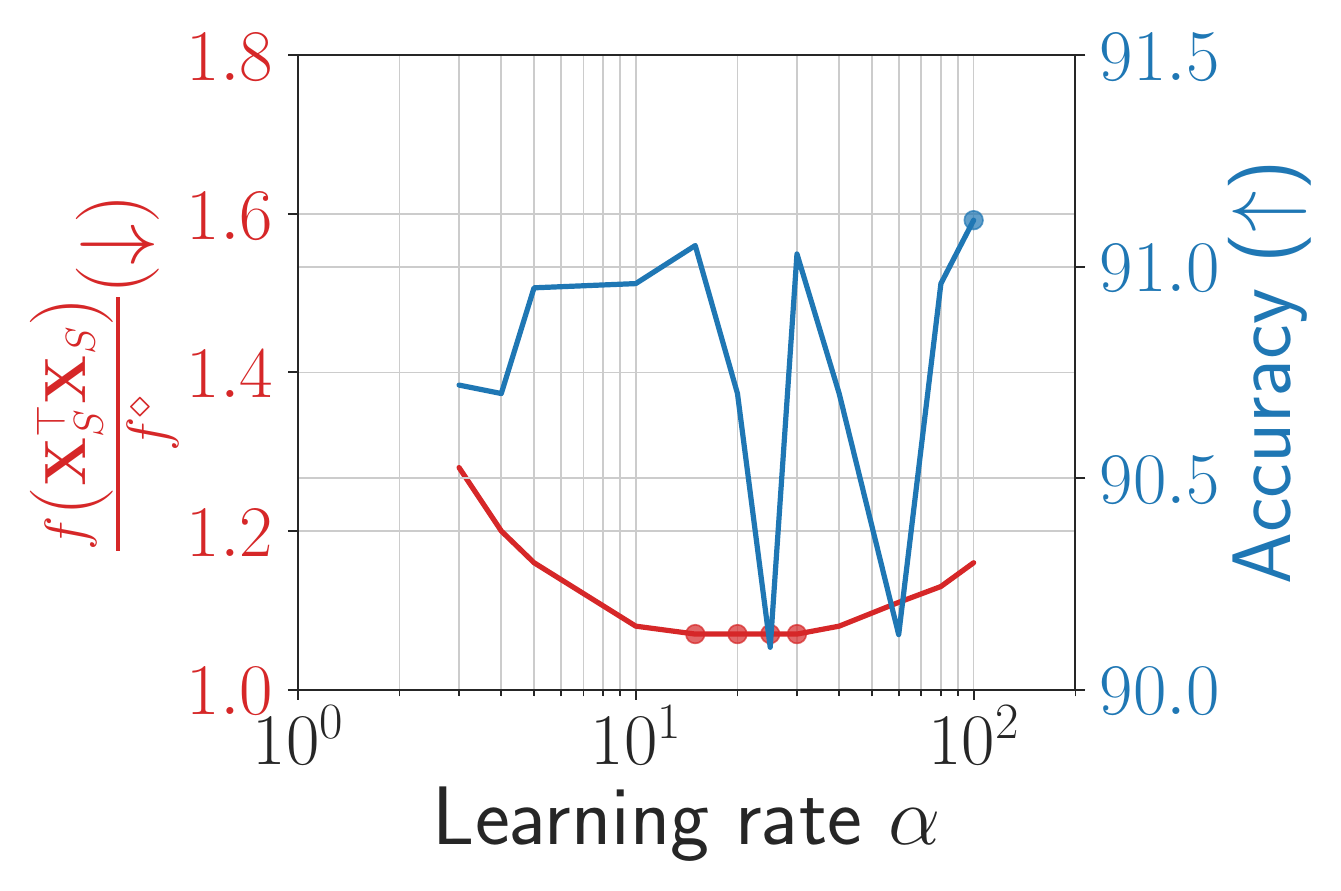}};
\node[] at (0.4,2.1) {\color{black}\regretent};
\node[] at (6.2,2.1) {\color{black}\regretlhalf};
\node[rotate=90,anchor=north] at (-3.4, 0.2) {\textbf{$k=150$}};
\node[rotate=90,anchor=north] at (-3.4, -3.4) {\textbf{$k=300$}};
\node[rotate=90,anchor=north] at (-4, -1.8) {\textbf{CIFAR-10}};
\node[inner sep=0pt] (a3) at (0,-7.4) {\includegraphics[width=5.5cm]{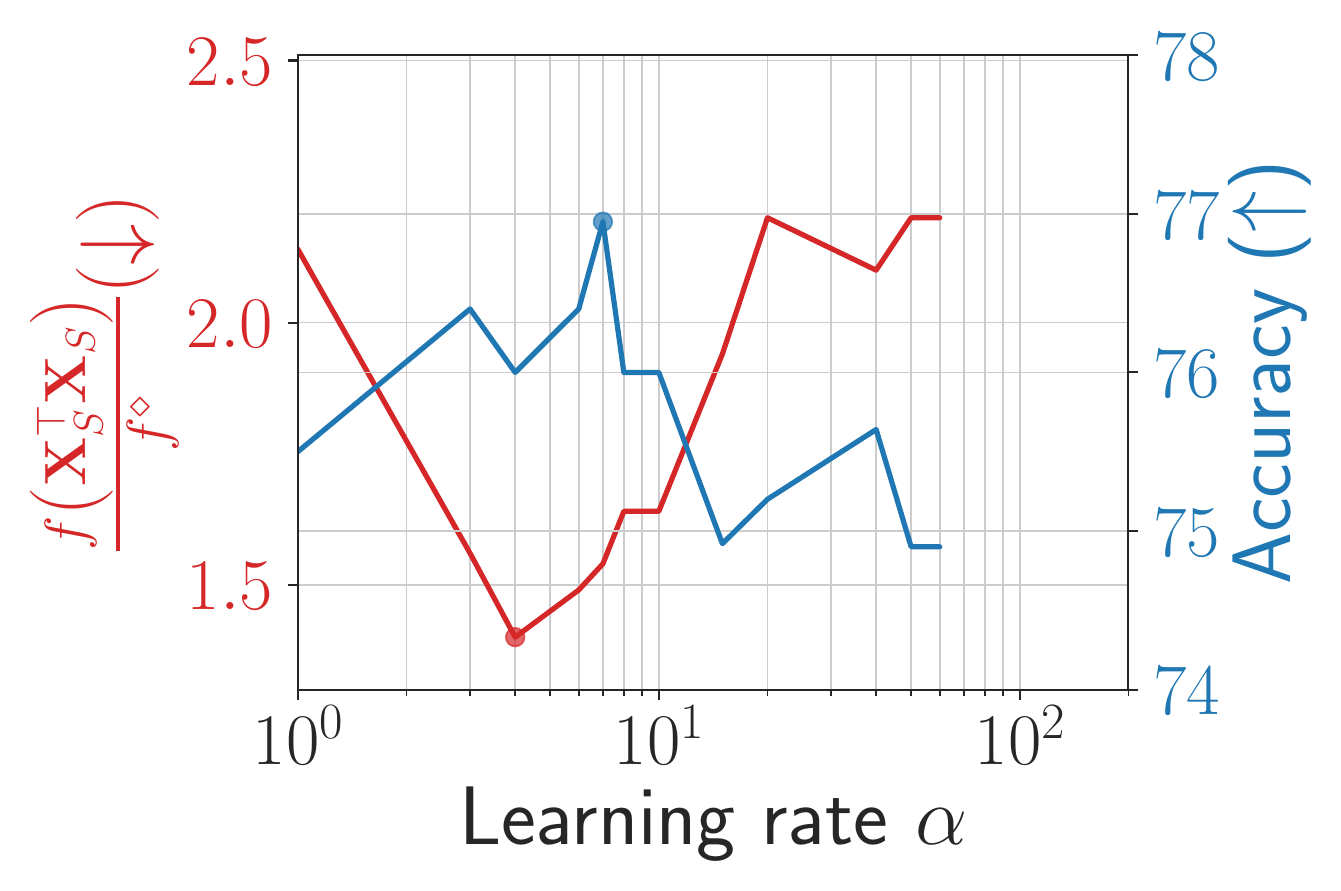}};
\node[inner sep=0pt] (b3) at (6,-7.4) {\includegraphics[width=5.5cm]{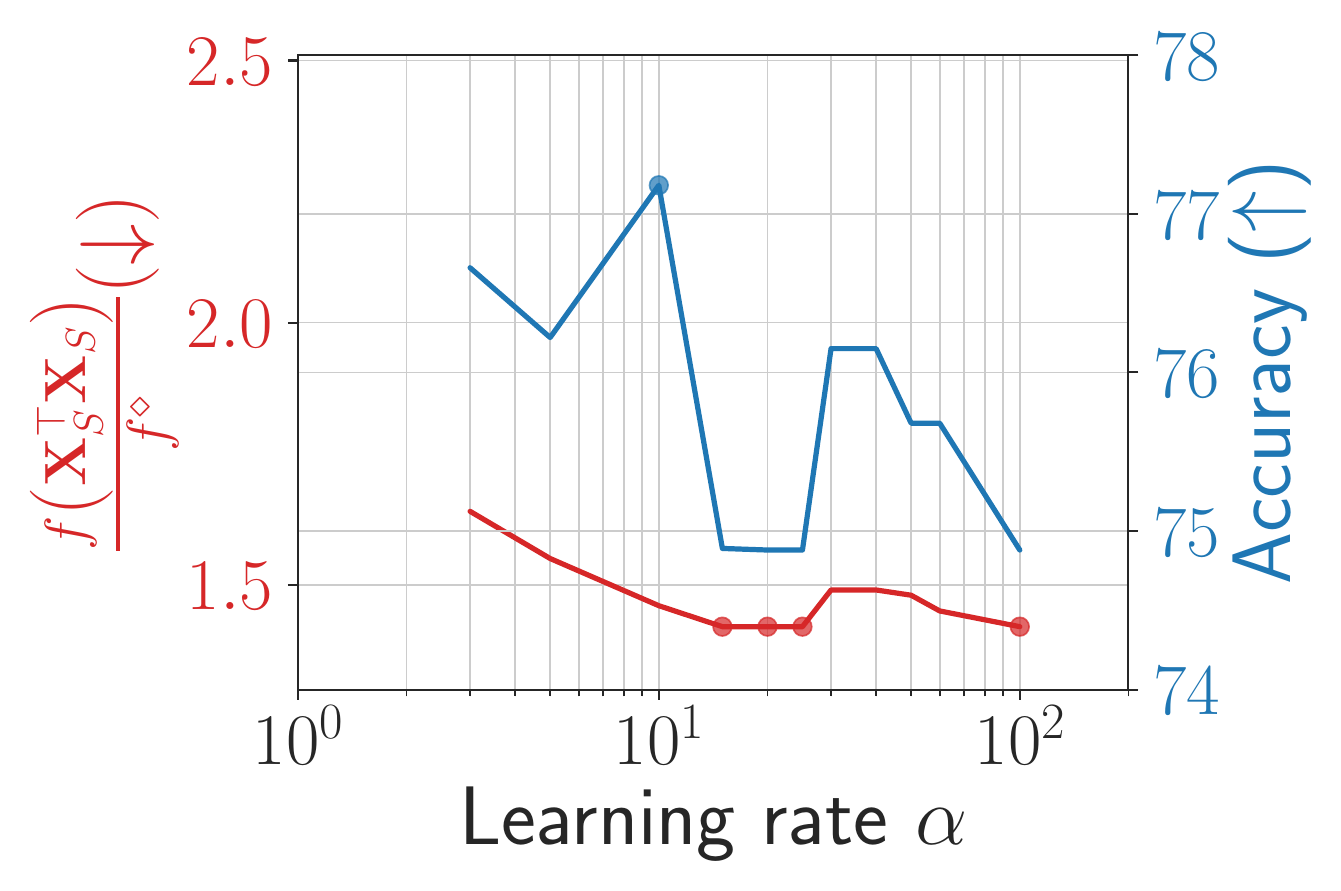}};
\node[inner sep=0pt] (a4) at (0,-11) {\includegraphics[width=5.5cm]{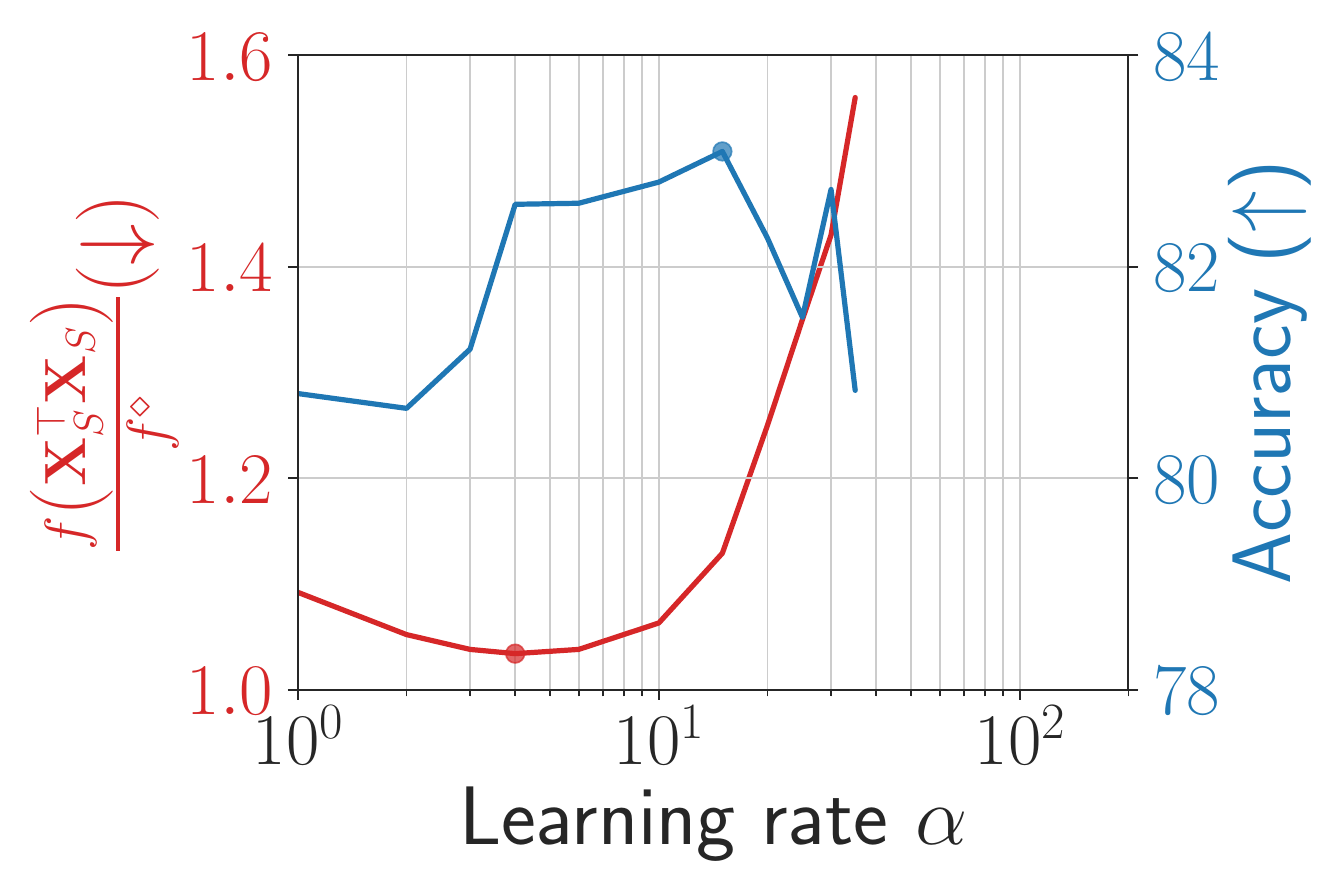}};
\node[inner sep=0pt] (b4) at (6,-11) {\includegraphics[width=5.5cm]{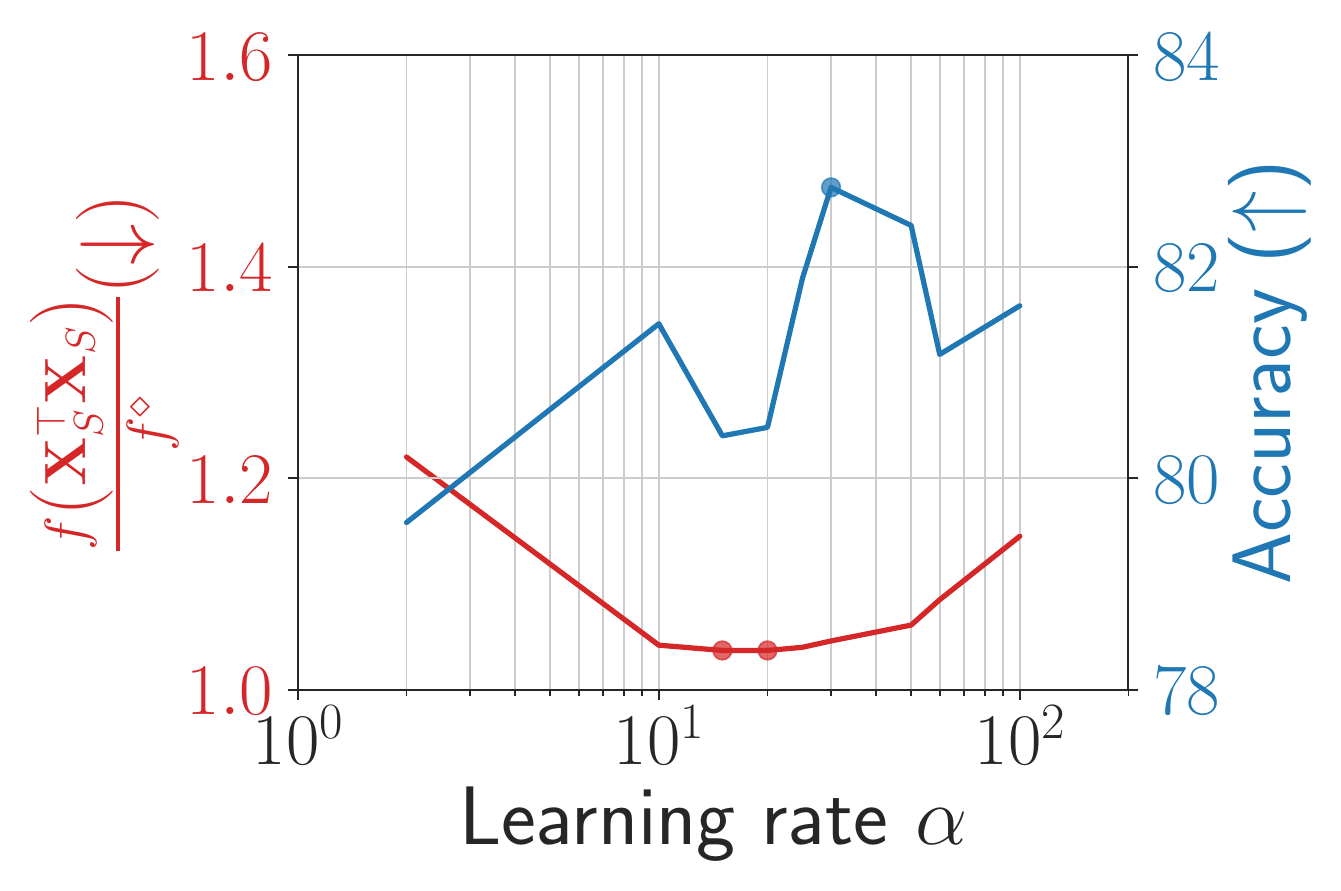}};
\node[rotate=90,anchor=north] at (-4, -8.5) {\textbf{ImageNet-50}};
\node[rotate=90,anchor=north] at (-3.4, -7.2) {\textbf{$k=50$}};
\node[rotate=90,anchor=north] at (-3.4, -10.5) {\textbf{$k=200$}};
\draw[] (-4, -5.4) -- (8.7,-5.4);
\draw[] (3,2.4) -- (3,-12.8);
\draw[] (-2.9,2.4) -- (-2.9,-12.8);
\draw (-4,-12.8) rectangle (8.7,1.8);
\draw (-4,1.8) rectangle (8.7,2.4);
\end{tikzpicture}
\caption{Comparison of \regretent and \regretlhalf (\cref{algo:rounding}) for A-design on CIFAR-10 (upper) and ImageNet-50 (lower). For CIFAR-10, the PCA-reduced features with dimension of 100 are used, for ImageNet-50, the PCA-reduced features with dimension of 50 are used. The red lines in the plot represent the relative value of the objective function $\frac{f(\X_S^\top \X_S)}{f^\diamond}$, where $\X_S$ is the selected samples and $f^\diamond$ is the optimal value of the relaxed problem~\cref{eq:lp}. The blue lines represent the logistic regression prediction accuracy. The dots on each line represent the optimal points.}
\label{fig:reg-compare}
\end{figure}


\begin{figure}[tbp]
\centering
  \footnotesize
\begin{tikzpicture}
\node[inner sep=0pt] (a1) at (0,0) {\includegraphics[width=5.5cm]{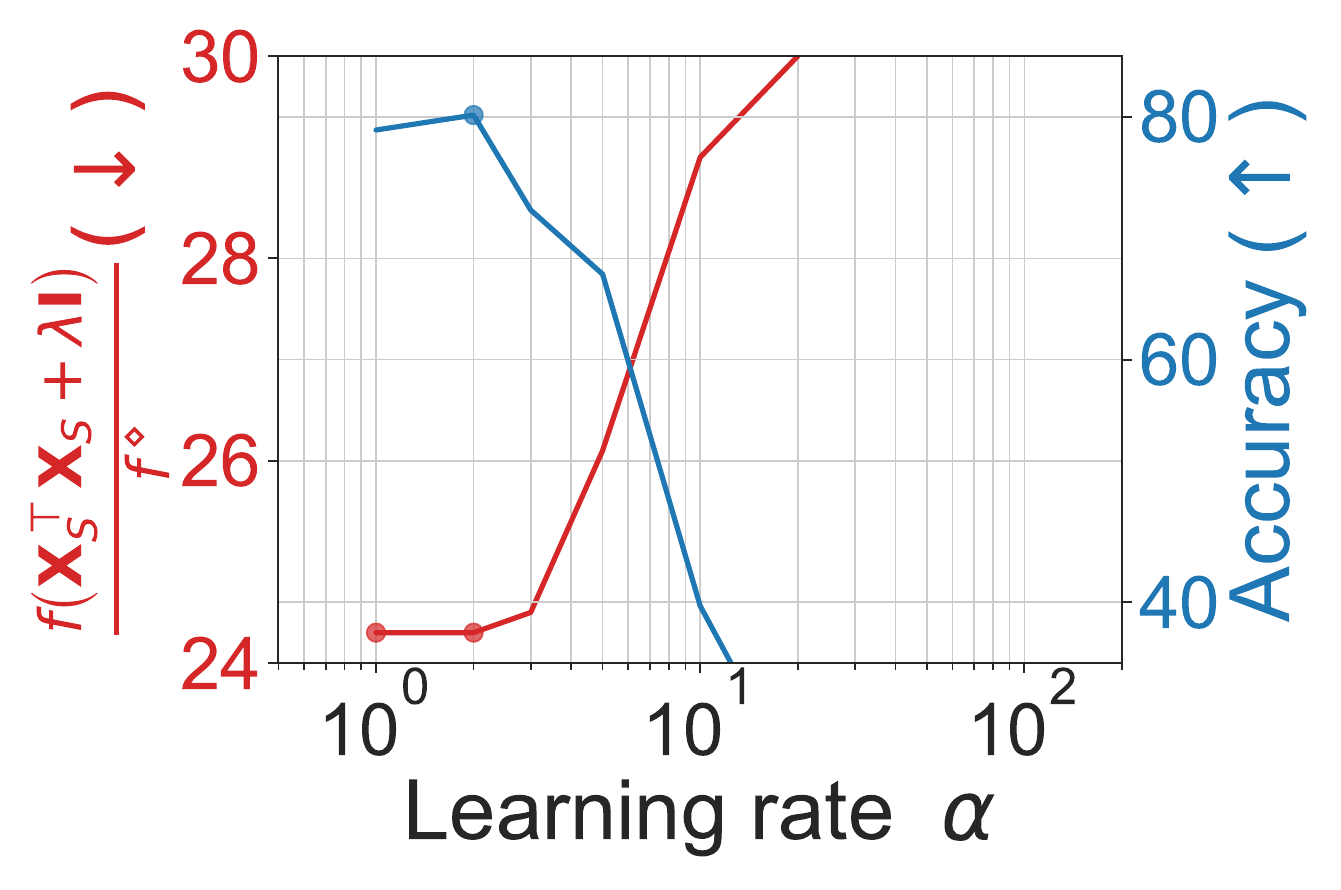}};
\node[inner sep=0pt] (b1) at (6,0) {\includegraphics[width=5.5cm]{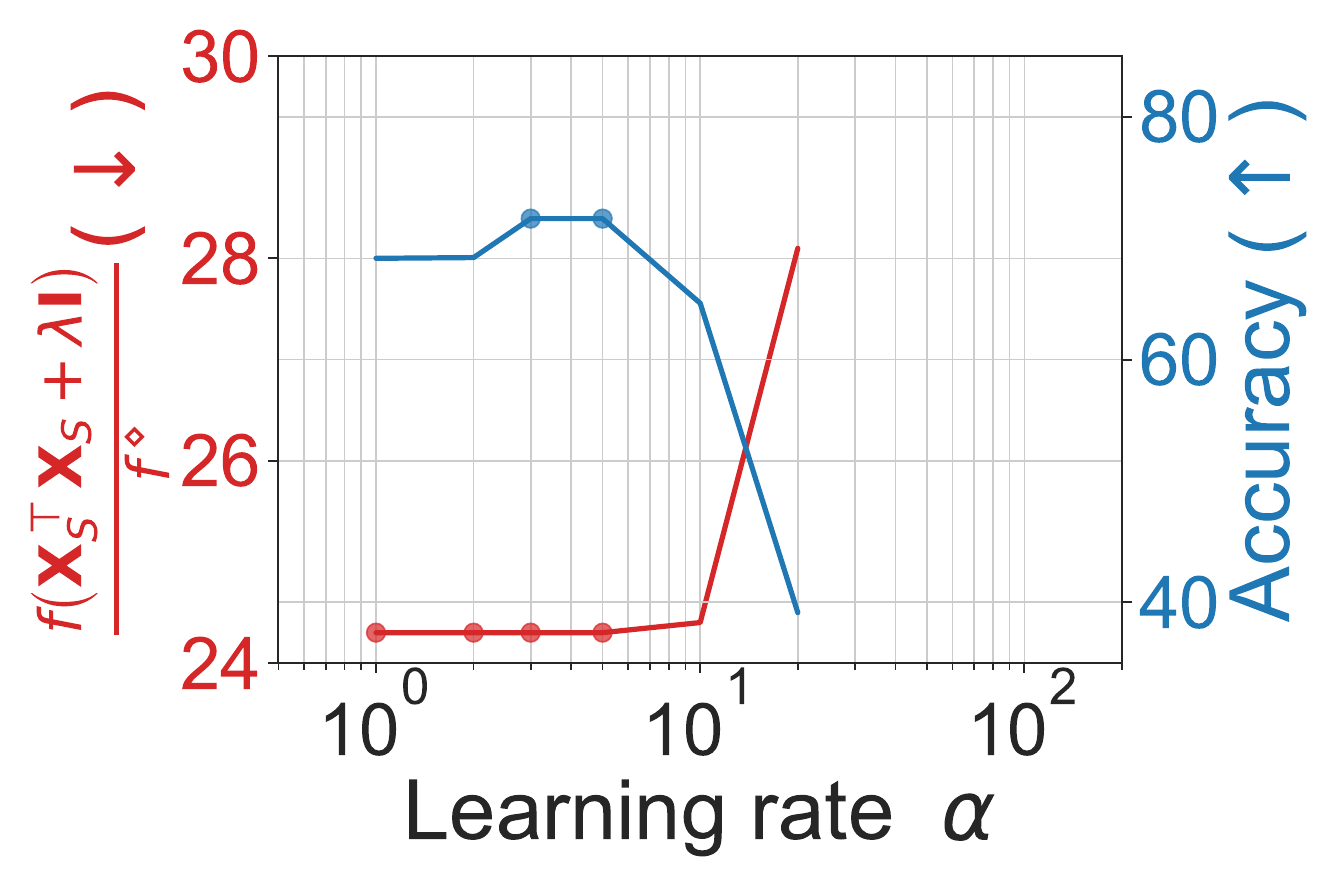}};
\node[inner sep=0pt] (a2) at (0,-3.6) {\includegraphics[width=5.5cm]{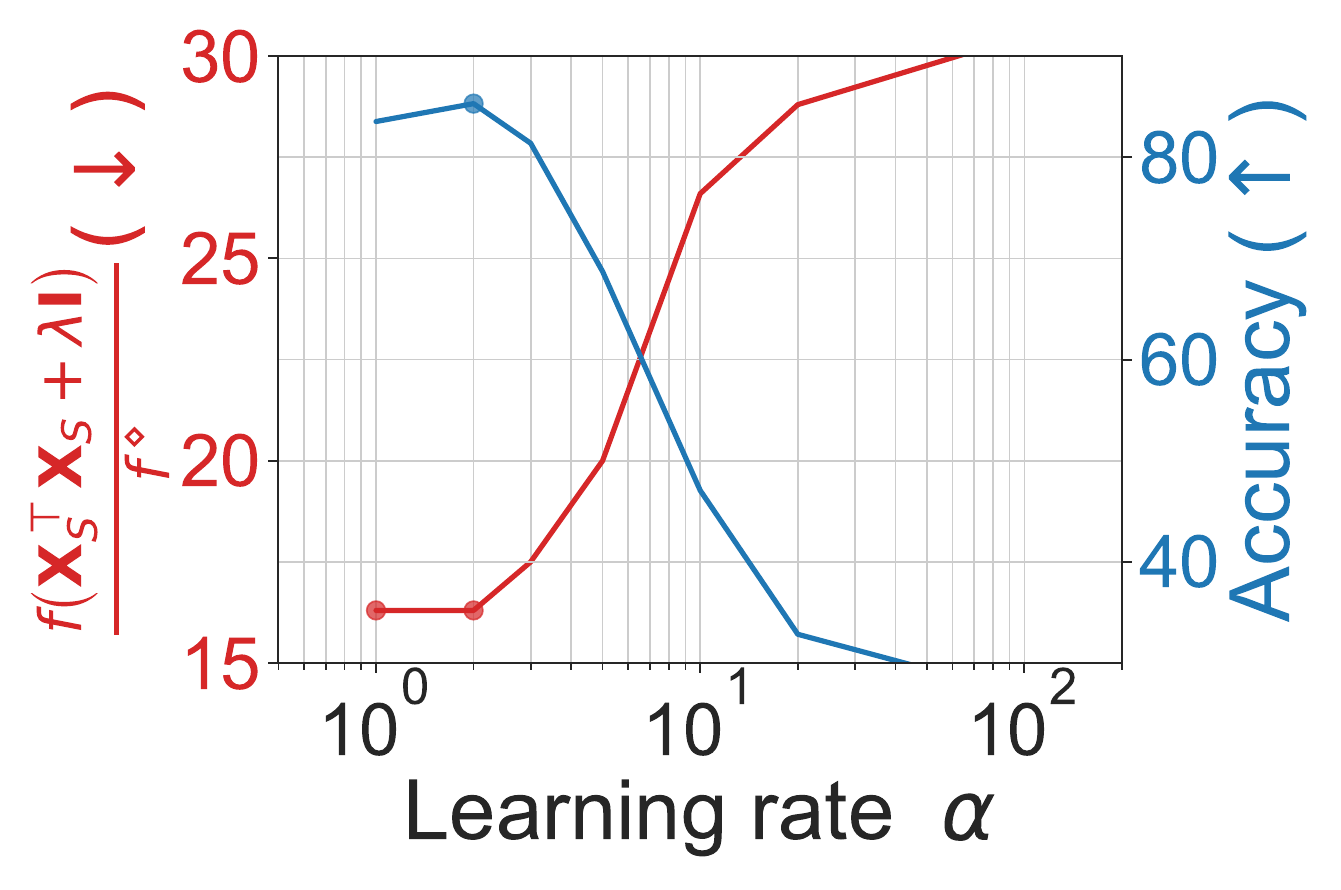}};
\node[inner sep=0pt] (b2) at (6,-3.6) {\includegraphics[width=5.5cm]{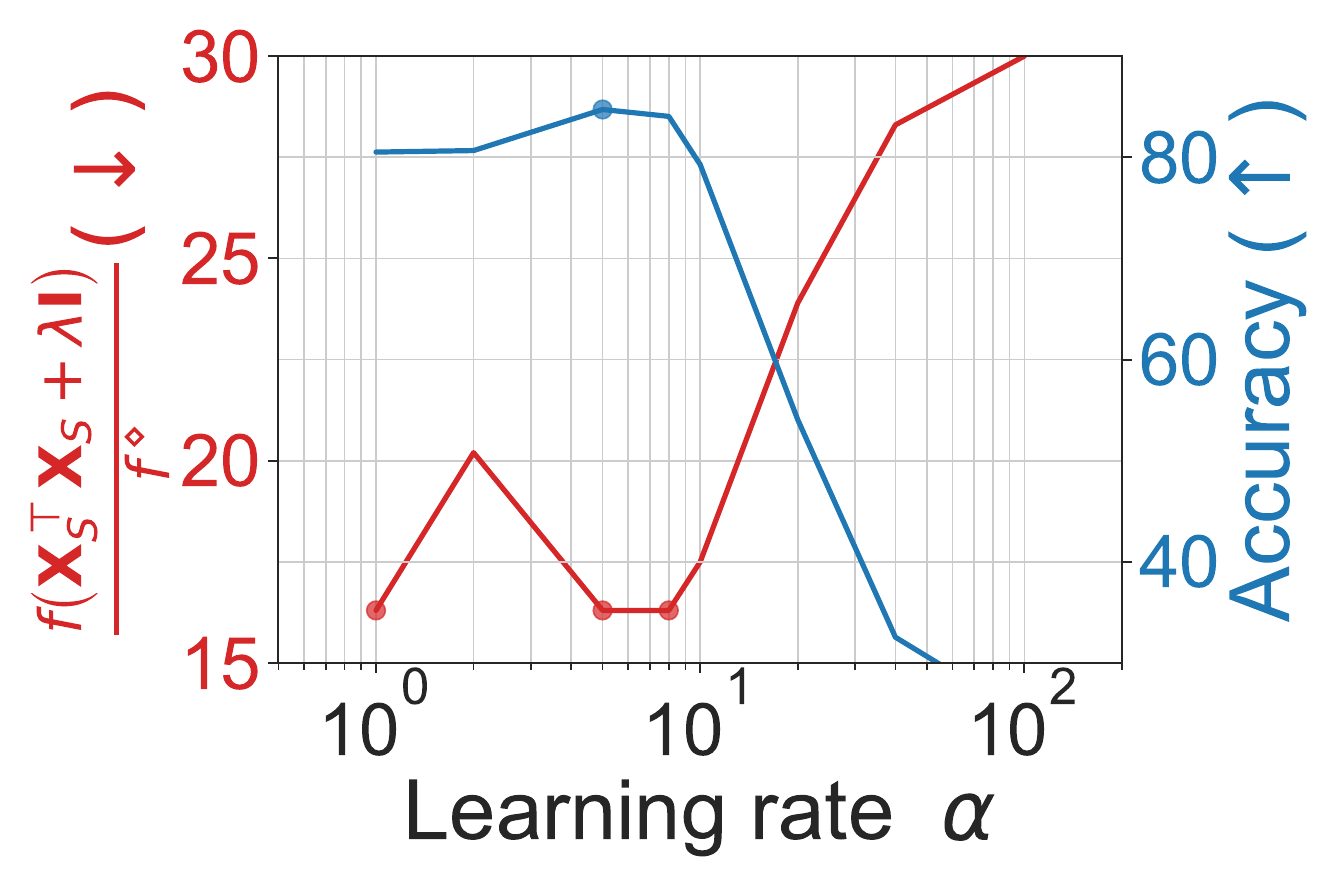}};
\node[] at (0.2,2.1) {\color{black}\regregretent};
\node[] at (6.2,2.1) {\color{black}\regregretlhalf};
\node[rotate=90,anchor=north] at (-3.4, 0.2) {\color{black}\textbf{$k=60$}};
\node[rotate=90,anchor=north] at (-3.4, -3.4) {\color{black}\textbf{$k=80$}};
\node[rotate=90,anchor=north] at (-4, -1.8) {\color{black}\textbf{CIFAR-10}};
\node[inner sep=0pt] (a3) at (0,-7.4) {\includegraphics[width=5.5cm]{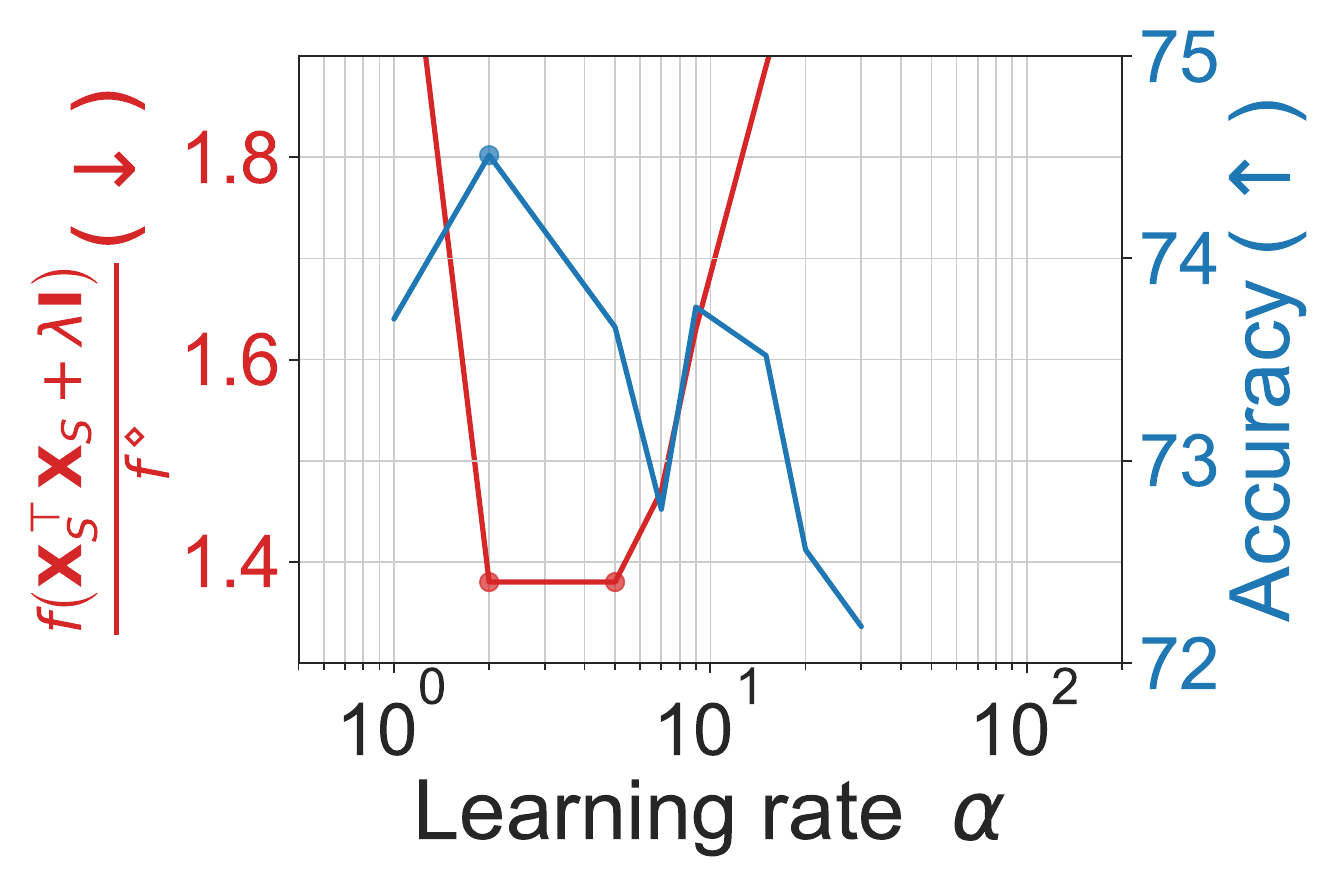}};
\node[inner sep=0pt] (b3) at (6,-7.4) {\includegraphics[width=5.5cm]{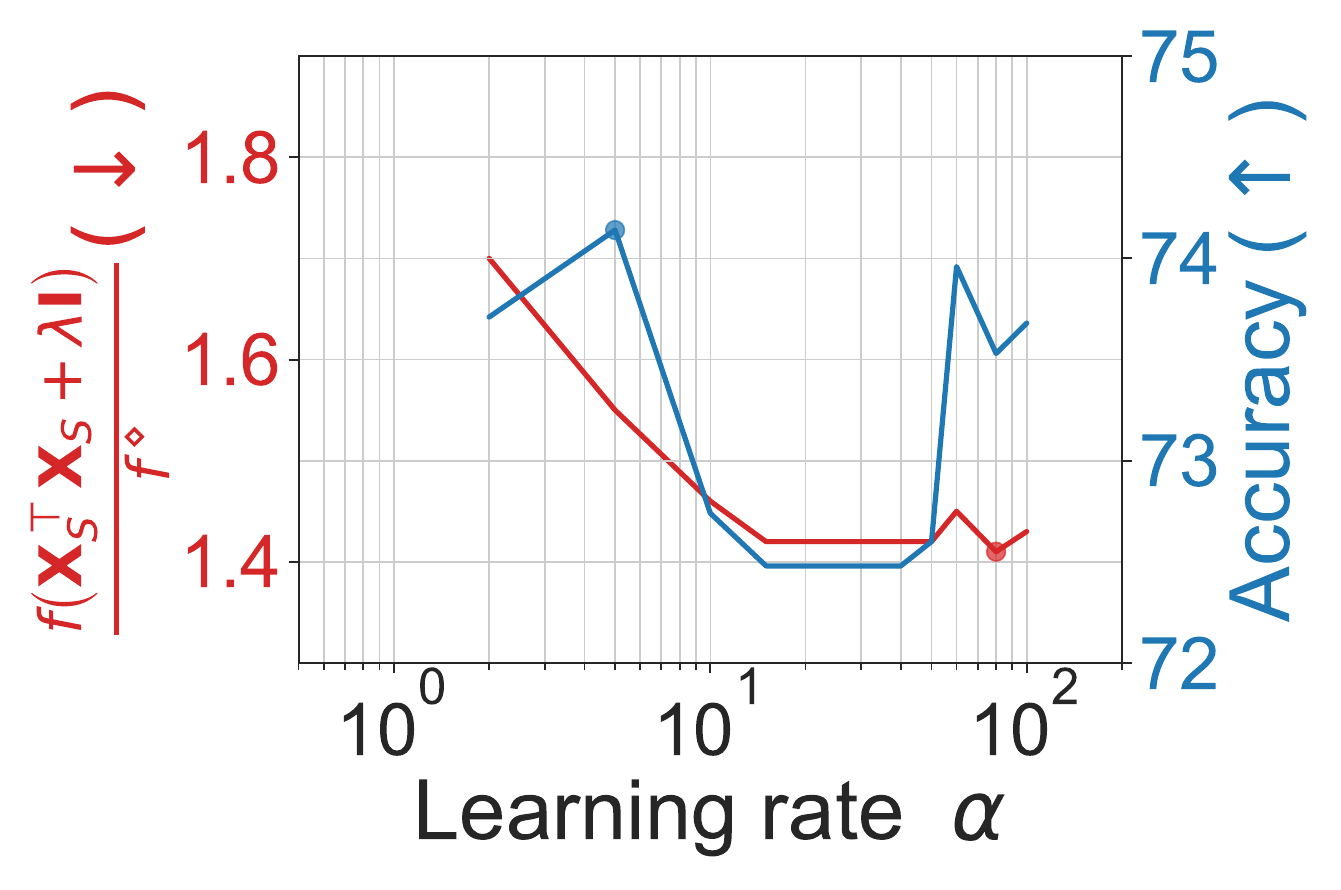}};
\node[inner sep=0pt] (a4) at (0,-11) {\includegraphics[width=5.5cm]{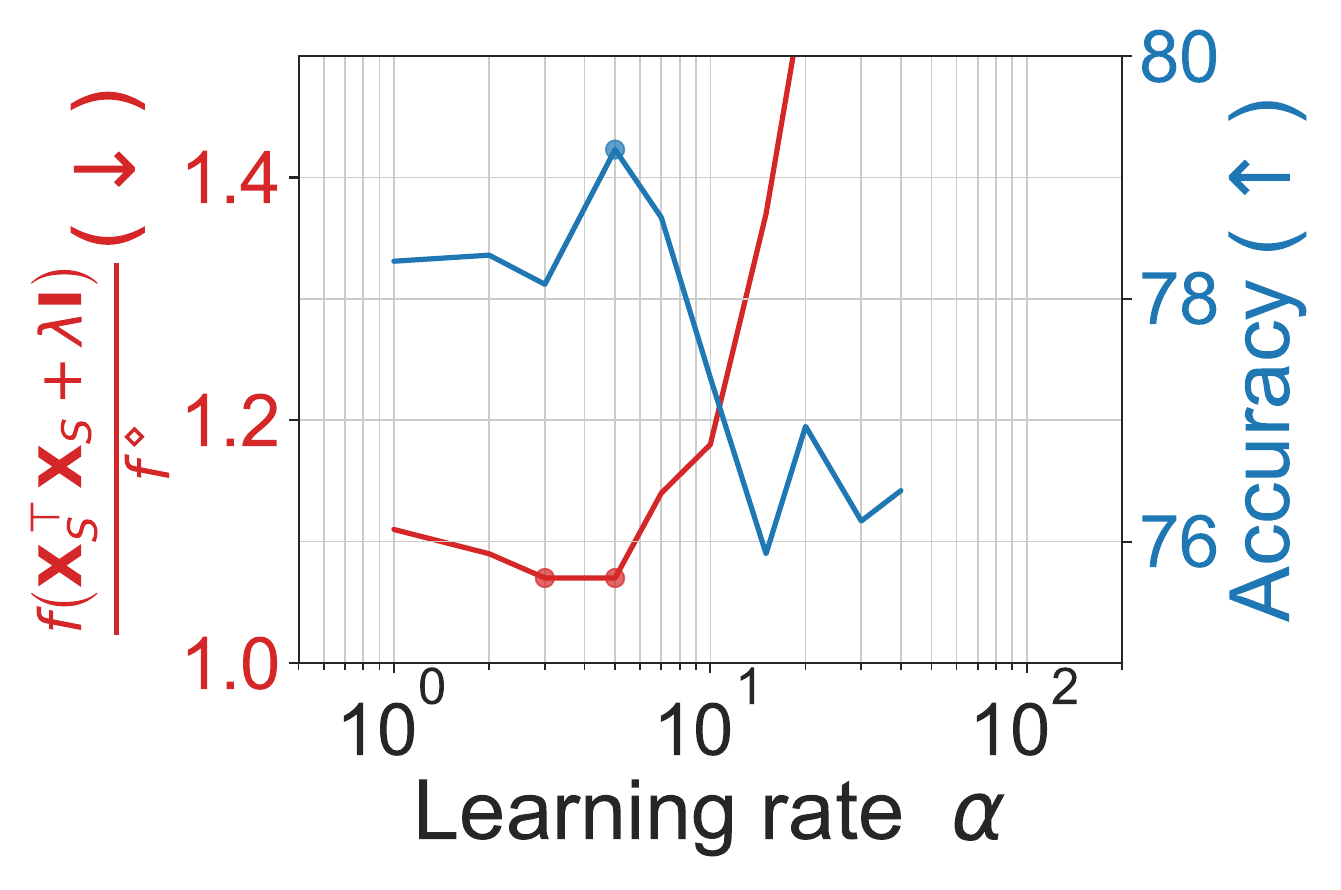}};
\node[inner sep=0pt] (b4) at (6,-11) {\includegraphics[width=5.5cm]{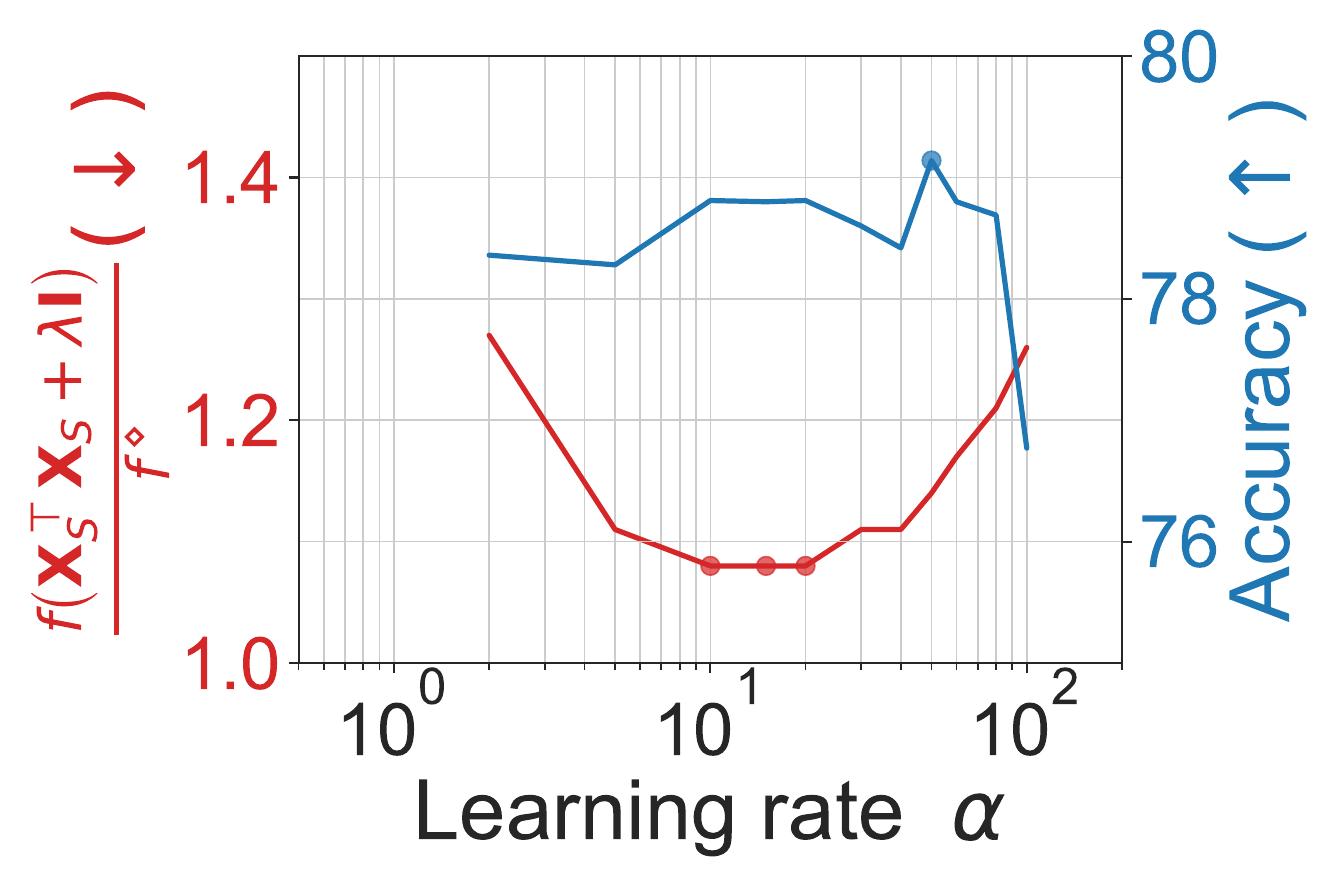}};
\node[rotate=90,anchor=north] at (-4, -8.5) {\color{black}\textbf{ImageNet-50}};
\node[rotate=90,anchor=north] at (-3.4, -7.2) {\color{black}\textbf{$k=50$}};
\node[rotate=90,anchor=north] at (-3.4, -10.5) {\color{black}\textbf{$k=100$}};
\draw[] (-4, -5.4) -- (8.7,-5.4);
\draw[] (3,2.4) -- (3,-12.8);
\draw[] (-2.9,2.4) -- (-2.9,-12.8);
\draw (-4,-12.8) rectangle (8.7,1.8);
\draw (-4,1.8) rectangle (8.7,2.4);
\end{tikzpicture}
\caption{\color{black}Comparison of \regregretent and \regregretlhalf (\cref{algo:rounding-regularize}) for A-design on CIFAR-10 (upper) and ImageNet-50 (lower). For CIFAR-10, the PCA-reduced features with dimension of 100 are used, for ImageNet-50, the PCA-reduced features with dimension of 50 are used. The red lines in the plot represent the relative value of the objective function $\frac{f(\X_S^\top \X_S+\lambda \b I)}{f^\diamond}$, where $\X_S$ is the selected samples and $f^\diamond$ is the optimal value of the relaxed problem~\cref{eq:reg-lp}. The blue lines represent the logistic regression prediction accuracy. The dots on each line represent the optimal points.}
\label{fig:reg-compare-regularized}
\end{figure}

\paragraph{Different optimal design criterion} 
We use CIFAR-10 as an example and compare the performance of Regret-Min when using different objective functions for optimal design. Specifically, we compare (A-, D-, E-, V-, G-) designs and present the results in \cref{fig:cifar10-compare-obj}. The plots display the logistic regression prediction accuracy using Laplacian eigenvectors and PCA features with a dimension of 40. The discrepancies between various design objectives are relatively small, with the largest accuracy difference being less than 2\%. When Laplacian eigenvectors are used, A-design and V-design demonstrate similar performance. While A-design and E-design exhibit similar performance  when PCA features are used. None of the design objectives outperforms the others in all tests, indicating that there is no single objective that consistently yields superior results. An interesting observation is that E-design shows the poorest performance when $k=40$, but it demonstrates the best performance when $k=100$ when Laplacian eigenvectors are utilized. Similarly, a similar pattern is observed for V-design when PCA features are used.

\begin{figure}[tbp]
\centering
  \footnotesize
\begin{tikzpicture}
\node[inner sep=0pt] (a1) at (0,0) {\includegraphics[width=5.7cm,height = 4cm]{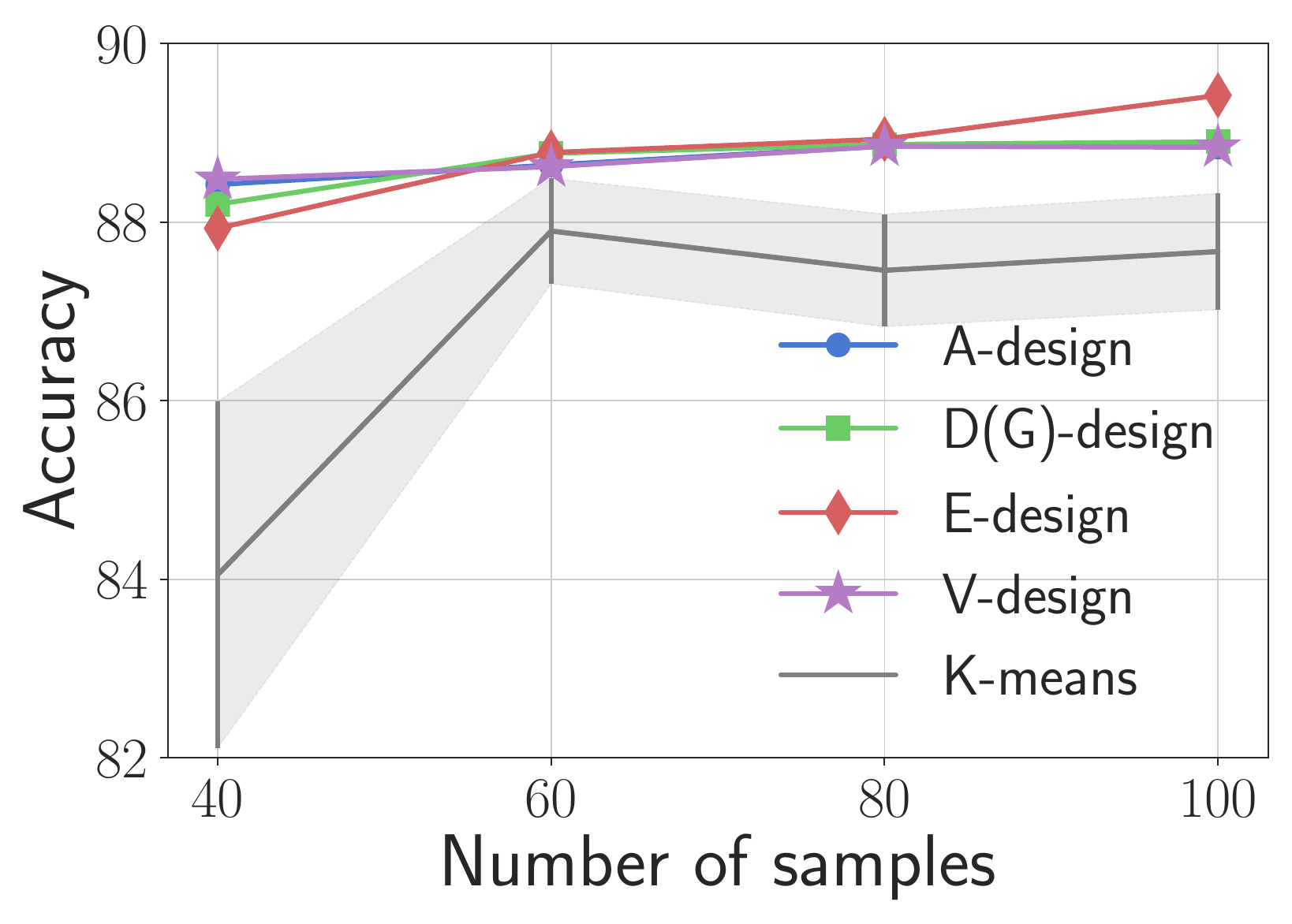}};
\node[inner sep=0pt] (a2) at (7,-0.07) {\includegraphics[width=5.7cm,height = 3.95cm]{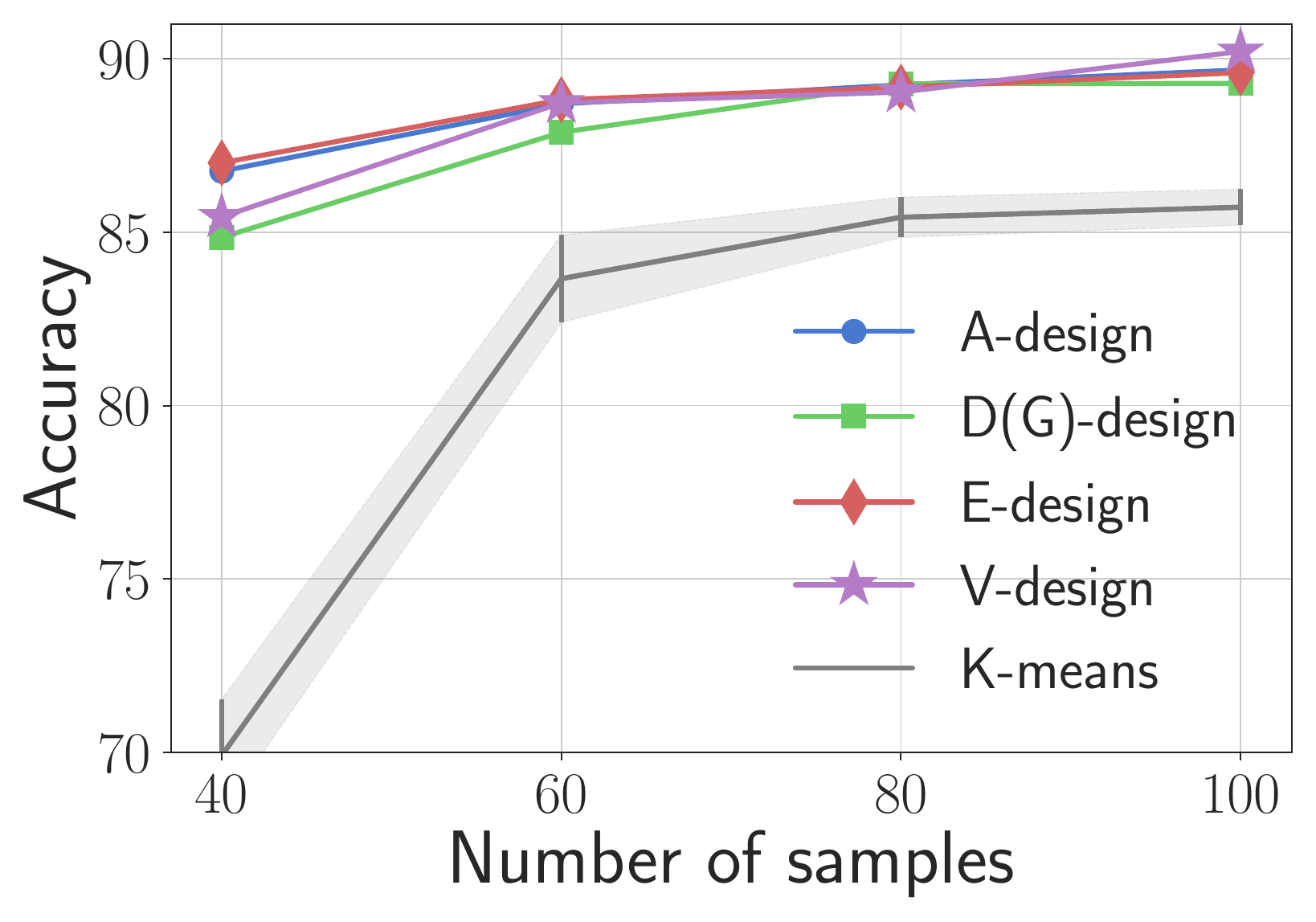}};
\node[] at (0.4,2.2) {\textbf{Laplacian eigenvectors ($d=40$)}};
\node[] at (7.1,2.2) {\textbf{PCA features ($d=40$)}};
\end{tikzpicture}
\caption{Comparison of using different experimental optimal design objectives on CIFAR-10. Left (Right) figure plots the accuracy of logistic regression prediction on unlabeled training data using Laplacian eigenvectors (PCA features) with dimension of 40.}
\label{fig:cifar10-compare-obj}
\end{figure}

\begin{figure}[!t]
\centering
  \footnotesize
\begin{tikzpicture}
\node[inner sep=0pt] (a1) at (0,0) {\includegraphics[width=11cm,height = 1.0cm]{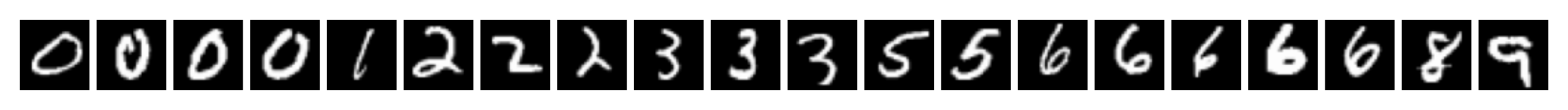}};
\node[inner sep=0pt] (a2) at (0,-.9) {\includegraphics[width=11cm,height = 1.0cm]{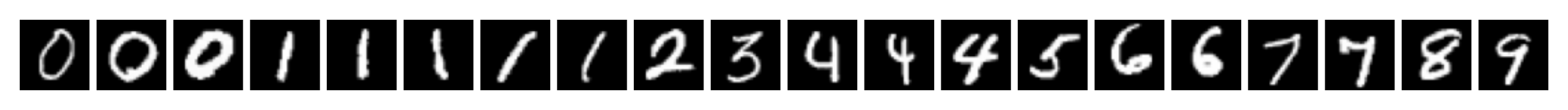}};
\node[inner sep=0pt] (a3) at (0,-1.8) {\includegraphics[width=11cm,height = 1.0cm]{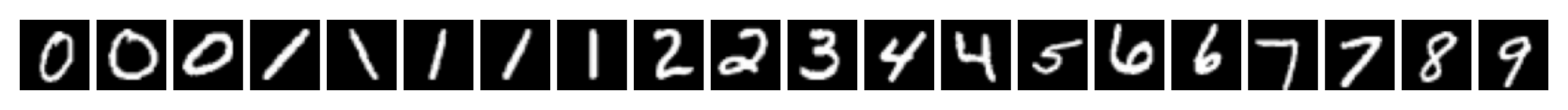}};
\node[inner sep=0pt] (a4) at (0,-2.7) {\includegraphics[width=11cm,height = 1.0cm]{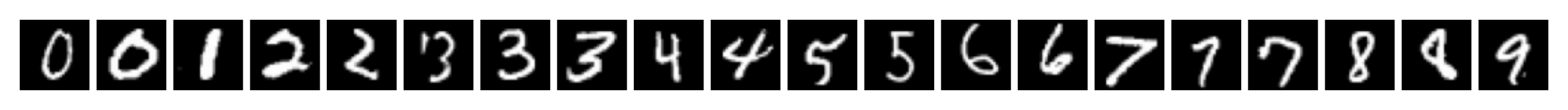}};
\node[inner sep=0pt] (a5) at (0,-3.6) {\includegraphics[width=11cm,height = 1.0cm]{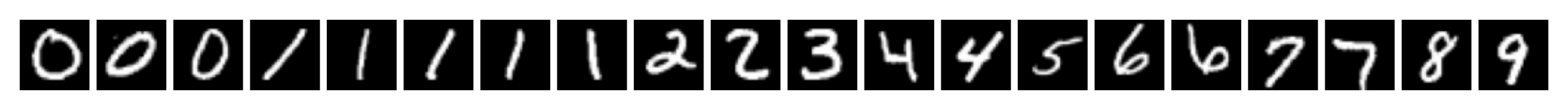}};
\node[]  at (-6.5,0) {Uniform};
\node[]  at (-6.5,-0.9) {K-Means};
\node[]  at (-6.5,-1.8) {RRQR};
\node[]  at (-6.5,-2.7) {MMD};
\node[]  at (-6.5,-3.6) {Regret-Min};
\end{tikzpicture}
\caption{20 samples selected by different sampling methods on MNIST ($d=20$).}
\label{fig:mnist}
\end{figure}

\begin{table}[!t]
\setlength{\tabcolsep}{4pt}
\scriptsize
\centering
      \caption{Logistic regression accuracy of different sampling methods on MNIST ($d=20$).}
  \label{table:mnist}
  \centering
  \begin{tabular}{cccccDEcc}
    \toprule
\# &Uniform&K-Means & RRQR  &MMD &Max-weights &Weighted-sampling & Greedy &Regret-Min \\\midrule
20 & 64.65$\pm$6.71 &88.93$\pm$2.48  &91.76 &85.37 &90.50 &85.71$\pm$3.84 &90.45$\pm$ 1.28 &\textbf{92.02} \\
40    &84.14$\pm$3.75 &91.14$\pm$0.99   &\textbf{92.53}  &88.91 &90.33 &90.34$\pm$0.03 &90.64$\pm$1.61 &\textbf{92.35}\\
60    &88.77$\pm$2.65   &\textbf{93.17}$\pm$0.75 &\textbf{93.02}  &89.37&88.93 &88.95$\pm$0.06 &91.42$\pm$1.21 &\textbf{93.12}\\
    \bottomrule
  \end{tabular}
\end{table}

\begin{table}[tbp]
\setlength{\tabcolsep}{4pt}
\scriptsize
\centering
      \caption{Logistic regression accuracy of different sampling methods on CIFAR-10.}
  \label{table:cifar10-logistic}
  \centering
  \begin{tabular}{cccccDEcc}
    \toprule
 \# &Uniform &K-Means & RRQR  &MMD &Max-weights &Weighted-sampling & Greedy &Regret-Min \\\midrule
\multicolumn{9}{c}{Laplacian eigenvectors (dimension: 10)} \\
\rowcolor{Gray}10 &50.72 $\pm$ 8.91     &71.37$\pm$ 1.25  &72.91  &42.24   &78.40 &72.54$\pm$5.16 &74.20$\pm$4.78 &\textbf{79.80} \\ 
\rowcolor{Gray}20 &65.38 $\pm$ 6.54     &80.18$\pm$ 0.60  &77.31  &66.32  &79.27 &79.44$\pm$0.21 & 77.80$\pm$2.60&\textbf{81.34} \\
\rowcolor{Gray}40 &76.43 $\pm$ 4.27     &80.96$\pm$ 1.11  &77.32 &\textbf{82.95} &77.00 & 76.66$\pm$1.02&80.95$\pm$1.13 &\textbf{82.48} \\
 \multicolumn{9}{c}{PCA features (dimension: 10)} \\
\rowcolor{LightCyan}10 &38.06 $\pm$ 6.27    &54.60$\pm$ 2.61  &51.58 &48.01  &48.39 &40.12$\pm$6.20 &48.06$\pm$6.05 &\textbf{55.60} \\
\rowcolor{LightCyan}20 &48.80 $\pm$ 4.79   &60.76$\pm$ 3.07  &49.78 &59.09  &57.66 &57.27$\pm$2.04 &58.45$\pm$2.90 &\textbf{62.30}\\   
\rowcolor{LightCyan}40 &62.17 $\pm$ 2.51  &\textbf{69.52}$\pm$ 1.89  &57.86 &66.53 &61.12 &61.10$\pm$0.15 &65.22$\pm$1.42 &67.09\\ \midrule
 \multicolumn{9}{c}{Laplacian eigenvectors (dimension: 40)} \\
\rowcolor{Gray}40 &65.01 $\pm$ 4.22    &84.05$\pm$1.94  &87.19 &72.92   &74.35 &68.52$\pm$5.45 &86.05$\pm$2.13 &\textbf{88.42} \\ 
\rowcolor{Gray}60 &73.38 $\pm$ 3.18   &87.90$\pm$0.59  &86.56 &82.81  &80.03 &75.39$\pm$3.15 &86.07$\pm$2.16 &\textbf{88.64} \\ 
\rowcolor{Gray}80 &78.54$\pm$ 2.78   &87.46$\pm$0.63 &88.04 &84.31 &83.09 &82.15$\pm$2.28 &86.96$\pm$1.27 &\textbf{88.85} \\
\rowcolor{Gray}100 &80.96 $\pm$ 3.00   &87.67$\pm$0.65 &\textbf{88.45} &86.86 &85.60 &85.52$\pm$2.01 &87.27$\pm$1.35 &\textbf{88.84} \\
\multicolumn{9}{c}{PCA features (dimension: 40)} \\
\rowcolor{LightCyan}40 &66.08 $\pm$ 4.07   &69.91$\pm$ 1.62  &82.97 &82.38 &78.57 &71.70$\pm$3.87 &79.94$\pm$3.46 &\textbf{86.76} \\
\rowcolor{LightCyan}60 &74.65 $\pm$ 2.49   &83.66$\pm$ 1.26  &81.49 &83.84 &86.66 &82.65$\pm$2.18 &85.93$\pm$0.81 &\textbf{88.70} \\
\rowcolor{LightCyan}80 &79.60 $\pm$ 2.27  &85.43$\pm$ 0.58  &85.59 &87.24 &88.33 &87.34$\pm$0.86 &87.20$\pm$0.40 &\textbf{89.25} \\
\rowcolor{LightCyan}100 &82.19 $\pm$ 2.10  &85.72$\pm$ 0.52 &86.27 &87.67 &88.96 &88.57$\pm$0.84 &87.72 $\pm$0.44 &\textbf{89.68} \\\midrule
\multicolumn{9}{c}{Laplacian eigenvectors (dimension: 100)} \\
\rowcolor{Gray} 100 &68.88 $\pm$ 2.96    &83.12$\pm$1.39  &87.05 &80.21  &83.52 &77.52$\pm$1.32 &86.78$\pm$0.81 &\textbf{87.85} \\ 
\rowcolor{Gray}200 &80.16 $\pm$ 1.98   &85.50$\pm$0.79  &87.84&85.77 &86.67 &86.69$\pm$0.09 &87.43$\pm$0.91 &\textbf{88.64} \\ 
\rowcolor{Gray}500 &87.07 $\pm$ 0.97  &88.75$\pm$0.17 &\textbf{89.18} &88.91  &88.43 &88.34$\pm$0.14 &88.70$\pm$0.66 &\textbf{89.71} \\
\multicolumn{9}{c}{PCA features (dimension: 100)} \\
\rowcolor{LightCyan}100 &71.98 $\pm$ 2.12    &84.77$\pm$0.77 &72.83 &81.50  &82.72 &76.42$\pm$2.29 &84.39$\pm$1.92 &\textbf{88.51}\\
\rowcolor{LightCyan}200 &82.52 $\pm$ 1.78   &87.04$\pm$0.54  &83.15 &87.58  &89.26&88.28$\pm$0.81 &88.81$\pm$0.59 &\textbf{90.11}\\
\rowcolor{LightCyan}500 &88.27 $\pm$ 0.61    &89.12$\pm$0.44 &85.84 &89.21 &\textbf{90.58} &\textbf{90.52}$\pm$0.12 &\textbf{90.83}$\pm$0.52 &\textbf{90.98}\\
 \bottomrule
  \end{tabular}
\end{table}

\begin{table}[tbp]
\scriptsize
\setlength{\tabcolsep}{4pt}
\centering
      \caption{Logistic regression accuracy of different sampling methods on CIFAR-10 (100 features using PCA). For Regularized-Regret-Min, we use regularization parameter $\lambda = 10^{-5}\times  k$ (where $k$ is the number of samples to be selected). }
  \label{table:cifar10-reg}
  \centering
  \begin{tabular}{cccccDEcG}
    \toprule
\#&Uniform &K-Means & RRQR  &MMD &Max-weights &Weighted-sampling & Greedy &Regularized-Regret-Min \\\midrule
20 &32.14$\pm$ 3.24 &46.77$\pm$2.92 &22.46 &\textbf{49.83} &33.63 &33.64$\pm$3.18 &43.88  & 45.37\\
40 &50.06 $\pm$2.83 &\textbf{69.97}$\pm$ 1.69 &47.58 &66.59  &50.86 &48.66$\pm$3.74 &62.78  &67.02\\
60 &60.34$\pm$2.56 &78.69$\pm$1.55 &64.40 &77.19 &67.47 &65.02$\pm$2.73 &73.03  &\textbf{81.14}\\
80 &67.11$\pm$2.39 &83.09$\pm$0.99 &81.66 &80.826 &77.97 &70.91$\pm$2.51 &76.87  &\textbf{86.30}\\
 \bottomrule
  \end{tabular}
\end{table}

\begin{table}[tbp]
\scriptsize
\setlength{\tabcolsep}{4pt}
\centering
      \caption{FixMatch prediction accuracy on unlabeled training data of CIFAR-10.}
  \label{table:cifar10-fixmatch}
  \centering
  \begin{tabular}{cccccc}
    \toprule
\#samples &Uniform &K-Means & RRQR  &MMD &Regret-Min \\\midrule
 \multicolumn{6}{c}{Laplacian eigenvectors} \\
\rowcolor{Gray}10 &46.50 $\pm$ 13.48    &52.85 $\pm$ 4.20  &\textbf{74.60} &45.30 &73.50 \\ 
\rowcolor{Gray}40 &75.44 $\pm$ 8.95   &77.30 $\pm$ 2.33  &\textbf{93.18} &90.08  &88.41 \\ 
\rowcolor{Gray}100 &90.69$\pm$2.86 &88.29$\pm$0.47 &87.55 &\textbf{93.11}  &\textbf{93.28} \\
\multicolumn{6}{c}{PCA features} \\
\rowcolor{LightCyan}10 &46.50 $\pm$ 13.48    &57.62 $\pm$ 1.33  &28.57 &58.10 &\textbf{64.99} \\ 
\rowcolor{LightCyan}40 &75.44 $\pm$ 8.95   &\textbf{91.20} $\pm$ 1.21  &74.44 &79.80 &\textbf{91.56} \\ 
\rowcolor{LightCyan}100 &90.69$\pm$2.86 &92.43$\pm$0.84 &\textbf{93.52}  &91.64  &\textbf{93.47} \\
 \bottomrule
  \end{tabular}
\end{table}

\begin{figure}[tbp]
\centering
  \footnotesize
\begin{tikzpicture}
\node[inner sep=0pt] (a1) at (0,-.3) {\includegraphics[width=6cm,height=5cm]{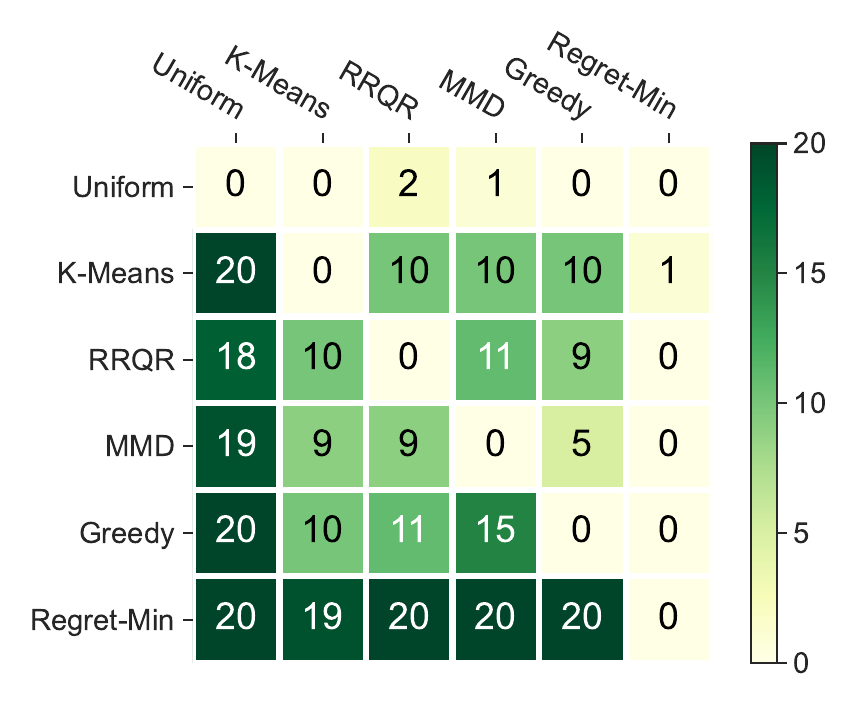}};
\end{tikzpicture}
\caption{\color{black}Pairwise comparison of different sampling methods in  tests of CIFAR-10 in \cref{table:cifar10-logistic}. For element in the $i$-th row and $j$-th column, the value indicates how many times method $i$ outperforms method $j$ out of 20 total tests.}
\label{fig:cifar10-compare}
\end{figure}

\begin{table}[tbp]
\scriptsize
\setlength{\tabcolsep}{4pt}
\centering
      \caption{Logistic regression accuracy of different sampling methods on class-balanced and class-imbalanced ImageNet-50 ($d=50$).}
  \label{table:image50-100}
  \centering
  \begin{tabular}{cccccDEcc}
    \toprule
\# &Uniform &K-Means & RRQR  &MMD &Max-weights &Weighted-sampling & Greedy  &Regret-Min \\
\midrule
 \multicolumn{9}{c}{Balanced} \\
\rowcolor{Gray}50 &48.04$\pm$ 3.73 &71.48$\pm$1.89 &73.84  &58.66 &63.71 &53.35$\pm$4.57 &73.53$\pm$1.03  &\textbf{77.18}  \\
\rowcolor{Gray}100    &63.84$\pm$3.47 &81.17$\pm$0.58    &{\color{black}76.30}  & 75.67&78.16 &75.38$\pm$1.09 &78.14$\pm$1.40   &\textbf{81.87} \\
\rowcolor{Gray}  150  &72.10$\pm$2.62     &\textbf{83.80}$\pm$0.61  &{\color{black}78.74}    &80.90  &80.65 &79.51$\pm$1.09  &80.08$\pm$1.28   &82.93  \\
 \multicolumn{9}{c}{Imbalanced} \\
\rowcolor{LightCyan}50 &46.86$\pm$ 3.33 &69.69$\pm$1.86 &73.27 &54.00&67.27 &56.92$\pm$3.04 &74.10$\pm$0.86   &\textbf{77.18} \\
\rowcolor{LightCyan}100    &62.28$\pm$3.17 &80.29$\pm$1.18     &{\color{black}74.74}    &69.96 &78.98  &77.32$\pm$1.80  &77.89$\pm$1.47 &\textbf{81.88} \\
\rowcolor{LightCyan}   150  &70.13$\pm$3.20    &81.80$\pm$0.63  &{\color{black}75.84} &76.53 &80.13  &79.90$\pm$0.27 &79.90$\pm$1.32    &\textbf{82.81} \\
    \bottomrule
  \end{tabular}
\end{table}

\begin{table}[tbp]
\scriptsize
\setlength{\tabcolsep}{4pt}
\centering
      \caption{Number of classes covered by the 50 samples selected by different methods on ImageNet-50 (50 classes in total).}
  \label{table:image50-100-classnumber}
  \centering
  \begin{tabular}{cccccDEcc}
    \toprule
\# &Uniform &K-Means & RRQR  &MMD &Max-weights &Weighted-sampling & Greedy  &Regret-Min \\
\midrule
 \multicolumn{9}{c}{Balanced} \\
\rowcolor{Gray} 50 &31.6$\pm$ 2.7 &41.7 $\pm$ 0.9  &44 &41 &40 &32.6$\pm$1.5 & 44.1$\pm$0.5  &\textbf{45}\\
 \multicolumn{9}{c}{Imbalanced} \\
\rowcolor{LightCyan} 50&29.8$\pm$ 2.9 &39.9$\pm$1.1  &41  &32 &39 &34.3$\pm$2.3 &42.8$\pm$0.75 &\textbf{45}\\
    \bottomrule
  \end{tabular}
\end{table}
\paragraph{Regret-Min vs. other methods} Regret-Min outperforms other methods in most of our experiments. The details for the comparison on each dataset are as follows.

\begin{itemize}[leftmargin=*]
\item MNIST: We compare various sampling methods for selecting $k \in \{20, 40, 60\}$ samples to label. The logistic regression accuracy on the whole unlabeled training set is reported in \cref{table:mnist}. Regret-Min outperforms the other methods when the sample size is 20 (i.e., 2 labeled samples per class). For sample sizes of 40 and 60, K-Means, {\color{black}RRQR}, and Regret-Min achieve comparable performance. The selected sets of 20 samples from each method are shown in \cref{fig:mnist}. Except for Uniform sampling, all methods successfully select samples that cover all data classes.

\item CIFAR-10: In \cref{table:cifar10-logistic}, we present the results of logistic regression prediction accuracy on unlabeled training data using various sampling methods. We evaluate the performance for $d\in\{10, 40, 100\}$. Regret-Min consistently outperforms the other methods in most of the tests. Additionally, when $d=40$ or 100 and the data utilizes eigenvectors of normalized graph Laplacian, RRQR shows similar performance to Regret-Min. Furthermore, in \cref{fig:cifar10-images}, we display the images of the 40 CIFAR-10 samples selected  by different methods.

We further evaluate the regularized version of our algorithm \cref{algo:rounding-regularize}, which is tailored for the optimal design for ridge regression setting.  We consider the PCA features of CIFAR-10 with $d=100$, denoted as $\X \in\mathbb{R}^{50000\times 100}$. Our goal is to compare various methods for selecting a subset of samples with sizes $k\in\{20, 40, 60, 80\}$, where $k < d$. In such regime, $\X_S^\top \X_S$ becomes rank-deficient, rendering the original optimal design problem ill-posed. Introducing regularization resolves this issue by incorporating a penalty term. The logistic regression results are summarized in \cref{table:cifar10-reg}. For $k=20$, MMD achieves the best performance. When $k=40$, K-Means outperforms other methods, with Regret-Min ranking second. Notably, Regret-Min achieves the best performance for $k=60$ and 80. Combining these findings with our previous CIFAR-10 experiments, we observe that Regret-Min consistently performs best when $k$ approaches or exceed $d$.

To further evaluate the representativeness of the selected samples, we employ FixMatch~\cite{sohn2020fixmatch}, a state-of-the-art semi-supervised learning method. The prediction accuracy on the unlabeled training data is reported in \cref{table:cifar10-fixmatch}. When sample selection is based on PCA features, Regret-Min achieves the best performance. In contrast, when using Laplacian eigenvectors for sample selection, RRQR performs best for $k=10$ and 40. However, it is worth noting that RRQR yields the worst performance when PCA features with dimensions 10 or 40 are used for selection.

In order to conduct a comprehensive comparison of the sampling methods, we aggregate the results of all the 30 tests on CIFAR-10  in \cref{table:cifar10-logistic,table:cifar10-reg,table:cifar10-fixmatch}. For each test, we rank the five methods and generate a pairwise comparison matrix, as depicted in \cref{fig:cifar10-compare}. It is evident from the matrix that Regret-Min consistently outperforms the other methods overall.


\item ImageNet-50:  \cref{table:image50-100} reports logistic regression accuracy on both class-balanced and class-imbalanced versions of the ImageNet-50 dataset, for varying numbers of selected samples $k \in \{50, 100, 150\}$.

On the class-balanced dataset, Regret-Min outperforms all other methods when $k = 50$ and $100$. At $k = 150$, K-Means achieves the highest accuracy, with Regret-Min ranking second. On the class-imbalanced dataset, Regret-Min consistently outperforms all baselines across all values of $k$. Interestingly, the performance of Regret-Min, Greedy, and Weighted Sampling remains stable when transitioning from the balanced to the imbalanced dataset. In contrast, Uniform, K-Means, and MMD exhibit a noticeable drop in performance under class imbalance.

For the $k=50$ setting, \cref{table:image50-100-classnumber} shows the number of classes represented in the selected samples for each method. Regret-Min consistently selects samples covering the largest number of classes in both the balanced and imbalanced datasets compared to other methods.

 We provide the full list of ImageNet-50 class names in \cref{table:image50-classname}. Additionally, \cref{fig:imagenet-fig-1,fig:imagenet-fig-2,fig:imagenet-fig-3} visualize the 100 samples selected by different methods. All images are center-cropped to a size of (224, 224, 3) from the original ImageNet dataset.

\end{itemize}

\section{Conclusions}
In this work, we explored extensions and applications of the regret minimization-based approach introduced in~\cite{design} for solving the optimal design problem (referred to as "Regret-min"). First, we extended Regret-min to incorporate the entropy regularizer by deriving the corresponding sparsification objective and establishing sample complexity bounds required to achieve near-optimal performance. Our analysis shows that this complexity matches that of the original $\ell_{1/2}$-regularized approach proposed in~\cite{design}. Experimental results on synthetic datasets confirmed that both regularizers perform similarly. However, on real-world classification tasks, the entropy regularizer demonstrated better performance in tuning the learning rate compared to the $\ell_{1/2}$ regularizer.

We further extended the Regret-min framework to address the regularized optimal design problem in the context of ridge regression. For both the entropy and $\ell_{1/2}$ regularizers, we derived the appropriate sparsification objectives and established sample complexity guarantees for achieving near-optimal performance.

Finally, we applied the Regret-min method to select representative samples from unlabeled datasets for classification tasks. Through experiments on real-world image datasets, we compared Regret-min against other representative selection methods. The results showed that Regret-min outperforms other methods in most of the tests.

\section*{Acknowledgments}
This material is based upon work supported by NSF award OAC 22042261; and  Cooperative Agreement 2421782 and the Simons Foundation award MPS-AI-00010515 (NSF-Simons AI Institute for Cosmic Origins---CosmicAI, https://www.cosmicai.org/); by the U.S. Department of Energy, Office of Science, Office of Advanced Scientific Computing Research, Applied Mathematics program, Mathematical Multifaceted Integrated Capability Centers (MMICCS) program, under award number DE-SC0023171; by the U.S. Department of Energy, National Nuclear Security Administration Award Number DE-NA0003969; and by the U.S. National Institute on Aging under award number  R21AG074276-01. Any opinions, findings, and conclusions or recommendations expressed herein are those of the authors and do not necessarily reflect the views of the DOE, NIH, and NSF. Computing time on the Texas Advanced Computing Centers Stampede system was provided by an allocation from TACC and the NSF.

\appendix

\section{Background of optimal experimental design with linear regression}\label{sec:design-meaning}

We consider the linear regression model introduced in \Cref{sec:linear-regression}. Define the error of weights by $\Delta\what\triangleq \what - \w_*$. By~\cref{eq:ls-what}, we have
\begin{align}
    \Delta\what = (\X_S^\top \X_S)^{-1} \X_S^\top \bepsilon_S, \quad \Delta\what \Delta\what^\top = (\X_S^\top \X_S)^{-1} \X_S^\top \bepsilon_S\bepsilon_S^\top \X_S (\X_S^\top \X_S)^{-1}.\nonumber
\end{align}
Thus the expectation of the error on weights is $\E[\Delta\what] = \b 0$. Since we assume the noise for  observation is independent, we have $\E[\bepsilon_S \bepsilon_S^\top] = \sigma^2 \b I$.
Thus the covariance matrix of $\Delta\what$ is
\begin{align}\label{eq:ls-cov-weights}
    \E_{\bepsilon_S}[\Delta\what \Delta\what^\top]& = (\X_S^\top \X_S)^{-1} \X_S^\top \E[\bepsilon_S\bepsilon_S^\top] \X_S (\X_S^\top \X_S)^{-1}\nonumber\\
    & = \sigma^2 (\X_S^\top \X_S)^{-1}\X_S^\top \X_S  (\X_S^\top \X_S)^{-1} = \sigma^2  (\X_S^\top \X_S)^{-1}.
\end{align}

(A-,D-,E-) optimal design objectives aim to select $S$ to minimize different scalar functions on the covariance matrix of the weights error, i.e. \cref{eq:ls-cov-weights}. In particular,
\begin{itemize}
    \item A-optimal design minimizes the trace of the covariance matrix, i.e. $\Tr\big((\X_S^\top \X_S)^{-1}\big)$. Note that this is equivalent to minimizes the $\ell_2$-norm of the error on weights since $\E_{\bepsilon_S}[\|\Delta\w \|_2^2] = \sigma^2\Tr\big((\X_S^\top \X_S)^{-1}\big)$.
    \item D-optimal design minimizes the log-determinant of  the covariance matrix, i.e. $\log \det (\X_S^\top \X_S)^{-1} = -\log \det (\X_S^\top \X_S)$.
    \item E-optimal design minimizes the spectral norm of covariance matrix, i.e. $ \| (\X_S^\top \X_S)^{-1}\|_2$.
\end{itemize}

Statistical meaning of (G-,V -) optimal design is related to the error on the predictions of the observations for all  points in the design matrix $\X$. For any point $\x_i$ ($i\in[n]$), the prediction of observation is $\yhat_i = \what^\top \x_i$, the error on predicted observation is $\Delta \yhat_i =\yhat_i - \w_*^\top \x_i = \Delta \what^\top \x_i $. Then the expectation of the square error of prediction for point $\x_i$ is 
\begin{align}\label{eq:ls-yerror2}
    \E_{\bepsilon_S}[\Delta \yhat_i^2]& = \E_{\bepsilon_S}[\x_i^\top \Delta\what \Delta \what^\top \x_i] = \x_i^\top \E_{\bepsilon_S}[ \Delta\what \Delta \what^\top] \x_i\nonumber\\
    &=  \sigma^2 \x_i^\top(\X_S^\top \X_S)^{-1} \x_i
\end{align}
where the last equality follows by~\cref{eq:ls-cov-weights}. 

\begin{itemize}
    \item V-optimal aims to minimize the average square error over all points in $\X$: 
\begin{align}\label{eq:ls-avg-yerror}
    \E_{\bepsilon_S}\bigg[\frac{1}{n}\sum_{i \in [n]}\Delta \yhat_i^2\bigg] = \frac{\sigma^2}{n}\big\langle(\X_S^\top \X_S)^{-1}, \X^\top \X \big\rangle.
\end{align}
    \item G-optimal design selects $S$ such that the maximum square error over all points in $\X$ is minimized, i.e. $\max_{i\in[n]} \big\langle(\X_S^\top \X_S)^{-1}, \x_i^\top \x_i \big\rangle$.
\end{itemize}


\section{Supporting Lemmas}\label{sec:support-lemma}
\begin{lemma}[Theorem 5.39 in \cite{vershynin2011introduction}]\label{lm:ls-subgaussian}
    Let $\A$ be the $n\times d$ matrix whose rows $\A_i$ are independent sub-Guassian isotropic random vectors in $\mathbb{R}^d$. Then for any $t\geq 0$, with probability at least $1-2\exp(-c t^2)$, we have
    \begin{align}
    \|\frac{1}{n}\A^\top \A - \b I\|_2 \leq \max(\gamma, \gamma^2), \quad {\color{black}\text{where }} \gamma = C\sqrt{\frac{d}{n}} + \frac{t}{\sqrt{n}}.
    \end{align} 
    Here $c = c_K$ and $C = C_K$ depend only on the sub-Gaussian norm $K = \max_i \|\A_i\|_{\psi_2}$.
 \end{lemma}
 
\begin{lemma}[Woodbury matrix identity]\label{lm:woodbury} 
Let $\A$ and $\b C$ be invertible matrices of size $n_A\times n_A$ and $n_C\times n_C$, $\b \U$ and $\b V$ are matrices of size $n_A\times n_C$ and $n_C \times n_A$, the following identity holds:
\begin{align}
    (\A + \b U \b C \b V)^{-1} = \A^{-1} - \A^{-1}\b U\big(\b C^{-1} + \b V \A^{-1} \b U\big)^{-1} \b V \A^{-1}.
\end{align}
    
\end{lemma}

\begin{lemma}[Sherman-Morrison matrix identity]\label{lm:sherman-morrison}
$\A \in\mathbb{R}^n$ is invertible matrix, $\b u,\b v\in\mathbb{R}^n$ are column vectors, then $\A + \b u \b v^\to$ is invertible iff $1 + \b v^\top \A^{-1} \b u \neq 0$. In this case,
\begin{align}
    (\A + \b u \b v^\top)^{-1} = \A^{-1} - \frac{\A^{-1} \b u\b v^\top \A^{-1}}{1 + \b v^\top \A^{-1} \b u}
\end{align}
    
\end{lemma}

{\color{black}
\begin{lemma}\label{lm:diameter}
The diameter of a scalar function $w$ over a set $\b S$ is defined as
\[
\mathrm{diam}_w(\b S) := \sup_{a,b \in \b S} w(a) - w(b).
\]
Given a dimension $d$, consider the matrix set
\[
\simplex := \left\{\A \in \mathbb{S}_+^{d} : \Tr(\A) = 1\right\}.
\]
Then the following bounds hold:
\begin{enumerate}[label=(\alph*)]
    \item If $w(\A) = \langle \A, \log \A - \b I \rangle$, then $\mathrm{diam}_w(\simplex) \le \log d$.
    \item If $w(\A) = -2 \Tr(\A^{1/2})$, then $\mathrm{diam}_w(\simplex) \le 2\sqrt{d}$.
\end{enumerate}
\end{lemma}

\begin{proof}
For any $\A \in \simplex$, let $\A = \V \bLambda \V^\top$ be its eigen-decomposition, where $\bLambda = \diag(\lambda_1, \dots, \lambda_d)$ and $\sum_{i=1}^d \lambda_i = 1$.
\begin{enumerate}[label=(\alph*)]
    \item If $w(\A) = \langle \A, \log \A - \b I \rangle$, then
    \begin{align}\label{eq:pf-ent-w}
        w(\A) = \langle \V \bLambda \V^\top, \V (\log \bLambda - \b I) \V^\top \rangle 
        = \sum_{i=1}^d \lambda_i \log \lambda_i - 1.
    \end{align}
    Considering the Shannon entropy $H(x) = -\sum_{i=1}^d x_i \log x_i$ for $x$ in the probability simplex $\{x \in \mathbb{R}^d : x_i \ge 0, \sum_i x_i = 1\}$, Theorem 2.6.4 in \cite{cover} gives the maximum entropy $H_{\max} = \log d$, attained at the uniform distribution $x_i = 1/d$ for all $i$.  
    Hence the minimum of $w(\A)$ is $w_{\min} = -\log d - 1$. For any $\A \in\simplex$,  $w(\A) \le -1$. Therefore,
    \[
        \mathrm{diam}_w(\simplex) = \sup_{\A,\B \in \simplex} w(\A) - w(\B) \le -1 - (-\log d -1) = \log d.
    \]

    \item If $w(\A) = -2 \Tr(\A^{1/2})$, then
    \begin{align}\label{eq:pf-l12-w}
        w(\A) = -2 \sum_{i=1}^d \sqrt{\lambda_i}.
    \end{align}
    Consider $f(x) = -\sum_{i=1}^d \sqrt{x_i}$ for $x$ in the probability simplex. Introducing a Lagrange multiplier $\lambda$ for $\sum_i x_i = 1$, the Lagrangian is
    \[
        \mathcal{L}(x, \lambda) = -\sum_{i=1}^d \sqrt{x_i} + \lambda \left(\sum_{i=1}^d x_i - 1\right).
    \]
    Differentiating with respect to $x_i$ yields
    \[
        \frac{\partial \mathcal{L}}{\partial x_i} = -\frac{1}{2\sqrt{x_i}} + \lambda = 0 
        \quad \implies \quad x_i = \frac{1}{(2\lambda)^2} \ \forall i\in[d].
    \]
    By the simplex constraint, $x_i = 1/d$ for all $i$, which gives the minimum $f_{\min} = -\sqrt{d}$. Since $w(\A) \le 0$ for any $\A \in \simplex$, we obtain
    \[
        \mathrm{diam}_w(\simplex) = \sup_{\A,\B \in \simplex} w(\A) - w(\B) \le 0 - (-2\sqrt{d}) = 2\sqrt{d}.
    \]
\end{enumerate}
\end{proof}
}

\begin{lemma}\label{lm:mtd-eigen-min}
    For any $\A \in \mathbb{S}_{+}^{d}$, $\lambda_{\min}(\A) = \inf_{\U\in \simplex} \langle \U, \A \rangle$.
\end{lemma}
\begin{proof}
    Since $\A \in \mathbb{S}_{+}^{d}$, we have eigendecomposition $\A = \V\bLambda \V^\top$, where $\bLambda=\text{diag}\{\lambda_1,\cdots, \lambda_{d}\}$. Assume that $\lambda_1 \geq \cdots\geq \lambda_{d} \geq 0$ and  $\b v_i$ is the eigenvector asscoiated with eigenvalue $\lambda_i$ for $i\in[d]$.

    We first show $\lambda_{\min}(\A) \geq \inf_{\U\in \simplex} \langle \U, \A \rangle$. Let $\B = \b  v_{d}\b  v_{d} ^\top$, then $\B \succeq \b 0$ and $\Tr(\B)=1$, i.e. $\B\in \simplex$. Thus
    \begin{align}\label{eq:eigen-min-1}
        \inf_{\U\in \simplex} \langle \U, \A \rangle \leq \langle \B, \A \rangle = \b v_{d} ^\top \V \bLambda \V^\top \b  v_{d} = \lambda_{d} = \lambda_{\min}(\A). 
    \end{align}

    On the other hand, for any $\U \in \simplex $, we have
    \begin{align}
        \langle \U,\A\rangle &= \langle \U, \sum_{i\in[d]} \lambda_i \b  v_i \b 
 v_i^\top \rangle = \sum_{i\in[d]} \lambda_i \b 
 v_i^\top \U\b  v_i \nonumber\\
        &\geq \lambda_{d}  \sum_{i\in[d]} \b v_i^\top \U \b v_i  = \lambda_{d}   \langle \U, \V \V^\top \rangle = \lambda_{d}  \Tr(\U) =\lambda_{d}.\label{eq:eigen-min-2}
    \end{align}
    Since \cref{eq:eigen-min-2} holds for any $\U \in \simplex$, then
\begin{align}\label{eq:eigen-min-3}
    \lambda_{\min}(\A) \leq \inf_{\U\in \simplex} \langle \U, \A \rangle.
\end{align}
 Combining \cref{eq:eigen-min-1} and  \cref{eq:eigen-min-3}, we can get $ \lambda_{\min}(\A) =\inf_{\U\in \simplex} \langle \U, \A \rangle$.
     
\end{proof}

\begin{lemma}\label{lm:inequ-0}
    For $x \geq0$, $\frac{x}{1- \exp(-x)} \leq 1 + x$. 
\end{lemma}
\begin{proof}
    For $x\geq0$, we have $1 + x\leq \exp(x)$, thus $-(1+x) \exp(-x) \geq -1$. Thus $(1+x) (1 - \exp(-x)) \geq x$ and $\frac{x}{1-\exp(-x)}\leq 1 + x$. Then the inequality is proved.
\end{proof}

\begin{lemma}\label{lm:inequ-1}
    For any $i\in [n]$, $a_i > 0 $, $b_i >0$, $\pi_i \geq 0$, then $\max_{i \in [n]} \frac{a_i}{b_i} \geq \frac{\sum_{i \in [n] }\pi_i a_i}{\sum_{i \in [n] }\pi_i b_i}$.
\end{lemma}

\begin{lemma}\label{lm:inequ-2}
    For any $i\in [n]$, $a_i \geq 0 $, $b_i \geq0$, then $\sum_{i \in [n]} \frac{a_i}{1 + b_i} \geq {\color{black}\frac{\sum_{i \in [n] } a_i}{1 + \sum_{i \in [n]}b_i}}$.
\end{lemma}

\begin{lemma}\label{lm:inequ-3}
    For any matrices $\A, \B\in \mathbb{S}_{+}^p$, we have
    \begin{align}
        \langle (\b I + \B)^{-1}, \A\rangle\geq \frac{\Tr(\A)}{1 + \Tr(\B)}.
    \end{align}
\begin{proof}
Denote the eigenvalues of matrix $\A$ by $\alpha_1\geq \alpha_2 \geq \cdots \geq \alpha_p\geq 0$ and eigenvalues of matrix $\B$ as $\beta_1 \geq \beta_2 \geq \cdots \geq \beta_p \geq 0$. Then eigenvalues of $(\b I+\B)^{-1}$ are $0 \leq 1+\beta_1)^{-1} \leq (1+\beta_2)^{-1} \leq \cdots \leq (1+\beta_p)^{-1}$. Thus we have
\begin{align}
    \langle (\b I + \B)^{-1}, \A\rangle &\stackrel{(a)}{\geq}  \sum_{i=1}^p \frac{\alpha_i}{1+\beta_i}\stackrel{(b)}{\geq}  \frac{\sum_{i=1}^p \alpha_i}{1+\sum_{i=1}^p \beta_i} = \frac{\Tr(\A)}{1+ \Tr(\B)},
\end{align}
where inequality (a) follows by the lower bound of Von Neumann’s trace inequality \cite{Ruhe-1970}, inequality (b) follows by lemma~\ref{lm:inequ-2}.
    \end{proof}
\end{lemma}


\section{Proofs of \Cref{sec:regression}}\label{sec:appendix-pf-regression}

{\color{black}
\textsc{Proposition}~2.1. \textit{Assume that \cref{assume:px} holds, let $K>0$ be the constant s.t. $\| \V_p^{-1} \x_i\|_{\psi_2} = K$. Then for any $\gamma, \delta \in(0,1)$, if 
    \begin{align}\tag{\ref{eq:ls-nbound}}
        n \geq \max\left\{\frac{4C_K^2 d}{\gamma^2},\frac{4\log(2/\delta)}{c_k \gamma^2} \right\},
    \end{align}
    with probability at least $1-\delta$, we have
    \begin{align}\tag{\ref{eq:risk-linear}}
       \frac{\sigma^2}{(1+\gamma)k} f_V(\X_S)\leq R_p(\what) \leq \frac{\sigma^2}{(1-\gamma)k} f_V(\X_S).
    \end{align}
}}

\begin{proof}[Proof of \cref{prop:risk-linear}]
By the definition of the generalization error $R_p(\what) $ in \cref{eq:ls-lp-def}, we have
\begin{align}
    R_p(\what) = \E_{\bepsilon_S, \x}[(\Delta \what^\top \x)^2] &= {\color{black} \E_{\bepsilon_S, \x} }\left\langle {\color{black}\Delta\what \Delta\what ^\top}, \x\x^\top\right\rangle
    = \left\langle\E_{\bepsilon_S} [{\color{black}\Delta\what \Delta\what ^\top}], \E_{\x} [\x\x^\top]\right\rangle\nonumber\\
    & = \sigma^2 \left\langle(\X_S^\top \X_S)^{-1} , \E[\x\x^\top]\right\rangle
\end{align}
where the last equality follows by \cref{eq:ls-cov-weights}.
Then it is sufficient to show that with probability at least $1-\delta$, under the sample complexity bound in \cref{eq:ls-nbound}, we have
\begin{align}\label{eq:ls-lp-newbound}
    \V_p:= \E[\x\x^\top], \qquad (1-\gamma)\V_p\preceq \frac{1}{n} \X^\top \X \preceq (1+ \gamma) \V_p.
\end{align}

Note that 
\begin{align}
    \E_{\x_i \sim p(\x)}[(\V_p^{-1/2} \x_i)(\V_p^{-1/2} \x_i)^\top] =\V_p^{-1/2} \E_{\x_i \sim p(\x)}[\x_i \x_i^\top ]\V_p^{-1/2} = \b I.
\end{align}
Thus $\V_p^{-1/2}\x_i$ is isotropic for any $i\in[n]$. Let $\A \in\mathbb{R}^{n \times d}$ be the matrix whose rows are $\A_i = \V_p^{-1/2} \x_i$, i.e. $\A = \X \V_p^{-1/2}$. Then the premise of \cref{lm:ls-subgaussian} is satisfied. 

Let $c_K$ and $C_K$ be the constants in \cref{lm:ls-subgaussian}, $t = \sqrt{\frac{1}{c_K}\log(2/\delta)}$, then $c_K t^2 = \log(2/\delta)$ and $1-2\exp(-c_Kt^2) = 1 -\delta$.  By the sample complexity bound in \cref{eq:ls-nbound}, we have
\begin{align}
    C_K \sqrt{\frac{d}{n}} \leq C_K \sqrt{\frac{d}{4 C_K^2 d/\delta^2}} = \frac{\gamma}{2}, \quad \frac{t}{\sqrt{n}} \leq \frac{\gamma}{2},\quad C_K \sqrt{\frac{d}{n}} + \frac{t}{\sqrt{n}} \leq \gamma \leq 1.
\end{align}
By \cref{lm:ls-subgaussian}, we have
\begin{align}
   \bigg \| \frac{1}{n}\V_p^{-1/2} \X^\top \X \V_p^{-1/2} - \b I\bigg  \|_2 \leq \max(\gamma, \gamma^2) = \gamma,
\end{align}
and thus \cref{eq:ls-lp-newbound} is proved.
\end{proof}

{\color{black}
\textsc{Proposition}~2.2. \textit{Assume that \cref{assume:px} and \cref{assume:hessian} hold. Denote $\widetilde{c}:=c-1$. Let $\rho_1^+, \rho_1^-, \rho_2^+, \rho_2^->1$ be constants such that $ \rho_1^-\bH_p\preceq \b I_{\widetilde{c}} \otimes \V_p \preceq \rho_1 ^+\bH_p$ and $\rho_2^- \bH_k(\btheta_*)\preceq\b I_{\widetilde{c}} \otimes \left( \frac{1}{k} \X_S^\top \X_S\right) \preceq \rho_2^+ \bH_k(\btheta_*) $. Then for any $\gamma, \delta \in(0,1)$, if 
    \begin{align}\tag{\ref{eq:complexity-risklog}}
    k\gtrsim \widetilde{c} d\log(ed/\delta)\quad \mathrm{and}\quad n \gtrsim  d/\gamma^2,
    \end{align}
    we have with probability at least $1-\delta$,
    \begin{align}
  \frac{e^{-\alpha} + \alpha -1}{\alpha^2}\frac{\rho_2^-}{ \rho_1^+(1+\gamma) }\frac{\widetilde{c}f_V(\X_S)}{k}  
  &\lesssim \E[L_p(\emptheta)] - L_p(\btheta_*) \nonumber\\
  &\lesssim   \frac{e^{\alpha} - \alpha -1}{ \alpha^2} \frac{ \rho_2^+}{\rho_1^-(1-\gamma)}\frac{\widetilde{c}f_V(\X_S)}{k}, \tag{\ref{eq:risk-log}}
    \end{align}
where $\alpha=\mathcal{O}\left(\sqrt{\rho_1\widetilde{c} d\log(e/\delta)/k}\right)$ and $\E$ is expectation over $\{y\sim p(y|\x, \btheta_*)\}_{\x \in \X_S}$.}
}

\begin{proof}[Proof of \cref{prop:risk-log}] 
We can substitute $q(x)$ in Theorem 3 of \cite{firal} into the empirical distribution of $\X_S$. The premise of the theorem is satisfied since $q(x)$ is sub-Gaussian and $\V_q = \frac{1}{k}\X_S^\top \X_S$ is invertible. Thus by Theorem 3 of \cite{firal}, we have with probability at least $1-\delta$,
\begin{align}\label{eq:pf-fir-bound}
   \frac{e^{-\alpha} + \alpha -1}{\alpha^2} \frac{\left\langle \bH_k(\btheta_*)^{-1}, \bH_p \right\rangle}{k} \lesssim \E[L_p(\emptheta)] - L_p(\btheta_*) \lesssim   \frac{e^{\alpha} - \alpha -1}{ \alpha^2}\frac{\left\langle \bH_k(\btheta_*)^{-1}, \bH_p\right\rangle}{k}.
\end{align}
Since $ \rho_1^-\bH_p\preceq \b I_{\widetilde{c}} \otimes \V_p \preceq \rho_1 ^+\bH_p$ and $\rho_2^- \bH_k(\btheta_*)\preceq\b I_{\widetilde{c}} \otimes \left( \frac{1}{k} \X_S^\top \X_S\right) \preceq \rho_2^+ \bH_k(\btheta_*) $, we have that
\begin{align}
\frac{1}{\rho_1^+} \b I_{\widetilde{c}} \otimes \V_p \preceq   & \bH_p \preceq \frac{1}{\rho_1^-} \b I_{\widetilde{c}} \otimes \V_p, \nonumber\\
\rho_2^-\b I_{\widetilde{c}} \otimes \left( \frac{1}{k} \X_S^\top \X_S\right)^{-1}\preceq    &\bH_k(\btheta_*)^{-1}\preceq \rho_2^+ \b I_{\widetilde{c}} \otimes \left( \frac{1}{k} \X_S^\top \X_S\right)^{-1}.\nonumber
\end{align}
Then
\begin{align}\label{eq:pf-up}
    \left\langle \bH_k(\btheta_*)^{-1}, \bH_p\right\rangle &\leq   \frac{\rho_2^+}{\rho_1^-}\left\langle \b I_{\widetilde{c}} \otimes \left( \frac{1}{k} \X_S^\top \X_S\right)^{-1},  \b I_{\widetilde{c}} \otimes \V_p  \right\rangle\nonumber\\
    &= \frac{\rho_2^+}{\rho_1^-}\Tr \left( \b I_{\widetilde{c}}\otimes \left(\left( \frac{1}{k} \X_S^\top \X_S\right)^{-1} \V_p \right) \right)\nonumber\\
    &\leq \frac{\rho_2^+ \widetilde{c}}{\rho_1^-(1-\gamma)} f_V(\X_S).
\end{align}
Substitute \cref{eq:pf-up} into \cref{eq:pf-fir-bound} we can get the upper bound of \cref{eq:risk-log}. The lower bound can be proved similarly.
    
\end{proof}

\section{Proofs of \Cref{sec:opt-design}}\label{sec:appendix-pf-design}
{\color{black}
We begin by recalling the definitions, gradients, and Bregman divergences of the entropy and $\ell_{1/2}$ regularizers, summarized in \cref{tab:regularizer-def}, as these will be used throughout the proofs in this section.
}
\renewcommand{\arraystretch}{1.5} 
\begin{table}[!t]
\footnotesize	
\centering
    \caption{\color{black}Definitions, gradient and Bregman divergence of entropy and $\ell_{1/2}$ regularizers.}
    \label{tab:regularizer-def}
    \begin{tabular}{cccc}
      \toprule  Regularizer  & Function $w (\A)$ & Gradient $\nabla w(\A)$ & Bregman divergence $D_w(\Y, \X)$ \\\hline
       Entropy  & $\left\langle \A, \log \A - {\b I} \right\rangle$  & $\log \A$ &  $\left\langle \Y, \log \Y - \log \X \right\rangle - \Tr(\Y-\X)$\\\hline
     $\ell_{1/2}$  &$ -2 \Tr\left(\A^{1/2}\right)$  & $-\A^{1/2}$ & $\left\langle\Y, \X^{-1/2} \right\rangle + \Tr\left(\X^{1/2}\right) - 2\Tr\left(\Y^{1/2}\right)$\\\bottomrule
    \end{tabular}
\end{table}

\subsection{Proof of \cref{prop:expressions}}
$ $

{\color{black}
\textsc{Proposition}~3.3. (Closed forms of $\A_t$ by FTRL). \textit{FTRL chooses action $\A_1 = \frac{1}{d} {\b I}$ for the first round, For the subsequent rounds $t\geq 2$, the actions are
\begin{align}
    \A_t &= \exp\left(-\alpha\sum_{s=1}^{t-1} \F_s - \nu_t {\b I}\right),\qquad (\text{entropy-regularizer}) \tag{\ref{eq:At-ent}}\\
    \A_t &= \left( \alpha\sum_{s=1}^{t-1} \F_s + \nu_t {\b I} \right)^{-2}, \qquad\qquad (\ell_{1/2}\text{-regularizer}) \tag{\ref{eq:At-l12}}
\end{align}
where $\nu_t$ is the unique constant that ensures $\A_t \in \simplex$.}
}

{\color{black}\paragraph{Proof sketch} We begin by establishing in \cref{lm:unique} the existence and uniqueness of the action matrix $\A_t$ for both the entropy and $\ell_{1/2}$ regularizers. We then derive explicit expressions for $\A_t$ using the KKT conditions.}

\begin{lemma}\label{lm:unique}
    Given $w$ as either entropy regularizer or $\ell_{1/2}$-regularizer (defined in \cref{tab:regularizer-def}), $\F \in \mathbb{S}^d_{+}$, there exists unique $\A \in \mathbb{S}^d_{++}$ and $\nu \in \mathbb{R}$ to ensure that
    \begin{align}
        &\A = (\nabla w)^{-1} (-\F - \nu {\b I}),\label{eq:cond-1}\\
        &\Tr(\A) = 1. \label{eq:trace_constraint}
    \end{align}
\end{lemma}
\begin{proof}
  If $w$ is the entropy regularizer, i.e. $w(\A) = \langle \A, \log \A - {\b I} \rangle$, then we have $\nabla w(\A) = \log \A$ and $(\nabla w)^{-1}(\A) = \exp \A$. Thus 
    \begin{align}
      \A = \exp(-\F - \nu \b I).
  \end{align}
  Note that $\nabla w$ is a bijection from $\mathbb{S}^d_{++}$ to $\text{int}(\mathbb{S}^d)$, and $ (\nabla w)^{-1}$ is a bijection from $\text{int}(\mathbb{S}^d)$ to $\mathbb{S}^d_{++}$. 
  Since $\F \in \mathbb{S}^d_{+}$, we can assume its eigendecomposition as $\F = \b P \blambda \PT$, where $\blambda = \text{diag}\{\lambda_1, \cdots, \lambda_d\} \succeq \b 0$. Then $\F + \nu \b I = \b P (\blambda + \nu \b I) \PT$. From equation~\eqref{eq:trace_constraint} we have
  \begin{align}
      \Tr(\A) = \sum_{i=1}^d \exp(-\lambda_i - \nu) = 1,\qquad
      \nu = \log\big(\sum_{i=1}^d \exp(-\lambda_i) \big).
  \end{align}
  Therefore $\nu$ exits and is unique. Since $(\nabla w)^{-1}$ is a bijective function, $\A$ is also unique.

  If $w$ is the $\ell_{1/2}$-regularizer, i.e. $w(\A) = -2 \Tr(\A^{1/2})$, then we have $\nabla w(\A) = -\A^{-1/2}$ and $(\nabla w)^{-1}(\A) = (-\A)^{-2}$. Thus
  \begin{align}
      \A = \big( \F + \nu \b I\big)^{-2}.
  \end{align}

  Note that $\nabla w$ is a bijection from $\mathbb{S}^d_{++}$ to $\mathbb{S}^d_{--}$, and $(\nabla w)^{-1}$ is a bijection from $\mathbb{S}^d_{--}$ to $\mathbb{S}^d_{++}$. From equation~\eqref{eq:trace_constraint} we have
  \begin{align}
      \Tr(\A) = \sum_{i=1}^d \frac{1}{(\lambda_i +\nu)^2} =1.
  \end{align}
  
  It is easy to see that there exits $\nu$ satisfy this condition. For the uniqueness, suppose that there are two different $\nu_1$ and $\nu_2$ satisfying equations~\eqref{eq:cond-1} and \eqref{eq:trace_constraint}. With loss of generality, we assume $\nu_2 = \nu_1 + \beta$ and $\beta >0 $.
  Since $-(\F + \nu \b I) \in \mathbb{S}^d_{--}$, then $(\lambda_i + \nu) >0 $ for $i\in [d]$. Note that $\beta >0$, we have
  \begin{align}
      1 = \sum_{i=1}^d \frac{1}{(\lambda_i + \nu_2)^2} = \sum_{i=1}^d \frac{1}{(\lambda_i +\nu_1)^2 + 2\beta(\lambda_i +\nu_1) + \beta^2} < \sum_{i=1}^d \frac{1}{(\lambda_i +\nu_1)^2} = 1,\nonumber
  \end{align}
  which is a contradiction. Therefore, $\nu$ exits and is unique.  Since $(\nabla w)^{-1}$ is a bijective function, $\A$ is also unique.
  
\end{proof}

Now we derive the actions chosen by FTRL strategy for both entropy regularizer and $\ell_{1/2}$-regularizer.

\begin{proof}[Proof of \cref{prop:expressions}]
$ $
    \begin{itemize}[leftmargin=*]
\item $t=1$: $\A_1 = \argmin_{\A \in \simplex} w(\A)$. {\color{black}We define the  Lagrangian $L: \mathbb{S}_+^d \times \mathbb{R}\rightarrow \mathbb{R}$ associated with the problem as $L(\A, \nu) = w(\A) + \nu (\Tr(\A) -1)$. The Karush–Kuhn–Tucker (KKT) conditions imply that the solution $\A_1$ satisfy:}
\begin{align}
    \nabla w(\A_1) + \nu_1 \b I = \b 0, 
\qquad \Tr(\A_1) = 1.\nonumber
\end{align}
{\color{black}For the entropy-regularizer (see \cref{tab:regularizer-def}), we have $\log \A_1 = -\nu_1 \b I$ and $\Tr(\A_1)=1$. Thus $\A_1 = \frac{1}{d}\b I$. Similarly for $\ell_{1/2}$-regularizer, we can get $\A_1 = \frac{1}{d}\b I$.}
\item $t\geq 2:$ $\A_{t} = \argmin_{\A \in \simplex} \Large\{ \alpha \sum_{s=1}^{t-1} \langle \A,\F_s \rangle + w(\A)\Large\}$. We define the  Lagrangian $L: \mathbb{S}_+^d \times \mathbb{R}\rightarrow \mathbb{R}$ associated with the problem as $L(\A, \nu) = \alpha \sum_{s=1}^{t-1} \langle \A,\F_s \rangle + w(\A) +  \nu (\Tr(\A) - 1)$. The KKT conditions yield
\begin{align}
   \alpha \sum_{s=1}^{t-1} \F_s +  \nabla w(\A_t) + \nu_1 \b I = \b 0,\qquad \Tr(\A_t) = 1.\nonumber
\end{align}
Then for entropy regularizer and $\ell_{1/2}$ regularizer:
\begin{align}
    \A_t = \exp\big(-\alpha\sum_{s=1}^{t-1} \F_s - \nu_t {\b I}\big), \qquad  \A_t = \big( \alpha\sum_{s=1}^{t-1} \F_s + \nu_t {\b I} \big)^{-2}, 
\end{align}
where $\nu_t$ is the constant to ensure that $\Tr(\A_t) = 1$. The uniqueness of $\A_t$ and $\nu_t$ are ensured by \cref{lm:unique}.
    \end{itemize}
\end{proof}

\subsection{Proof of \cref{prop:regret-bound}}
$ $

{\color{black}
\textsc{Proposition}~3.4. (Regret bound of FTRL). \textit{Let $\alpha > 0$, and suppose that $\{\A_t\}_{t \in [k]}$ are selected by FTRL with either the entropy regularizer or the $\ell_{1/2}$-regularizer. Then the regret defined in~\cref{eq:regret} satisfies
\begin{align}
    R &\leq \frac{\log d}{\alpha} 
        + \sum_{t=1}^k \langle \A_t, \F_t \rangle 
        - \frac{1}{\alpha} \sum_{t=1}^k \Bigl(1 - \langle \A_t, \exp(-\alpha \F_t) \rangle \Bigr),
        \quad\text{(entropy regularizer)} \tag{\ref{eq:regret-entropy}}\\[1ex]
    R &\leq \frac{\sqrt{d}}{\alpha} 
        + \sum_{t=1}^k \langle \A_t, \F_t \rangle 
        - \frac{1}{\alpha} \sum_{t=1}^k 
          \Tr\!\left[\A_t^{1/2} - \bigl(\A_t^{-1/2} + \alpha \F_t \bigr)^{-1}\right],
         \quad\text{($\ell_{1/2}$ regularizer)}. \tag{\ref{eq:regret-l12}}
\end{align}
}
}

\paragraph{\color{black} Proof sketch} 
{\color{black}Our analysis primarily relies on the well-known regret bound of Mirror Descent (MD) in \cref{thm:md-bound}, whose definition is given in \cref{def:md}. We first show in \cref{prop:equi} that MD is equivalent to FTRL for our problem. The full details of the proof of \cref{prop:regret-bound} are then provided in \Cref{sec:full-proof-regret-bound}.}


\begin{definition}[Mirror Descent]\label{def:md}
Given a regularizer $w:\simplex\rightarrow\mathbb{R}$ and learning rate $\alpha > 0$, the Mirror Descent (MD) strategy initializes with
\[
    \A_1 = \argmin_{\A \in \simplex} w(\A)
\]
for the first round. For subsequent rounds $t \geq 2$, the actions are
\begin{align}\label{eq:md-def-1}
     \A_{t} = \argmin_{\A \in \simplex} \big\{ \alpha \langle \A, \F_{t-1} \rangle  + D_w(\A, \A_{t-1}) \big\},\qquad t=2,\dots,k,
\end{align}
where $D_w(\cdot, \cdot)$ denotes the Bregman divergence induced by $w(\cdot)$. The Bregman divergences for the entropy and $\ell_{1/2}$ regularizers are listed in~\cref{tab:regularizer-def}.  

Equation~\eqref{eq:md-def-1} is equivalent to the following two-step procedure:
\begin{align}
    &\Aint_{t} = \argmin_{\A \succeq \b 0} \big\{ \alpha \langle \A, \F_{t-1} \rangle + D_w(\A, \A_{t-1}) \big\},\label{eq:md-1} \\
    &\A_{t} = \argmin_{\A \in \simplex} D_w(\A, \Aint_{t}).\label{eq:md-2}
\end{align}
\end{definition}


\begin{theorem}[Mirror Descent Regret Bound (Theorem 28.4 in \cite{md-bound})]\label{thm:md-bound}
$ $
Let $\alpha >0$, assume that $\A_1$, $\A_2$, $\cdots$, $\A_t$  are chosen by MD or FTRL, then the regret defiened in \cref{eq:regret} is bounded by
\begin{align}\label{eq:regret-md}
        R &\leq \frac{\text{diam}_w(\simplex)}{\alpha}+ \frac{1}{\alpha}\sum_{t=1}^k D_w(\A_t, \Aint_{t+1}),
    \end{align}
where $\text{diam}_w(\simplex) = \sup_{\A, {\b B} \in \simplex} w(\A) - w(\b B)$ is the diameter of $\simplex$ with respect to $w$.
\end{theorem}


\begin{proposition}\label{prop:equi}
Given the same learning rate ${\color{black}\alpha}>0$ and loss matrices $\{\F_t \}_{t=1}^k$ for all rounds, MD is equivalent to FTRL for both entropy regularizer and $\ell_{1/2}$-regularizer.
\end{proposition}

\begin{proof}[Proof of \cref{prop:equi}]
    For entropy regularizer or $\ell_{1/2}$-regularizer, we can use induction to prove the equivalence between mirror descent (MD) and FTRL. To differentiate these two methods, we use $(L,\A_t,\nu_t)$  and $(L^\prime,\A^\prime_t,\nu^\prime_t)$ to represent FTRL's  and MD's Lagrangian and its primal and dual optimal, respectively.
    
    The Lagrangian of the MD method when $t\geq 2$ (equation~\eqref{eq:md-def-1}) is
    \begin{align}
        L^\prime (\A^\prime, \nu^\prime) &= \alpha \langle \A^\prime, \F_{t-1} \rangle + D_w(\A^\prime, \A^\prime_{t-1}) + \nu^\prime (\Tr(\A^\prime) -1)\nonumber\\
        &= \alpha \langle \A^\prime, \F_{t-1} \rangle + w(\A^\prime) - w(\A^\prime_{t-1}) - \langle \nabla w(\A^\prime_{t-1}), \A^\prime - \A^\prime_{t-1} \rangle\nonumber\\
        &+ \nu^\prime (\Tr(\A^\prime) -1).
    \end{align}
    Since the constriants satisfy the Slater's condition, the strong duality holds and the primal and dual optimal, i.e. $\A^\prime_t$ and $\mu_t$ should satisfy the KKT conditions:
    \begin{align}
        \alpha \F_{t-1} + \nabla w(\A^\prime_t) - \nabla w(\A^\prime_{t-1}) + \nu^\prime_t \b I = \b 0, \qquad \Tr(\A^\prime_t) = 1. \nonumber
    \end{align}
    Then  
\begin{align}\label{eq:sol-md}
    \A^\prime_t = (\nabla w)^{-1}\big(\nabla w (\A^\prime_{t-1}) - \nu^\prime_t \b I - \alpha \F_{t-1}\big),
\end{align}
where $\nu^\prime_t$ is the constant to ensure that $\Tr(\A^\prime_t) = 1$.

When $t=1$, $\A_1 = \A_1^\prime$ by definition.

Assume that at some $t\geq 2$, $\A_{t-1} = \A^\prime_{t-1}$, we need to show that $\A_{t} = \A^\prime_{t}$. By \cref{prop:expressions}, 
\begin{align}
    \A^\prime_{t-1} = \A_{t-1} =  (\nabla w)^{-1}\big( -\alpha \sum_{s=1}^{t-2} \F_s - \nu_{t-1} \b I\big), \label{eq:sol-t-1}\\
    \A_{t} =  (\nabla w)^{-1}\big( -\alpha \sum_{s=1}^{t-1} \F_s - \nu_{t} \b I\big).\label{eq:sol-ftrl-t}
\end{align}
Substituting \eqref{eq:sol-t-1} into \eqref{eq:sol-md}, we can get
\begin{align}
    \A^\prime_{t} &= (\nabla w)^{-1}\big(-\alpha \sum_{s=1}^{t-2} \F_s - \nu_{t-1} \b I -\nu^\prime_t \b I - \alpha \F_{t-1})\nonumber \\
    &=(\nabla w)^{-1}\big(-\alpha \sum_{s=1}^{t-1} \F_s - (\nu_{t-1} + \nu^\prime_{t}) \b I \big)\label{eq:sol-md-t}
\end{align}
Since $\Tr(\A_t) = \Tr(\A^\prime_{t}) = 1$, by comparing equations~\eqref{eq:sol-md-t} and \eqref{eq:sol-ftrl-t} and by \cref{lm:unique}, we have
\begin{align}
    \A_t = \A^\prime_t,\qquad \nu_t = \nu_{t-1} + \nu^\prime_t. \nonumber
\end{align}
By induction, we have proved the equivalence between MD and FTRL when $w$ is either entropy regularizer or $\ell_{1/2}$-regularizer.
\end{proof}

\subsubsection{\color{black} Full proof of \cref{prop:regret-bound}}\label{sec:full-proof-regret-bound}
\begin{proof}
We mainly use \cref{eq:regret-md} to derive the regret bounds for entropy- and $\ell_{1/2}$-regularizer.
\begin{enumerate}[label=(\alph*),leftmargin=*]
\item  For the entropy regularizer $w(\A) = \langle \A, \log \A - {\b I} \rangle$, {\color{black} we have  $\text{diam}_w(\simplex) \leq \log d$ by \cref{lm:diameter}}. From the two-step implementation of MD in \eqref{eq:md-1} and \eqref{eq:md-2}, we have:
        \begin{align}
            \Aint_{t+1} = \exp(\log \A_t - \alpha \F_t),\quad \A_{t}=\exp(c_t {\b I} - \alpha \sum_{s=1}^{t-1} \F_s),
        \end{align}
        where $c_t$ is constant s.t. $\Tr(\A_{t}) = 1$.
        
        By the Bregman divergence definition,
        \begin{align}\label{eq:ent-1}
            D_w(\A_t, \Aint_{t+1}) &= \langle \A_t, \log \A_t - \log \Aint_{t+1} \rangle + \Tr(\Aint_{t+1}) - \Tr(\A_t) \nonumber\\
            &=
     \langle \A_t, \alpha \F_t \rangle +\Tr[\exp(\log \A_t - \alpha \F_t)] - 1.
        \end{align}
By using Golden-Thompson inequality \cite{golden-thompson}, we have 
    \begin{align}\label{eq:ent-11}
            \Tr[\exp(\log \A_t - \alpha \F_t)]  \leq \left\langle \A_t, \exp(-\alpha\F_t) \right\rangle.
        \end{align}
Substitute \cref{eq:ent-11} and \cref{eq:ent-1} together into \cref{eq:regret-md}, we can get the regret bound for the case of entropy-regularizer, i.e. \cref{eq:regret-entropy}.
\item  For the $\ell_{1/2}-$regularizer $w(\A) = -2 \Tr[\A^{-1/2}]$, {\color{black} we have $\text{diam}_w(\simplex) \leq 2\sqrt{d}$ by \cref{lm:diameter}}. From the 2-step MD implementation, we have:
        \begin{align}
            \Aint_{t+1}^{-1/2} = \A_t^{-1/2} + \alpha \F_t, \qquad \A_{t}=\left(c_t {\b I} +\alpha \sum_{s=1}^{t-1} \F_s\right)^{-2},\nonumber
        \end{align}
        where $c_t$ is constant s.t. $\Tr(\A_{t}) = 1$.
      \begin{align}
          D_w(\A_t, \Aint_{t+1}) &= \langle \A_t, \Aint_{t+1}^{-1/2} \rangle + \Tr(\Aint_{t+1}^{1/2}) -2 \Tr(\A_t^{1/2})\nonumber\\
    &=\langle\A_t,  \alpha \F_t\rangle + \Tr[(\A_t^{-1/2}+ \alpha \F_t)^{-1}] -\Tr(\A_t^{1/2}).\label{eq:l12-1}
      \end{align} 
Substitute \cref{eq:l12-1} into \cref{eq:regret-md}, then we can get the regret bound of FTRL with $\ell_{1/2}$-regularizer, i.e. \cref{eq:regret-l12}.
\end{enumerate}
    
\end{proof}

\subsection{Proof of \cref{thm:min-eigen}}\label{sec:proof-min-eigen}

{\color{black}
\thmmineigen*
}


\begin{proof}
We mainly use the regret bounds from \cref{prop:regret-bound} to get the lower bound of $\lambda_{\min}(\F)$. First, recall that the regret is defined by $R :=\sum_{t=1}^k \left\langle\A_t , \F_t \right\rangle -  \inf_{\U \in \simplex} \left\langle \U, \sum_{t=1}^k \F_t \right\rangle$. In \cref{lm:mtd-eigen-min}, we showed that $\lambda_{\min}(\F) = \inf_{\U \in \simplex} \left\langle \U, \sum_{t=1}^k \F_t \right\rangle$. Thus,
\begin{align}\label{eq:new-regret-def}
    R = \sum_{t=1}^k \left\langle\A_t , \F_t \right\rangle - \lambda_{\min}(\F)
\end{align}
\begin{enumerate}[label=(\alph*),leftmargin=*]
    \item 
    For entropy-regularizer, substituting \cref{eq:new-regret-def} into \cref{eq:regret-entropy}, we can get
    \begin{align}\label{eq:pf-eigen-min-1}
        \lambda_{\min}(\F) \geq -\frac{\log d}{\alpha} + \frac{1}{\alpha} \sum_{t=1}^k \Bigl(1 - \langle \A_t, \exp(-\alpha \F_t) \rangle \Bigr).
    \end{align}
    Let $\F_t = \xx \xxt$, $\U \in \mathbb{R}^{d\times d}$ be a unitary matrix such that the first column is $\xx/\|\xx\|_2$, then it is easy to see that
    \begin{align}
        \exp(-\alpha \xx\xx^\top) = \U (\b I - \bLambda)\U^\top, \quad \bLambda = \mathrm{diag}\{1-\exp(-\alpha \|\xx\|_2^2, 0, 0, \cdots, 0 \}.
    \end{align}
    Thus
    \begin{align}\label{eq:ent-3}
       1-  \langle\A_t, \exp(-\alpha \F_t)\rangle &= 1-\langle \A_t, \exp(-\alpha \xx \xxt)\rangle  \nonumber \\&= 1- \left\langle \A_t, \b I- \left(1 - \exp(-\alpha \|\xx\|_2^2) \right)\frac{\xx \xx^\top}{\|\xx\|_2^2} \right\rangle \nonumber\\
        & = \left(1 - \exp(-\alpha \|\xx\|_2^2)\right)\frac{\xx^\top\A_t \xx}{\| \xx\|_2^2}.
    \end{align}
    Substitute \cref{eq:ent-3} into \cref{eq:pf-eigen-min-1}, we can get \cref{eq:ent-bound}.
\item  
For $\ell_{1/2}$-regularizer, substituting \cref{eq:new-regret-def} into \cref{eq:regret-l12}, we can get
    \begin{align}\label{eq:pf-eigen-min-2}
        \lambda_{\min}(\F) \geq -\frac{\sqrt{d}}{\alpha} + \frac{1}{\alpha} \sum_{t=1}^k 
          \Tr\!\left[\A_t^{1/2} - \bigl(\A_t^{-1/2} + \alpha \F_t \bigr)^{-1}\right].
    \end{align}
    
Let $\alpha \F_t = \xx \xx^\top$, by Sherman–Morrison formula \cref{lm:sherman-morrison},
,
    \begin{align}\label{eq:l12-2}
        \Tr[(\A_t^{-1/2}+ \alpha \F_t)^{-1}]&= \Tr\bigg[\A_t^{1/2} - \frac{\A_t^{1/2} \xx\xxt \A_t^{1/2}}{1+\xxt \A_t^{1/2} \xx}\bigg]\nonumber\\
        &= \Tr(\A_t^{1/2}) - \frac{\xx \A_t\xxt}{1+ \xt \A_t^{1/2} \xx}.
    \end{align}
Substitute \cref{eq:l12-2} into \cref{eq:pf-eigen-min-2}, we can obtain \cref{eq:l12-bound}.
\end{enumerate}
\end{proof}

\subsection{Proof of \cref{thm:design-complexity}}\label{sec:pf-complexity}

{\color{black}
\thmcomplexity*
}

\paragraph{\color{black}Proof sketch} 
{\color{black}Let $\{x_{i_t}\}_{t=1}^k$ denote the points selected over $k$ steps, and define $\F = \sum_{t=1}^k x_{i_t} x_{i_t}^\top$. By \cref{prop:goal}, if we can establish a lower bound $\lambda_{\min}(\F) \ge \tau$, then it follows that $f(\X_S^\top \X_S) \le \tau^{-1} f^*$.

Using \cref{thm:min-eigen}, for both the entropy and $\ell_{1/2}$ regularizers, we obtain
\begin{align}
    \lambda_{\min}(\F) \ge c + \sum_{t=1}^k s_t,
\end{align}
where $c$ is a constant, and $s_t$ is the objective value of the point selected at step $t$ as defined in \cref{def:opt-sample}.

Next, we provide lower bounds for $s_t$ in \cref{lm:lowerbound-max}: (a) and (b) correspond to the entropy-regularizer case, while (c) applies to the $\ell_{1/2}$-regularizer.

Finally, we complete the proof of \cref{thm:design-complexity} by combining these results at the end of this section.}

\begin{lemma}\label{lm:lowerbound-max}
Suppose  $\sum_{i=1}^n \pistar \x_i \xt_i = \b I$, where $\pistar \geq 0$ and $\sum_{i=1}^n  \pistar = k $. Then for $1 \leq t \leq k$, $\alpha >0$, and any non-zero $\A_t \in \simplex$,  we have the followings:

\begin{flalign}
&\text{(a) }\max_{i \in [n]}\frac{1 - \exp(-\alpha \xnorm)}{\alpha \xnorm}  \xt_i \A_t \x_i \geq \frac{1}{k + \alpha d}.\nonumber&&\\
&\text{(b) }  \max_{i \in [n]}\frac{1 - \exp(-\alpha \xnorm)}{\alpha \xnorm}   \xt_i \A_t \x_i \geq \frac{1-\exp(-\alpha \min_{i\in[n]} \|x_i\|^2)}{\alpha d}.\nonumber &&\\
&\text{(c) }  \max_{i \in [n]}\frac{\xt_i \A_t \x_i}{1 + \alpha \xt_i \A_t^{1/2} \x_i }   \geq \frac{1}{k + \alpha \sqrt{d}}.\nonumber&&
\end{flalign}
\end{lemma}

\begin{proof}
$ $
\begin{enumerate}[label=(\alph*),leftmargin=*]
\item Recall that $\Tr(\A_t) =1$ and $\sum_{i=1}^n \pistar \x_i \xt_i = \b I$, then 
\begin{align}
    \sumin \pistar \xt_i \A_t \x_i = \sumin \pistar \langle \x_i \xt_i, \A_t \rangle = \langle \b I, \A_t \rangle = 1, \label{eq:norminator}\\
    \sumin \pistar \xnorm = \sumin \pistar \langle \x_i \xt_i, \b I \rangle = \langle \b I,\b I  \rangle = d.
\end{align}
By~\cref{lm:inequ-0}, we have
\begin{align}
    \entropyfactor \leq 1 + \alpha \xnorm,
\end{align}
and thus
\begin{align}\label{eq:ent-denorm-1}
    \sumin \pistar \entropyfactor\leq \sumin \pistar + \alpha \sumin \pistar \xnorm = k + \alpha d. 
\end{align}
Combining equations~\eqref{eq:norminator} and \eqref{eq:ent-denorm-1}, we can get
\begin{align}
    \max_{i\in[n]} \frac{1 - \exp(-\alpha \xnorm)}{\alpha \xnorm} \xt_i \A_t \x_i  \stackrel{\cref{lm:inequ-1}}{\geq} \frac{\sumin \pistar \xt_i \A_t \x_i}{\sumin \pistar \entropyfactor} \geq \frac{1}{k + \alpha d}.
\end{align}
\item We have
\begin{align}
     \max_{i\in[n]} \frac{1 - \exp(-\alpha \xnorm)}{\alpha \xnorm} \xt_i \A_t \x_i  &\geq \min_{i\in[n]} \left[1-\exp(-\alpha \|x_i\|^2) \right]\cdot \max_{i\in[n]} \frac{ \xt_i \A_t \x_i}{\alpha \xnorm}\nonumber \\
     & \geq \left[1- \exp(-\alpha \min_{i\in[n]} \|x_i\|^2)\right] \cdot \frac{\sumin \pistar \xt_i \A_t \x_i}{\sumin \pistar \alpha \xnorm} \nonumber\\
     &= \frac{1-\exp(-\alpha \min_{i\in[n]} \|x_i\|^2)}{\alpha d},
\end{align}
where the last inequality is derived by \cref{lm:inequ-1}.
\item Note that 
\begin{align}
    \sumin \pistar(1 + \alpha \xt_i \A_t^{1/2} \x_i) = \sumin \pistar + \alpha \sumin \pistar \xt_i \A_t^{1/2} \x_i = k + \Tr(\A_t^{1/2}).
\end{align}
Let $\lambda_1, \lambda_2, \cdots, \lambda_d$ be the eigenvalues of $\A_t$, then $\Tr(\A_t) = \sum_{j=1}^d \lambda_j = 1$. Since $\sqrt{x}$ is a concave function for $x\geq 0$, we have
\begin{align}
    \Tr(\A_t^{1/2}) = d \cdot \sum_{j=1}^d \sqrt{\lambda_j} /d \leq d\cdot \sqrt{\sum_{j=1}^d \lambda_j /d} = d \cdot \sqrt{1/d} = \sqrt{d}.
\end{align}
Thus
\begin{align}
     \max_{i \in [n]}\frac{\xt_i \A_t \x_i}{1 + \alpha \xt_i \A_t^{1/2} \x_i }   &\stackrel{\cref{lm:inequ-1}}{\geq} \frac{\sumin \pistar \xt_i \A_t \x_i}{\sumin\pistar(1 + \alpha \xt_i \A_t^{1/2} \x_i) }  \geq \frac{1}{k + \alpha \sqrt{d}}.
\end{align}
\end{enumerate}
\end{proof}

\begin{proof}[Full proof of \cref{thm:design-complexity}]
$ $
\begin{enumerate}[label=(\alph*),leftmargin=*]
\item 
Let $\alpha = 4 \log d/ \epsilon$, $k\geq16 d\log d/\epsilon^2$, i.e. $C_1 = 4$, $C_2 \geq 16$, then 
\begin{align}\label{eq:ent-complex-1}
    \frac{\log d}{\alpha} = \frac{\epsilon}{4}.
\end{align}
By \cref{eq:ent-sample} of the sampling strategy for entropy regularizer and \cref{lm:lowerbound-max}, we have
\begin{align}\label{eq:ent-complex-2}
    \frac{1}{\alpha}\sum_{t=1}^k [1-\exp(-\alpha \Tr(\F_t)] \frac{\langle \A_t, \F_t \rangle}{\Tr(\F_t)} &= \sum_{t=1}^k \max_{i\in[n]} \frac{1 - \exp(-\alpha \xxnorm)}{\alpha \xxnorm} \xxt_i \A_t \xx_i \nonumber\\
    &\geq \frac{k}{k + \alpha d} \geq \frac{1}{1+\frac{\epsilon}{4}} \geq 1-\frac{\epsilon}{4}.
\end{align}
Substituting equations~\eqref{eq:ent-complex-1} and \eqref{eq:ent-complex-2} into the lower bound of equation~\eqref{eq:ent-bound}, then for $\epsilon \in (0,1)$, we can get
\begin{align}
    \inf_{\U \in \simplex} \langle \U, \sum_{t=1}^k \F_t \rangle \geq 1-\frac{\epsilon}{4} - \frac{\epsilon}{4} \geq 1-\frac{\epsilon}{2} \geq \frac{1}{1+\epsilon}.\nonumber
\end{align}
By~\cref{lm:mtd-eigen-min}, we have $ \lambda_{\min} (\XXT \BS \XX) \geq 1/(1+\epsilon)$. We finish the proof for entropy-regularizer by using \cref{prop:goal}.
\item 
By \cref{eq:ent-sample} and \cref{lm:lowerbound-max}, we have
\begin{align}\label{eq:ent-complex-3}
 \frac{1}{\alpha}\sum_{t=1}^k [1-\exp(-\alpha \Tr(\F_t)] \frac{\langle \A_t, \F_t \rangle}{\Tr(\F_t)} &= \sum_{t=1}^k \max_{i\in[n]} \frac{1 - \exp(-\alpha \xxnorm)}{\alpha \xxnorm} \xxt_i \A_t \xx_i \nonumber\\
 &\geq \left[1 - \exp\left(-\min_{t\in[k]} \|\widetilde{x}_{i_t}\|^2\right)\right]\frac{k}{ \alpha d} \geq 1.
\end{align}
Substituting equations~\eqref{eq:ent-complex-3} into the lower bound of equation~\eqref{eq:ent-bound}, then for $\epsilon \in (0,1)$, we can get
\begin{align}
    \inf_{\U \in \simplex} \langle \U, \sum_{t=1}^k \F_t \rangle  \geq 1-\frac{\log d}{\alpha} \geq 1 -\frac{\epsilon}{2} \geq \frac{1}{1+\epsilon} .\nonumber
\end{align}
By \cref{lm:mtd-eigen-min}, we have $ \lambda_{\min} (\XXT \BS \XX) \geq 1/(1+\epsilon)$. We finish the proof by using \cref{prop:goal}.
\item 
Let $\alpha = 8 \sqrt{d}/\epsilon$, $k \geq 32 d/\epsilon^2$, i.e. $C_4 =8$, $C_5=32$, then
\begin{align}\label{eq:l12-complex-1}
    \frac{2\sqrt{d}}{\alpha} = \frac{\epsilon}{4}.
\end{align}
By \cref{eq:l12-sample} of the sampling strategy for entropy regularizer and \cref{lm:lowerbound-max}, we have
\begin{align}\label{eq:l12-complex-2}
    \sum_{t=1}^k \frac{\langle\A_t, \F_t \rangle}{1 + \langle \A_t^{1/2}, \alpha \F_t\rangle} & = \sum_{t=1}^k \max_{i \in [n]}\frac{\xxt_i \A_t \xx_i}{1 + \alpha \xxt_i \A_t^{1/2} \xx_i } \nonumber \\
    &\geq \frac{k}{k + \alpha \sqrt{d}} \geq \frac{1}{1 + \frac{\epsilon}{4}} \geq 1 - \frac{\epsilon}{4}.
\end{align}
Substituting equations~\eqref{eq:l12-complex-1} and \eqref{eq:l12-complex-2} into the lower bound of equation~\eqref{eq:l12-bound}, then for $\epsilon\in (0,1)$, we can get
\begin{align}
    \inf_{\U \in \simplex} \langle \U, \sum_{t=1}^k \F_t \rangle
      \geq 1-\frac{\epsilon}{4} - \frac{\epsilon}{4} \geq 1-\frac{\epsilon}{2} \geq \frac{1}{1+\epsilon}.\nonumber
\end{align}
By~\cref{lm:mtd-eigen-min}, we have $ \lambda_{\min} (\XXT \BS \XX) \geq 1/(1+\epsilon)$. We finish the proof by using~\cref{prop:goal}.

\end{enumerate}
\end{proof}

\section{Proofs of \Cref{sec:reg}}\label{sec:appendix-pf-regdesign}


\subsection{Proof of \cref{thm:reg-min-eigen}}\label{sec:pf-reg-min-eigen}

{\color{black}
\thmlowerboundreg*
}

\begin{proof}
Similar to proof of \cref{thm:min-eigen}, we mainly use the regret bound in \cref{eq:regret-md} to derive lower bounds for $\lambda_{\min}(\sum_{t\in[k]} \F_t)$ in cases of entropy-regularizer and $\ell_{1/2}$-regularizer. But for the regularized optimal design problem, we have $\F_t = \xx_{i_t} \xx_{i_t}^\top + \frac{\lambda}{{\color{black}k}} \optsigma^{-1}$.
\begin{enumerate}[label=(\alph*),leftmargin=*]
\item For the entropy regularizer $w(\A) = \langle \A, \log \A - {\b I} \rangle$, we have diameter over $\simplex$ upper bounded by $\log d$, i.e. $\text{diam}_w(\simplex) \leq \log d$. 
        
        From the two-step implementation of MD in \eqref{eq:md-1} and \eqref{eq:md-2}, we have:
        \begin{align}
            \Aint_{t+1} = \exp(\log \A_t - \alpha \F_t),\quad \A_{t}=\exp(c_t {\b I} - \alpha \sum_{s=1}^{t-1} \F_s),
        \end{align}
        where $c_t$ is constant s.t. $\Tr(\A_{t}) = 1$.
        
        By the Bregman divergence definition,
        \begin{align}\label{eq:reg-ent-1}
            D_w(\A_t, \Aint_{t+1}) &= \langle \A_t, \log \A_t - \log \Aint_{t+1} \rangle + \Tr(\Aint_{t+1}) - \Tr(\A_t) \nonumber\\
            &=
     \langle \A_t, \alpha \F_t \rangle +\Tr[\exp(\log \A_t - \alpha \F_t] - 1.
        \end{align}
    Now, let $\F_t = \xx \xxt + \frac{\lambda}{{\color{black}k}}\optsigma^{-1}$, by Golden–Thompson inequality~\cite{golden-thompson},
     \begin{align}\label{eq:reg-ent-2}
\Tr[\exp(\log \A_t - \alpha \F_t)] - 1 &= \Tr\big[\exp\big(\log \A_t -\alpha \frac{\lambda}{{\color{black}k}}\optsigma^{-1} - \alpha  \xx \xxt\big)\big] - 1\nonumber\\
&\leq \langle \exp\big(\log \A_t -\alpha \frac{\lambda}{{\color{black}k}}\optsigma^{-1}), \exp(-\alpha \xx \xxt)\rangle - 1.
    \end{align}
    We define $\B_t\in\mathbb{R}^{d\times d}$ as symmetric positive definite matrix such that 
    \begin{align}\label{eq:reg-ent-22}
        \B_t \triangleq  \exp\big(\log \A_t -\alpha \frac{\lambda}{{\color{black}k}}\optsigma^{-1}).
    \end{align}

    Combing \cref{eq:reg-ent-1,eq:reg-ent-2,eq:reg-ent-22} and substituting into \cref{eq:regret-md}, we can get
    \begin{align}\label{eq:reg-ent-23}
        \lambda_{\min}\left(\sum_{t=1}^k \F_t\right)\geq -\frac{\log d}{\alpha} + \frac{1}{\alpha} \sum_{t=1}^k 1 - \left \langle \B_t, \exp(-\alpha \xx\xx^\top)\right \rangle.
    \end{align}

    Let $\U \in \mathbb{R}^{d\times d}$ be a unitary matrix such that the first column is $\xx/\|\xx\|_2$, then it is easy to see that
    \begin{align}
        \exp(-\alpha \xx\xx^\top) = \U (\b I - \bLambda)\U^\top, \quad \bLambda = \mathrm{diag}\{1-\exp(-\alpha \|\xx\|_2^2, 0, 0, \cdots, 0 \}.
    \end{align}
    Then
    \begin{align}\label{eq:reg-ent-3}
        \langle \B_t, \exp(-\alpha \xx \xxt)\rangle - 1 &= \Big\langle \B_t, \b I- \big(1 - \exp(-\alpha \|\xx\|_2^2) \big)\frac{\xx \xx^\top}{\|\xx\|_2^2} \Big\rangle - \Tr(\A_t)\nonumber\\
        & = \Tr(\B_t - {\color{black}\A_t}) -[1 - \exp(-\alpha \|\xx\|_2^2)]\frac{\xx^\top{\color{black} \B_t} \xx}{\| \xx\|_2^2}.
    \end{align}
    Substituting \eqref{eq:reg-ent-3} into \eqref{eq:reg-ent-23}, we can get the lower bound in \eqref{eq:reg-ent-bound}.
\item For the $\ell_{1/2}-$regularizer $w(\A) = -2 \Tr[\A^{-1/2}]$, we have diameter over $\simplex$ upper bounded by $2\sqrt{d}$, i.e. $\text{diam}_w(\simplex) \leq 2\sqrt{d}$.
        
        From the two-step implementation of MD, we have:
        \begin{align}
            \Aint_{t+1}^{-1/2} = \A_t^{-1/2} + \alpha \F_t, \qquad \A_{t}=(c_t {\b I} +\alpha \sum_{s=1}^{t-1} \F_s)^{-2},\nonumber
        \end{align}
        where $c_t$ is constant s.t. $\Tr(\A_{t}) = 1$.
      \begin{align}\label{eq:reg-l12-1}
          D_w(\A_t, \Aint_{t+1}) &= \langle \A_t, \Aint_{t+1}^{-1/2} \rangle + \Tr(\Aint_{t+1}^{1/2}) -2 \Tr(\A_t^{1/2})\nonumber\\
    &=\langle\A_t, \A_t^{-1/2}+ \alpha \F_t\rangle + \Tr[(\A_t^{-1/2}+ \alpha \F_t)^{-1}] -2 \Tr(\A_t^{1/2})
      \end{align} 
        Let $\alpha \F_t = \xx \xx^\top \frac{\lambda}{k}\optsigma^{-1}$, by Sherman–Morrison formula \cref{lm:sherman-morrison},
,
    \begin{align}\label{eq:reg-l12-2}
        \Tr[(\A_t^{-1/2}+ \alpha \F_t)^{-1}]&=\Tr[(\A_t^{-1/2}+\alpha\frac{\lambda}{k}\optsigma^{-1} +\alpha \xx\xxt )^{-1}] \nonumber\\
        &= \Tr[(\B_t^{-1/2} +\alpha \xx\xxt )^{-1}]= \Tr\bigg[\B_t^{1/2} - \frac{\B_t^{1/2} \xx\xxt \B_t^{1/2}}{1+\xxt \B_t^{1/2} \xx}\bigg]  \nonumber\\
        &= \Tr(\B_t^{1/2}) - \frac{\xx \B_t\xxt}{1+ \xt \B_t^{1/2} \xx},
    \end{align}
    where we define $\B_t\in\mathbb{R}^{d\times d}$ as symmetric positive definite matrix such that
    \begin{align}
        \B_t^{-1/2} = \A_t^{-1/2} + \alpha \frac{\lambda}{k}\optsigma^{-1}.
    \end{align}
    By substituting \eqref{eq:reg-l12-2} into \cref{eq:reg-l12-1}, and then into \cref{eq:regret-md}, we obtain the right-hand side of \eqref{eq:reg-l12-bound}.
\end{enumerate}
\end{proof}

\subsection{Proof of Proposition \ref{prop:guarantee-entropy-reg}}\label{sec:pf-reg-entropy}
$ $

{\color{black}
\textsc{Proposition}~4.4.  \textit{For entropy-regularizer, if $k \geq \alpha d$, at each round $t\in[k]$, we have
\begin{align}\tag{\ref{eq:guarantee-entropy-reg}}
 \max_{i\in[n]} \frac{1}{\alpha}  \left\{\Tr\left(\A_t-\B_t\right) + \left[1-\exp\left(-\alpha \|\xx_{i}\|_2^2\right)\right] \frac{\xx_{i}^\top\B_t \xx_{i}}{\| \xx_{i}\|_2^2}\right\} \geq \frac{1}{k +\alpha d}.
\end{align}}
}

\paragraph{\color{black}Proof sketch} 
{\color{black}The key difference between the regularized and unregularized cases lies in the form of the loss matrix. 
If $i_t$ denotes the index of the selected point at step $t$, then in the unregularized case we have 
$\F_t = \xx_{i_t}\xx_{i_t}^\top$, whereas in the regularized case 
$\F_t = \xx_{i_t}\xx_{i_t}^\top + \tfrac{\lambda}{k}\bSigma_\diamond^{-1}$. 
The main challenge arises because no assumptions are made on the data points beyond the full-rank condition of the input matrix $\X$.  

To address the extra term $\tfrac{\lambda}{k}\bSigma_\diamond^{-1}$, we first derive a lower bound involving it, denoted by $h(\tD)$ in \cref{eq:pf-ent-1}, where $\tD = \tfrac{\lambda}{k}\bSigma_\diamond^{-1}$. 
We then show that $h$ is convex over its domain $\mathrm{dom}\,h = \{ \tD \in \mathbb{S}_{+}^{d} : \tD \preceq \tfrac{1}{k}\b I \}$. 
Finally, this convexity allows us to establish a lower bound on $h(\tD)$, which leads directly to \cref{eq:guarantee-entropy-reg}.}

\begin{proof}[Full proof]
    For the ease of discussion, we denote $\tD = \frac{\lambda}{k}\optsigma^{-1}$. By~\cref{eq:bt-ent} we have
    \begin{align}\label{eq:bt-pf-ent}
        \B_t = \exp(\log \A_t - \alpha \tD).
    \end{align}
By~\cref{eq:reg-transforms}, we have
    \begin{align}
        \sum_{i\in[n]} \pi_{i}^{\diamond} \xx_i \xx_i^\top &= \sum_{i\in[n]} \pi_{i}^{\diamond} \optsigma^{-1/2} \x_i \x_i^\top \optsigma^{-1/2} = \optsigma^{-1/2} (\optsigma- \lambda \b I) \optsigma^{-1/2} = \b I - \lambda \optsigma^{-1}\nonumber\\
        &= \b I - k\tD ,\label{eq:reg-cov_sum}
    \end{align}
and
\begin{align}\label{eq:reg-norm_sum}
     \sum_{i\in[n]} \pi_{i}^{\diamond} \| \xx_i\|_2^2 =\Tr( \sum_{i\in[n]} \pi_{i}^{\diamond} \xx_i \xx_i^\top) = d - k\Tr(\tD).
\end{align}

Since $\Tr(\A_t)=1$, we have
\begin{align}\label{eq:pf-ent-1}
    &\max_{i\in[n]} \frac{1}{\alpha}  \Big\{\Tr(\A_t-\B_t) + [1-\exp(-\alpha \|\xx_{i}\|_2^2)] \frac{\xx_{i}^\top\B_t \xx_{i}}{\| \xx_{i}\|_2^2}\Big\}\nonumber\\
    &=\max_{i\in[n]} \frac{1-\Tr(\B_t)} {\alpha} + \frac{\xx_{i}^\top\B_t \xx_{i}}{\alpha\| \xx_{i}\|_2^2/(1-\exp(-\alpha \|\xx_{i}\|_2^2))}\nonumber\\
    &\stackrel{\cref{lm:inequ-0}}{\geq}  \frac{1-\Tr(\B_t)} {\alpha} + \max_{i\in[n]}\frac{\xx_{i}^\top\B_t \xx_{i}}{1 + \alpha\| \xx_{i}\|_2^2}\nonumber\\
    &\stackrel{\cref{lm:inequ-1}}{\geq} \frac{1-\Tr(\B_t)} {\alpha} + \frac{\langle \B_t, \sum_{i\in[n]} \pi_i^{\diamond} \xx_{i}\xx_{i}^\top\rangle}{\sum_{i\in[n]} \pi_i^{\diamond} + \alpha\sum_{i\in[n]} \pi_i^{\diamond} \| \xx_{i}\|_2^2}\nonumber\\
    & \geq \frac{1 - \Tr(\B_t)}{\alpha} + \frac{\langle \B_t, \b I-k\tD\rangle}{k + \alpha d} \triangleq h(\tD),
\end{align}
where the last inequality follows by \cref{eq:reg-norm_sum} and the fact that $\sum_{i\in[n]} \pi_i^{\diamond} = k$.
We define $h:\mathbb{S}_{+}^{d}\rightarrow \mathbb{R}$. By \cref{eq:reg-cov_sum}, $k \tD \preceq \b I$ and thus the domain of function $h$ is $\mathrm{dom}h = \{ \tD \in \mathbb{S}_{+}^{d} : \tD \preceq \frac{1}{k }\b I\}$.

We intend to find a lower bound for $h(\tD)$. First we prove that $h(\tD)$ is convex in its domain. We can verify its convexity by considering an arbitrary line, given by $\Z+t\V$, where $\Z, \V \in\mathbb{S}_{+}^{d}$. Define $g(t) \triangleq h(\Z+t\V)$, where $t$ is restricted to the interval such that $\Z + t\V \in \mathrm{dom}h$. By convex analysis theory, it is sufficient to prove the convexity of a function $g$. Substitute $\tD =\Z+t\V $ in to \cref{eq:pf-ent-1}, we have
\begin{align}
    g(t) = h(\Z + t\V) &= \frac{1}{\alpha}-\frac{1}{\alpha}\Tr\big(\exp(\log \A_t  - \alpha \Z - \alpha t \V)\big)\nonumber\\
    &+ \frac{1}{k + \alpha d} \big \langle \exp(\log \A_t  - \alpha \Z - \alpha t \V), \b I - k\Z - kt \V \big \rangle,
\end{align}
and the second order derivative has the following form:
\begin{align}
    g^{\prime \prime }(t) &=-\alpha \Tr(\V^2 \B_t) + \frac{2 \alpha k}{k + \alpha d} \Tr(\V^2 \B_t) + \frac{\alpha^2}{k + \alpha d} \big \langle \V^2\B_t, \b I - k\Z - kt \V \big \rangle\nonumber\\
    &\stackrel{(a)}{\geq} \frac{\alpha(k-\alpha d)}{k+\alpha d} \Tr(\V^2 \B_t) \stackrel{(b)}{\geq} 0,
\end{align}
where inequalities (a) and (b) follow by $\b I - k\Z - kt \V \succeq \b 0$ and $\V\succeq \b 0$ and $\B_t \succ \b 0$. Since $g^{\prime \prime}(t) \geq 0$, $g(t)$ is convex and so is $h(\tD)$. Note that the gradient of $h(\tD)$ is
\begin{align}
    \nabla h(\tD) &= \B_t + \frac{1}{k + \alpha d}[-\alpha \B_t - k \B_t + \alpha k \B_t \tD]\nonumber\\
    &= \big(1 - \frac{\alpha + k}{k + \alpha d}\big) \B_t + \frac{\alpha k}{k + \alpha d} \B_t\tD \succeq \b 0.
\end{align}
Thus 
 \begin{align}
     \inf_{\tD \in \text{dom}h} h(\tD) = h(\b 0) = \frac{1}{k + \alpha d}.
 \end{align}    
Substitute this into \cref{eq:pf-ent-1}, we finish the proof.
 
\end{proof}

\subsection{Proof of Proposition \ref{prop:guarantee-l12-reg}}\label{sec:pf-reg-l12}
$ $

{\color{black}
\textsc{Proposition}~4.5. For the $\ell_{1/2}$-regularizer, at each round $t\in[k]$, the following holds:
    \begin{align}\tag{\ref{eq:guarantee-l12-reg}}
        \max_{i\in[n]} \frac{1}{\alpha} \left\{\Tr\left(\A_t^{1/2}-\B_t^{1/2}\right) +  \frac{\alpha\xx_{i}^\top \B_t\xx_{i}}{1 +\alpha\xx_{i}^\top\B_t^{1/2}\xx_{i}}\right\} \geq \frac{1 - \frac{\alpha}{2k}}{k + \alpha\sqrt{d}}.
    \end{align}

}

\paragraph{\color{black}Proof sketch} 
{\color{black}As in \cref{prop:guarantee-entropy-reg}, the main difficulty comes from handling the additional $\tfrac{\lambda}{k}\bSigma_\diamond^{-1}$ term in the loss matrix. 
We follow a similar strategy, though the $\ell_{1/2}$-regularizer case requires more involved algebraic manipulations.} 

{\color{black}We begin by applying the inequalities from \Cref{sec:support-lemma} to obtain a lower bound on the objective that incorporates $\tfrac{\lambda}{k}\bSigma_\diamond^{-1}$ (steps 1 and 2 of the full proof):
\begin{align}\label{eq:sketch-1}
    &\max_{i\in[n]}\frac{1}{\alpha}\left\{\Tr(\A_t^{1/2}-\B_t^{1/2}) +  \frac{\alpha\xx_{i}^\top \B_t\xx_{i}}{1 +\alpha\xx_{i}^\top\B_t^{1/2}\xx_{i}}\right\} \nonumber\\&\geq \frac{\langle \A_t, k\tD\rangle + \langle \B_t, \b I - k\tD\rangle}{k + \alpha[\langle\A_t^{1/2}, k\tD \rangle + \langle\B_t^{1/2}, \b I - k\tD \rangle]}.
\end{align}}

{\color{black}Next, using convex analysis, we derive a lower bound for the numerator of \cref{eq:sketch-1} (step 3):
\begin{align}\label{eq:sketch-2}
    \langle \A_t, k\tD\rangle + \langle \B_t, \b I- k\tD\rangle \geq 1-\frac{\alpha}{2k}.
\end{align}}

{\color{black}We then establish an upper bound for the denominator term involving $\tD$ (step 4):
\begin{align}\label{eq:sketch-3}
     \langle \A_t^{1/2},k\tD\rangle + \langle \B_t^{1/2}, \b I -k\tD\rangle  \leq \sqrt{d}.
\end{align}}

{\color{black}Finally, substituting \cref{eq:sketch-2,eq:sketch-3} into \cref{eq:sketch-1} yields the desired guarantee in \cref{eq:guarantee-l12-reg}.}

\begin{proof}[Full proof]
For the ease of discussion, we denote $\tD = \frac{\lambda}{k}\optsigma^{-1}$. By~\cref{eq:bt-l12,eq:reg-transforms} we have
    \begin{align}
        \B_t^{-1/2} &= \A_t^{-1/2} + \alpha \tD,\label{eq:bt-pf-l12}\\
        \sum_{i\in[n]} \pi_{i}^{\diamond} \xx_i \xx_i^\top &= \sum_{i\in[n]} \pi_{i}^{\diamond} \optsigma^{-1/2} \x_i \x_i^\top \optsigma^{-1/2} = \optsigma^{-1/2} (\optsigma- \lambda \b I) \optsigma^{-1/2} = \b I - \lambda \optsigma^{-1}\nonumber\\
        &= \b I - k\tD.\label{eq:reg-cov_sum_2}
    \end{align}
    
\textbf{step 1.} At this step, we intend to find a lower bound for $\frac{1}{\alpha}\Tr(\A_t^{1/2} - \B_t^{1/2})$. By \cref{eq:bt-pf-l12} and Woodbury's matrix identity (\cref{lm:woodbury}), we have
\begin{align}
    \B_t^{1/2} &= \left( \A_t^{-1/2} + \optsigma^{-1/2} \left(\frac{\alpha \lambda}{k}\b I\right) \optsigma^{-1/2}\right)^{-1}\nonumber\\
    &= \A_t^{1/2} -\alpha \A_t^{1/2} \underbrace{\optsigma^{-1/2} \left(\frac{k}{\lambda}\b I + \alpha \optsigma^{-1/2} \A_t^{1/2} \optsigma^{-1/2} \right)^{-1} \optsigma^{-1/2}}_{\triangleq \b E}  \A_t^{1/2}.\label{eq:pf-l12-Bt}
\end{align}
Thus
\begin{align}
    \frac{1}{\alpha} \Tr(\A_t^{1/2} -\B_t^{1/2})& =  \Tr \left( \A_t^{1/2}\optsigma^{-1/2} \left(\frac{{\color{black}k}}{\lambda}\b I + \alpha\optsigma^{-1/2} \A_t^{1/2} \optsigma^{-1/2} \right)^{-1} \optsigma^{-1/2}  \A_t^{1/2}\right)\nonumber\\
    & =  \left \langle \left(\frac{{\color{black}k}}{\lambda}\b I + \alpha\optsigma^{-1/2} \A_t^{1/2} \optsigma^{-1/2} \right)^{-1},\optsigma^{-1/2}  \A_t \optsigma^{-1/2}   \right \rangle    \nonumber\\
   &   = \left \langle \left( \b I + \alpha \frac{\lambda}{k} \optsigma^{-1/2} \A_t^{1/2} \optsigma^{-1/2} \right)^{-1}, \frac{\lambda}{k} \optsigma^{-1/2} \A_t\optsigma^{-1/2}    \right \rangle\nonumber\\
     & \stackrel{\cref{lm:inequ-3}}{\geq} \frac{\Tr\left(\frac{\lambda}{k} \optsigma^{-1/2} \A_t \optsigma^{-1/2}\right)}{1 + \alpha \Tr\left(\frac{\lambda}{k} \optsigma^{-1/2} \A_t^{1/2}\optsigma^{-1/2} \right)}\nonumber\\
    & = \frac{\langle \A_t, \tD \rangle}{1 + \alpha \langle \A_t^{1/2}, \tD\rangle}.\label{eq:pf-l12-bound-1}
\end{align}

\textbf{step 2.} Following \cref{eq:pf-l12-bound-1} and \cref{lm:inequ-2}, we have
\begin{align}
    \frac{1}{\alpha}\left\{\Tr(\A_t^{1/2}-\B_t^{1/2}) +  \frac{\alpha\xx_{i}^\top \B_t\xx_{i}}{1 +\alpha\xx_{i}^\top\B_t^{1/2}\xx_{i}}\right\} &\geq  \frac{\langle \A_t, d \rangle}{1 + \alpha \langle \A_t^{1/2}, d\rangle} + \frac{\langle \B_t, \xx_i \xx_i^\top\rangle}{1 + \alpha \langle \B_t^{1/2}, \xx_i \xx_i^\top\rangle} \nonumber\\
    & \geq \frac{\langle \A_t, \tD \rangle +\langle \B_t, \xx_i \xx_i^\top\rangle }{1 + \alpha [\langle \A_t^{1/2}, \tD\rangle +\langle \B_t^{1/2}, \xx_i \xx_i^\top\rangle ]} 
\end{align}
Thus by~\cref{lm:inequ-1} and \cref{eq:reg-cov_sum}, we have
\begin{align}
    &\max_{i\in[n]}\frac{1}{\alpha}\left\{\Tr(\A_t^{1/2}-\B_t^{1/2}) +  \frac{\alpha\xx_{i}^\top \B_t\xx_{i}}{1 +\alpha\xx_{i}^\top\B_t^{1/2}\xx_{i}}\right\} \nonumber\\
    \geq&  \max_{i\in[n]} \frac{\langle \A_t, \tD \rangle +\langle \B_t, \xx_i \xx_i^\top\rangle }{1 + \alpha [\langle \A_t^{1/2}, \tD\rangle +\langle \B_t^{1/2}, \xx_i \xx_i^\top\rangle ]}  \nonumber\\
     \geq& \frac{\sum_{i\in[n]} \pi_{i}^{\diamond}\langle \A_t, \tD \rangle + \langle\B_t, \sum_{i\in[n]} \pi_{i}^{\diamond} \xx_{i} \xx_{i}^\top\rangle  }{\sum_{i\in[n]} \pi_{i}^{\diamond} + \alpha [\sum_{i\in[n]} \pi_{i}^{\diamond} \langle\A_t^{1/2},\tD \rangle + \langle\B_t^{1/2},\xx_{i} \xx_{i}^\top \rangle]} \nonumber\\
     & = \frac{\langle \A_t, k\tD\rangle + \langle \B_t, \b I - k\tD\rangle}{k + \alpha[\langle\A_t^{1/2}, k\tD \rangle + \langle\B_t^{1/2}, \b I - k\tD \rangle]} \label{eq:pf-l12-bound-2},
\end{align}
where the last equality also follows by the fact that $\sum_{i\in[n]} \pi_{i}^{\diamond} = k$.

\textbf{step 3.} In this step, we intend to show that the numerator of \cref{eq:pf-l12-bound-2} is lower bounded by $1 - \alpha /2 k$.
 First note that we have derived that $\B_t^{1/2} = \A_t^{1/2} - \alpha  \A_t^{1/2} \b E  \A_t^{1/2}$ in \cref{eq:pf-l12-Bt}. Then
\begin{align}\label{eq:pf-l12-s3-1}
     \B_t &= (\A_t^{1/2} - \alpha  \A_t^{1/2} \b E  \A_t^{1/2}) ^2\nonumber\\
     &= \A_t - \underbrace{(\alpha \A_t \b E  \A_t^{1/2} + \alpha \A_t^{1/2} \b E  \A_t - \alpha^2 \A_t^{1/2} \b E   \A_t \b E  \A_t^{1/2})}_{\triangleq\b G} = \A_t - \b G.
\end{align}
Substitute this into the numerator of \eqref{eq:pf-l12-bound-2}, we have
\begin{align}\label{eq:pf-l12-s3-2}
    \langle \A_t, k\tD\rangle + \langle \B_t, \b I- k\tD\rangle &= \langle \A_t, k\tD\rangle + \langle\A_t - \b G, \b I -k\tD \rangle\nonumber\\
    &= \Tr(\A_t) - \langle \b G, \b I -k\tD\rangle = 1 - \langle \b G, \b I-k\tD\rangle,
\end{align}
where the last equality follows by $\Tr(\A_t) = 1$. 

Now we intend to find an upper bound for $\langle \b G, \b I - k\tD \rangle$. First note that since $\A_t^{1/2}\b E   \A_t \b E  \A_t^{1/2}\succeq \b 0 $, by the definition of $\b G$ in \cref{eq:pf-l12-s3-1} we have
\begin{align}\label{eq:pf-l12-s3-3}
    \b G \preceq \alpha \A_t \b E  \A_t^{1/2} + \alpha \A_t^{1/2} \b E  \A_t.
\end{align}
Recall the definition of $\b E$ in \cref{eq:pf-l12-Bt}, we claim that $\b E \preceq \tD$. Indeed, since $\optsigma^{-1/2} \A_t^{1/2} \optsigma^{-1/2}$ is positive definite, we have 
\[
\frac{k}{\lambda}\b I + \alpha \optsigma^{-1/2} \A_t^{1/2} \optsigma^{-1/2} \succeq \frac{k}{\lambda}\b I ,
\]
 Thus $\left[ \frac{k}{\lambda}\b I  + \alpha \optsigma^{-1/2} \A_t^{1/2} \optsigma^{-1/2}\right]^{-1} \preceq \frac{\lambda}{k}\b I $ and therefore, 
\begin{align}\label{eq:pf-l12-s3-4}
    \b E\triangleq \optsigma^{-1/2}\left[ \frac{k}{\lambda}\b I  + \alpha \optsigma^{-1/2} \A_t^{1/2} \optsigma^{-1/2}\right]^{-1}  \optsigma^{-1/2} \preceq \optsigma^{-1/2}\left(\frac{\lambda}{k}\b I\right)  \optsigma^{-1/2}= d.
\end{align}

Now we have
\begin{align}\label{eq:pf-l12-s3-5}
    \langle \b G, \b I-k\tD\rangle &\stackrel{\cref{eq:pf-l12-s3-3}}{\leq} \alpha \langle \A_t \b E  \A_t^{1/2} +  \A_t^{1/2} \b E  \A_t , \b I-k\tD\rangle \nonumber\\
    &=\alpha \langle \b E, \A_t^{1/2} (\b I-k\tD) \A_t\rangle + \alpha \langle\b E, \A_t(\b I-k\tD) \A_t^{1/2} \rangle\nonumber\\
    &\stackrel{\cref{eq:pf-l12-s3-4}}{\leq} \alpha \langle \tD, \A_t^{1/2} (\b I-k\tD) \A_t + \A_t(\b I-k\tD) \A_t^{1/2}\rangle\nonumber\\
    & = 2 \alpha \Tr(\A_t^{3/2} \tD) - 2\alpha {\color{black}k}\Tr(\A_t^{1/2} \tD \A_t \tD)\triangleq h(\tD),
\end{align}
where we define function $h:\mathbb{S}_{+}^{d}\rightarrow \mathbb{R}$. By \cref{eq:reg-cov_sum}, $k\tD\preceq \b I$ and thus the domain of function $h$ is $\mathrm{dom}h = \{ \tD \in \mathbb{S}_{+}^{d} : \tD \preceq \frac{1}{k }\b I\}$.

We intend to find an upper bound for $h(\tD)$. First we prove that $h(\tD)$ is a concave function. We can verify its concavity by considering an arbitrary line, given by $\Z+t\V$, where $\Z, \V \in\mathbb{S}_{+}^{d}$. Define $g(t): = h(\Z+t\V)$, where $t$ is restricted to the interval such that $\Z + t\V \in \mathrm{dom}h$. By convex analysis theory, it is sufficient to prove the concavity of function $g$. Then
\begin{align}
    g(t) &= 2\alpha \Tr[ \A_t^{3/2}(\Z+t\V)] - 2\alpha k\Tr[\A_t^{1/2} (\Z +t\V) \A_t (\Z+t\V)]\nonumber\\
    &= -2\alpha k t^2 \Tr(\A_t^{1/2} \V \A_t \V) + 2\alpha t\Tr( \A_t^{3/2} \V)  \nonumber\\
    & -2\alpha k t\Tr(\A_t^{1/2}\V \A_t \Z+\A_t^{1/2}\Z\A_t \V) + 2\alpha \Tr(\Z \A_t^{3/2}) -2\alpha k\Tr(\A_t^{1/2} \Z \A_t \Z),\nonumber
\end{align}
and 
\begin{align}
    g^{\prime \prime}(t) = -4\alpha k\Tr(\A_t^{1/2} \V \A_t \V).\nonumber
\end{align}
Since $\A_t^{1/2} \V \A_t \V \succeq \b 0$, we have $g^{\prime \prime}(t) \geq 0$. Therefore $g(t)$ is concave and so is $h(\tD)$. Now consider the gradient of  $h(d)$:
\begin{align}\label{eq:pf-l12-s3-6}
    \nabla h(\tD) = 2 \alpha \A_t^{3/2} - 4\alpha k\A_t^{1/2} \tD \A_t.
\end{align}
Let $\nabla h(\tD)=0$, we can get $\tD = \frac{1}{2k}\b I \in \text{dom}h$. Thus 
\begin{align}\label{eq:pf-l12-s3-7}
    \sup_{\tD\in \text{dom}h} h(\tD) = h\left(\frac{1}{2k}\b I\right)=\frac{\alpha}{k}\Tr(\A_t^{3/2})- \frac{\alpha}{2k}\Tr(\A_t^{3/2}) = \frac{\alpha}{2k}\Tr(\A_t^{3/2}) \leq \frac{\alpha}{2k},
\end{align}
where the last inequality follows by the fact that all eigenvalues of $\A_t$ lie in $[0,1]$ and $\Tr(\A_t) =1$.

Combining \cref{eq:pf-l12-s3-2,eq:pf-l12-s3-5,eq:pf-l12-s3-7}, we can conclude that
\begin{align}\label{eq:pf-l12-bound-3}
    \langle \A_t, k\tD\rangle + \langle \B_t, \b I- k\tD\rangle \geq 1-\frac{\alpha}{2k}.
\end{align}

\textbf{step 4.} Now we derive an upper bound for the denominator of the right hand side of \cref{eq:pf-l12-bound-2}. By \cref{eq:pf-l12-Bt}, we have
\begin{align}\label{eq:pf-l12-bound-4}
      \langle \A_t^{1/2},k\tD\rangle + \langle \B_t^{1/2}, \b I -k\tD\rangle  &=  \langle \A_t^{1/2},k\tD\rangle + \langle \A_t^{1/2} - \alpha\A_t^{1/2}\b E\A_t^{1/2}, \b I-k\tD\rangle\nonumber\\
      &=\Tr(\A_t^{1/2}) - \alpha \langle\A_t^{1/2}\b E\A_t^{1/2}, \b I- k\tD \rangle\nonumber\\
      &\stackrel{(a)}{\leq}\Tr(\A_t^{1/2})  \stackrel{(b)}{\leq }\sqrt{d},
\end{align}
where (a) follows by the fact that both $\A_t^{1/2}\b E\A_t^{1/2}$ and $\b I- kd$ are positive semidefinite, (b) follows by the following property:
\begin{align}
    \Tr(\A_t^{1/2}) = \sum_{i\in[d]} \lambda_i(\A_t^{1/2}) \leq \sqrt{d}\sqrt{\sum_{i\in[d]} \lambda_i^2(\A_t^{1/2})}  = \sqrt{d}\sqrt{\sum_{i\in[d]} \lambda_i(\A_t)}  =\sqrt{d}.
\end{align}
where $\lambda_i(\A_t)$ is the $i$-th eigenvalue of $\A_t$, the inequality follows by the Cauchy-Schwarz inequality, the last equality follows by $\Tr(\A_t)=1$.

\textbf{step 5.} Now substitute \cref{eq:pf-l12-bound-3} and \cref{eq:pf-l12-bound-4} into \cref{eq:pf-l12-bound-2}, we have
    \begin{align}
        &\max_{i\in[n]}\frac{1}{\alpha}\left\{\Tr(\A_t^{1/2}-\B_t^{1/2}) +  \frac{\alpha\xx_{i}^\top \B_t\xx_{i}}{1 +\alpha\xx_{i}^\top\B_t^{1/2}\xx_{i}}\right\} \nonumber\\
        &\geq \frac{\langle \A_t, k\tD\rangle + \langle \B_t, \b I- k\tD\rangle}{k + \alpha [\langle \A_t^{1/2}, k\tD\rangle + \langle \B_t^{1/2}, \b I- k\tD\rangle]} \geq \frac{1-\frac{\alpha}{2k}}{k +\alpha\sqrt{d}} .
    \end{align}
    
\end{proof}

\subsection{Proof of \cref{thm:reg-design-complexity}}

{\color{black}
\thmcomplexityreg*
}

\begin{proof}
Similar to the proof of \cref{thm:design-complexity}, our main approach involves utilizing \cref{prop:guarantee-entropy-reg,prop:guarantee-l12-reg} to establish the complexity guarantee for both the entropy-regularizer and $\ell_{1/2}$-regularizer.
    \begin{enumerate}[label=(\alph*)]
\item entropy-regularizer: Let $\alpha = 4 \log d/ \epsilon$, $k\geq16 d\log d/\epsilon^2$,, then 
\begin{align}\label{eq:reg-ent-complex-1}
    \frac{\log d}{\alpha} = \frac{\epsilon}{4}.
\end{align}
By \cref{prop:guarantee-entropy-reg}, we have
\begin{align}\label{eq:reg-ent-complex-2}
    &\frac{1}{\alpha} \sum_{t=1}^k \Big\{\Tr(\A_t-\B_t) + [1-\exp(-\alpha \|\xx_{i_t}\|_2^2)] \frac{\xx_{i_t}^\top\B_t \xx_{i_t}}{\| \xx_{i_t}\|_2^2}\Big\} \nonumber\\
    \geq &\frac{k}{k + \alpha d} \geq \frac{1}{1+\frac{\epsilon}{4}} \geq 1-\frac{\epsilon}{4}.
\end{align}
Substitute equations~\eqref{eq:reg-ent-complex-1} and \eqref{eq:reg-ent-complex-2} into \cref{eq:reg-ent-bound},  we can get
\begin{align}
   \lambda_{\min} (\sum_{t=1}^{\color{black}k} \F_t) \geq 1-\frac{\epsilon}{4} - \frac{\epsilon}{4} \geq 1-\frac{\epsilon}{2} \geq \frac{1}{1+\epsilon}.\nonumber
\end{align}
By \cref{prop:reg-goal}, we have
\begin{align}
    f\Big(\sum_{t=1}^k \F_t\Big) =f\Big( \X_S^\top \X_S + \lambda \b I\Big) \leq (1+\epsilon) f^*.
\end{align}

\item $\ell_{1/2}$-regularizer:
 Let $k= 32d/\epsilon^2 + 16\sqrt{d}/\epsilon^2$, $\alpha = 8 \sqrt{d}/\epsilon$, by \cref{prop:guarantee-l12-reg}, we have
    \begin{align}\label{eq:pf-reg-l12-complex-1}
        &\frac{1}{\alpha}\sum_{t=1}^k \Big\{ \Tr(\A_t^{1/2}-\B_t^{1/2}) +  \frac{\alpha\xx_{i_t}^\top \B_t\xx_{i_t}}{1 +\alpha\xx_{i_t}^\top\B_t^{1/2}\xx_{i_t}} \Big\}\nonumber\\
        \geq& \sum_{t=1}^k\frac{1-\frac{\alpha}{2k}}{k+\alpha\sqrt{d}} = \frac{k- \frac{\alpha}{2}}{k + \alpha\sqrt{d}} \geq \frac{32d/\epsilon^2 + 16\sqrt{d}/\epsilon^2 - 4\sqrt{d}/\epsilon}{32d/\epsilon^2 + 16\sqrt{d}/\epsilon^2  + 8 d/\epsilon}\nonumber\\
        \geq &\frac{32d/\epsilon^2 + 16\sqrt{d}/\epsilon^2+ 8 d/\epsilon - ( 8 d/\epsilon+4 \sqrt{d}/\epsilon )}{32d/\epsilon^2 + 16\sqrt{d}/\epsilon^2+ 8 d/\epsilon} \nonumber\\
        =& 1 - \frac{8 d/\epsilon+4 \sqrt{d}/\epsilon}{\frac{4}{\epsilon} (8 d/\epsilon+4 \sqrt{d}/\epsilon) + 8\sqrt{d}/\epsilon}
        \geq & 1 - \frac{\epsilon}{4}.
    \end{align}
Substitute \cref{eq:pf-reg-l12-complex-1}  into \cref{eq:reg-l12-bound} in \cref{thm:reg-min-eigen}, we have
\begin{align}
    \lambda_{\min} (\sum_{t=1}^b \F_t) 
      &\geq -\frac{2\sqrt{ d}}{\alpha} +  \frac{1}{\alpha}\sum_{t=1}^k \Big\{ \Tr(\A_t^{1/2}-\B_t^{1/2}) +  \frac{\alpha\xx_{i_t}^\top \B_t\xx_{i_t}}{1 +\alpha\xx_{i_t}^\top\B_t^{1/2}\xx_{i_t}} \Big\}\nonumber\\
      &\geq -\frac{2\sqrt{d}}{8 \sqrt{d}/\epsilon} + 1 - \frac{\epsilon}{4}= 1 - \frac{\epsilon}{2} \geq\frac{1}{1+\epsilon}.
\end{align}

By  \cref{prop:reg-goal}, we can get
\begin{align}
    f\Big(\sum_{t=1}^k \F_t\Big) =f\Big( \X_S^\top \X_S + \lambda \b I\Big) \leq (1+\epsilon) f^*.
\end{align}
    \end{enumerate}
\end{proof}


\section{Additional Experimental Details}\label{sec:new-experiments}
\definecolor{Gray}{gray}{0.9}
\definecolor{LightCyan}{rgb}{0.88,1,1}
\paragraph{Spectral embedding via normalized graph Laplacian} \Cref{algo:laplacian} outlines the algorithm for obtaining spectral embedding features using the normalized graph Laplacian.

\begin{algorithm}[t]
\caption{Spectral embedding via normalized graph Laplacian}
\label{algo:laplacian}
\hspace*{\algorithmicindent} \textbf{Input:} \nolinebreak
 data points $\X \in \mathbb{R}^{N \times D}$, nearest neighbor number $k$, target out put dimension $d$   \\
 \hspace*{\algorithmicindent} \textbf{Output:} \nolinebreak $\widehat{\X} \in\mathbb{R}^{N\times d}$
\begin{algorithmic}[1]
    \STATE Obtain $k$-nearest neighbor graph $\mathcal{G}$ on $\X$.
    \STATE Obtain adjacency matrix $\A$ and its degree matrix $\b D$ from $\mathcal{G}$ (using ones as weights).
   \STATE Calculate normalized Laplacian $\b L\gets \b I - {\b D}^{-1/2} \A {\b D}^{-1/2}$.
   \STATE Calculate the first $d$ eigenvectors of $\b L$ (corresponding to the $d$ smallest eigenvalues of $\b L$): $\{ v_i \}_{i\in[d]}$.
    \STATE Form matrix $\widehat{\X}$ by stacking  $\{v_i\}_{i\in[d]}$ column-wise.
\end{algorithmic}

\end{algorithm}

\paragraph{Logistic regression classifier}
For our classification tasks, we used the \texttt{LogisticRegression} implementation from scikit-learn (\cite{scikit-learn}, version~1.6.1). 
We employed the \texttt{lbfgs} solver with a maximum of 1000 iterations and set \texttt{class\_weight} to \texttt{balanced}. 
We applied $\ell_{2}$ regularization with inverse regularization strength $C = 1.0$, fixed the random seed to 0 for reproducibility, and disabled the intercept term by setting \texttt{fit\_intercept} to \texttt{False}. 
All other parameters were left at their default values.

\paragraph{Comparison between entropy- and $\ell_{1/2}$-regularizers} 
\Cref{fig:subset-cifar10-compare} compares \regretent and \regretlhalf on five randomly selected CIFAR-10 subsets, each formed by sampling 1,000 unlabeled examples per class. Similarly, \Cref{fig:subset-imagenet-compare} shows the results on five ImageNet-50 subsets, where each subset contains 400 randomly sampled examples per class.

\begin{figure}[tbp]
\centering
  \footnotesize
\begin{tikzpicture}
\node[] at (0.2,2.1) {\color{black}\regretent};
\node[] at (6.2,2.1) {\color{black}\regretlhalf};
\node[inner sep=0pt] (a1) at (0,0) {\includegraphics[width=5.5cm]{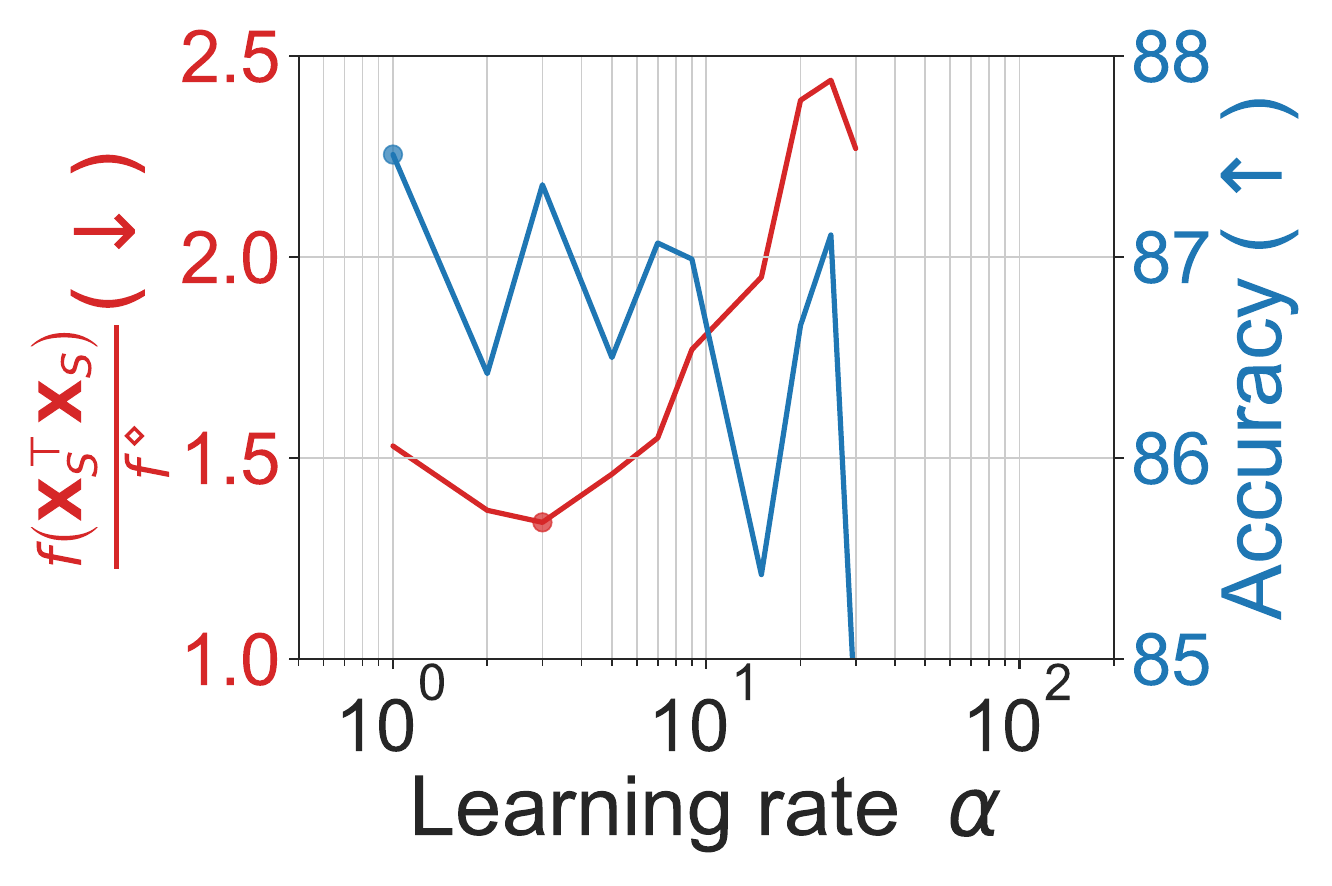}};
\node[inner sep=0pt] (b1) at (6,0) {\includegraphics[width=5.5cm]{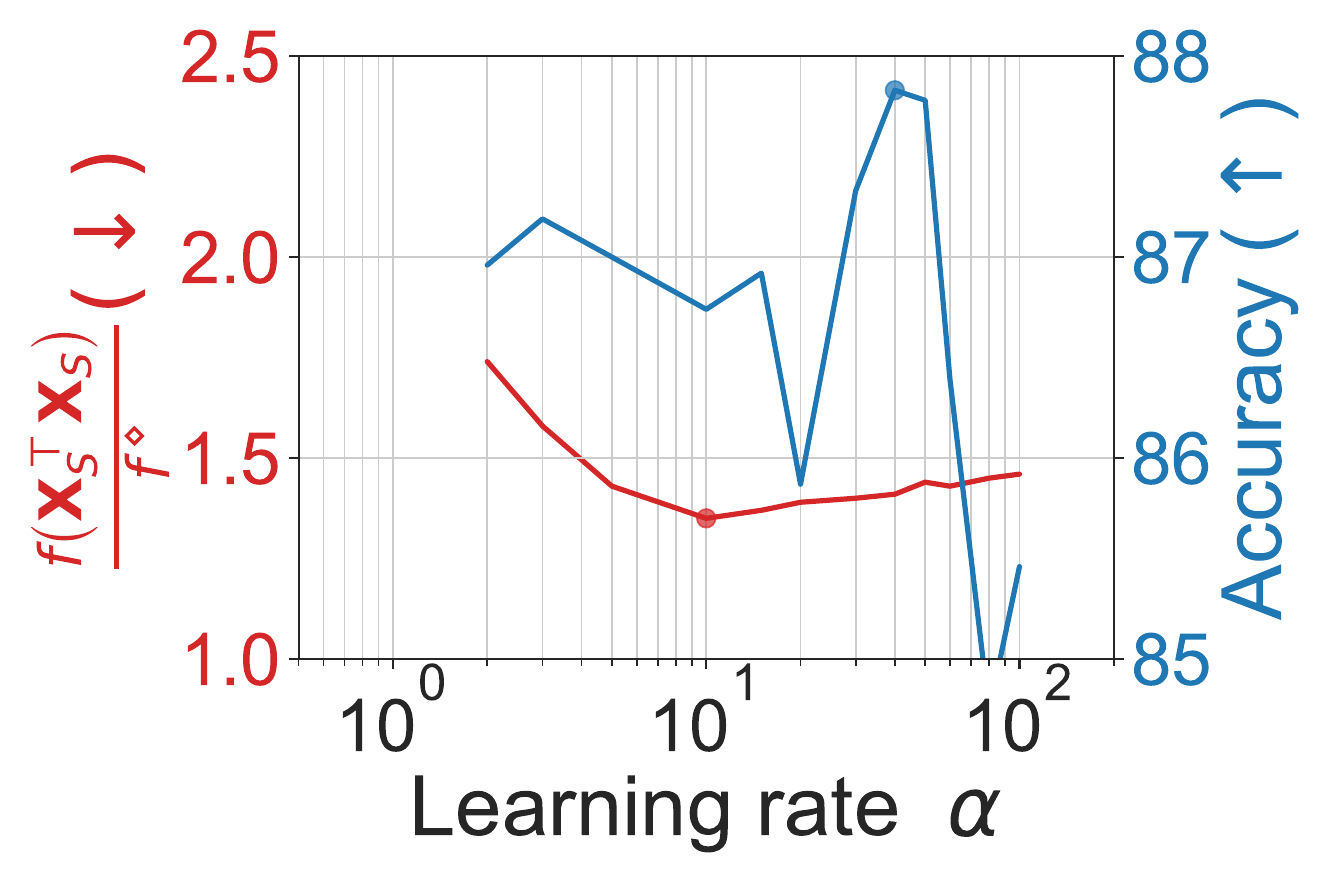}};
\node[inner sep=0pt] (a2) at (0,-3.6) {\includegraphics[width=5.5cm]{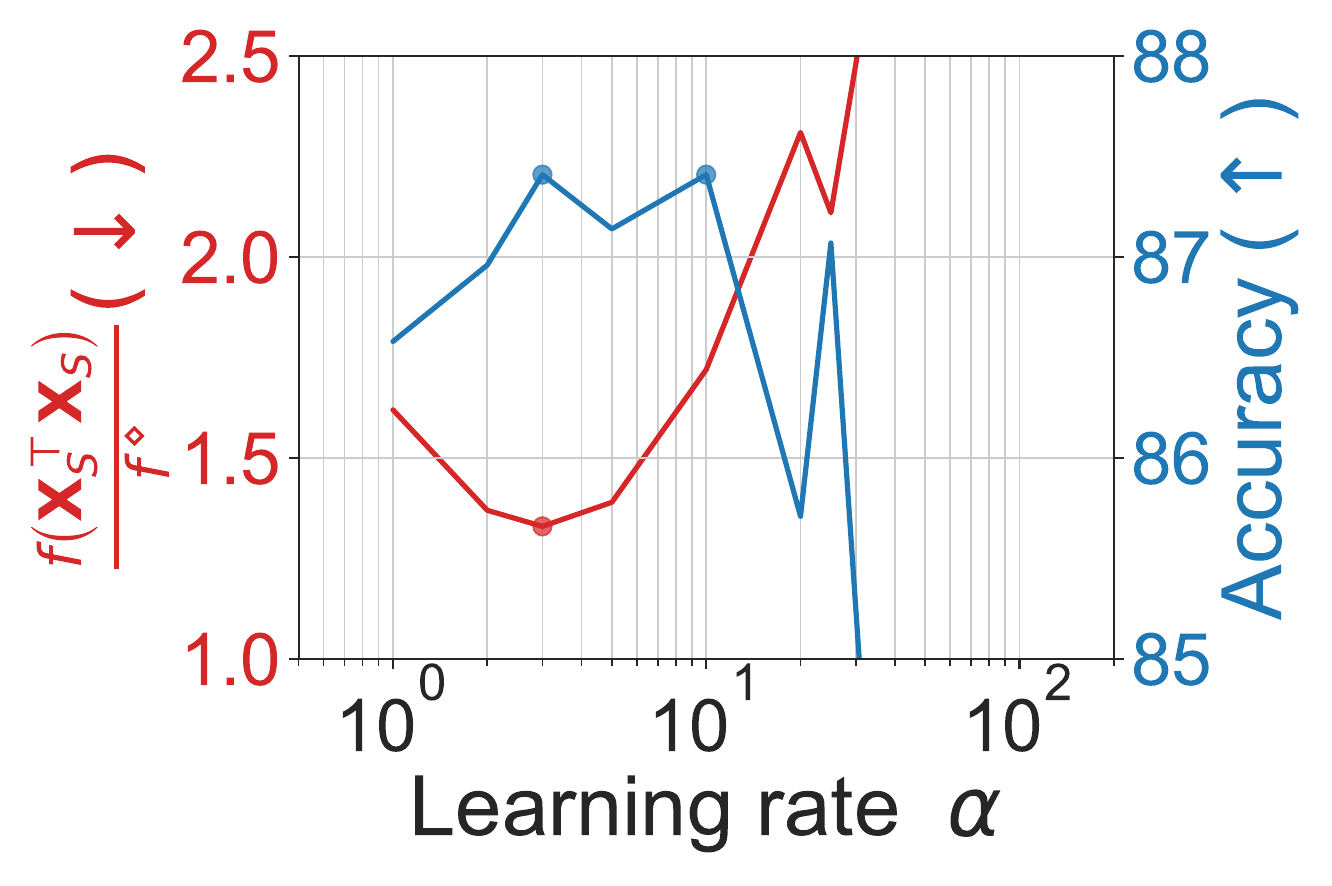}};
\node[inner sep=0pt] (b2) at (6,-3.6) {\includegraphics[width=5.5cm]{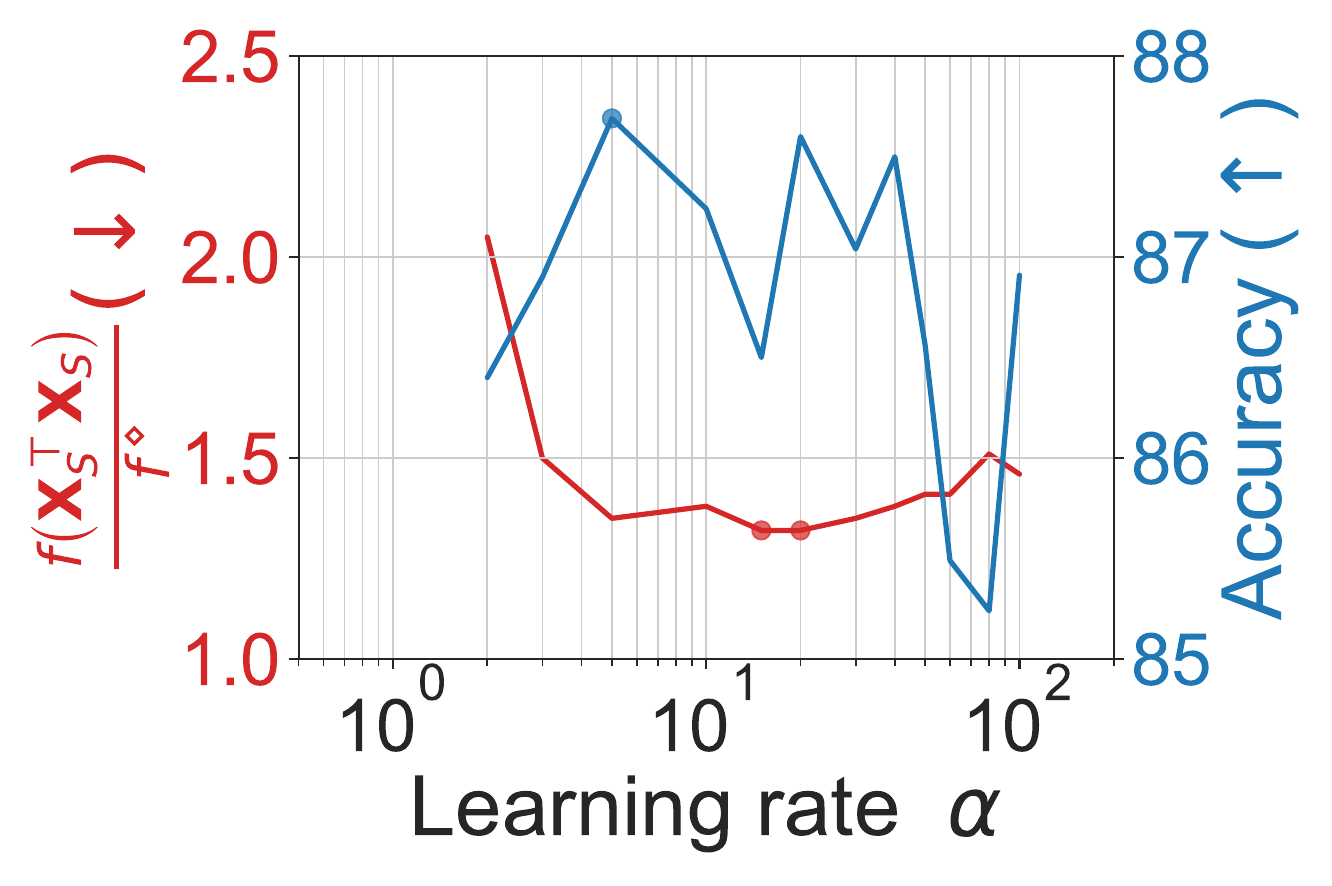}};
\node[inner sep=0pt] (a3) at (0,-7.2) {\includegraphics[width=5.5cm]{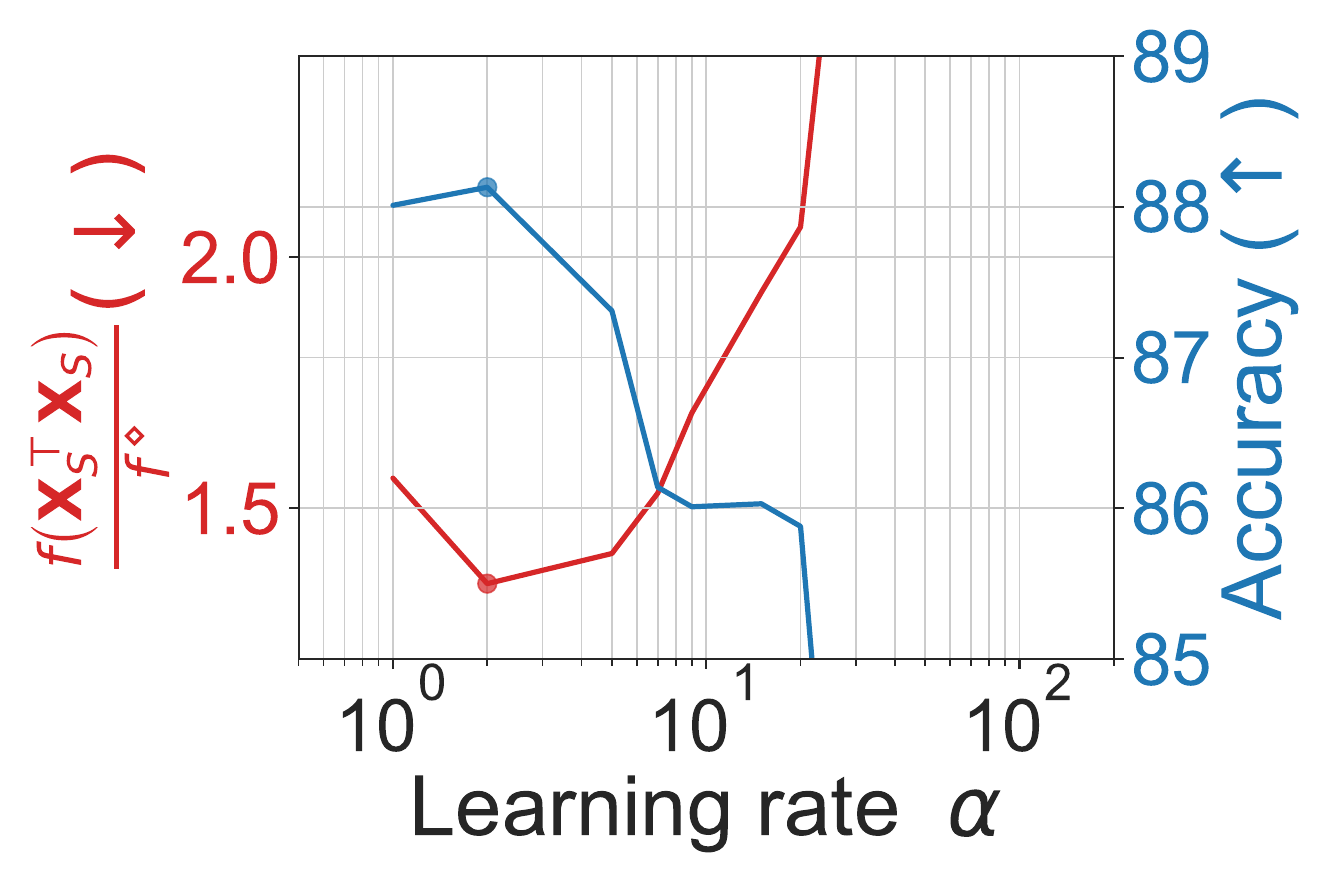}};
\node[inner sep=0pt] (b3) at (6,-7.2) {\includegraphics[width=5.5cm]{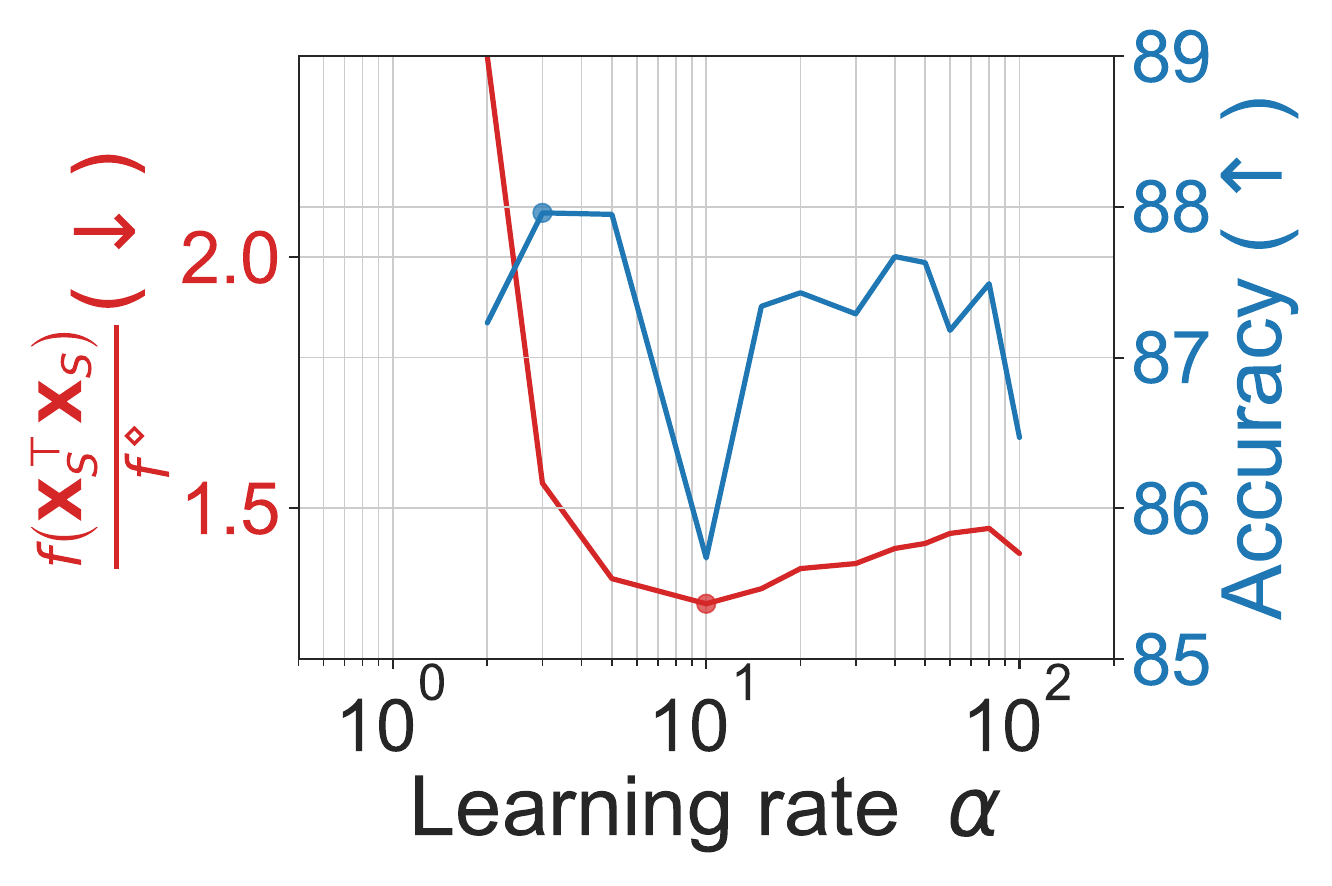}};
\node[inner sep=0pt] (a4) at (0,-10.8) {\includegraphics[width=5.5cm]{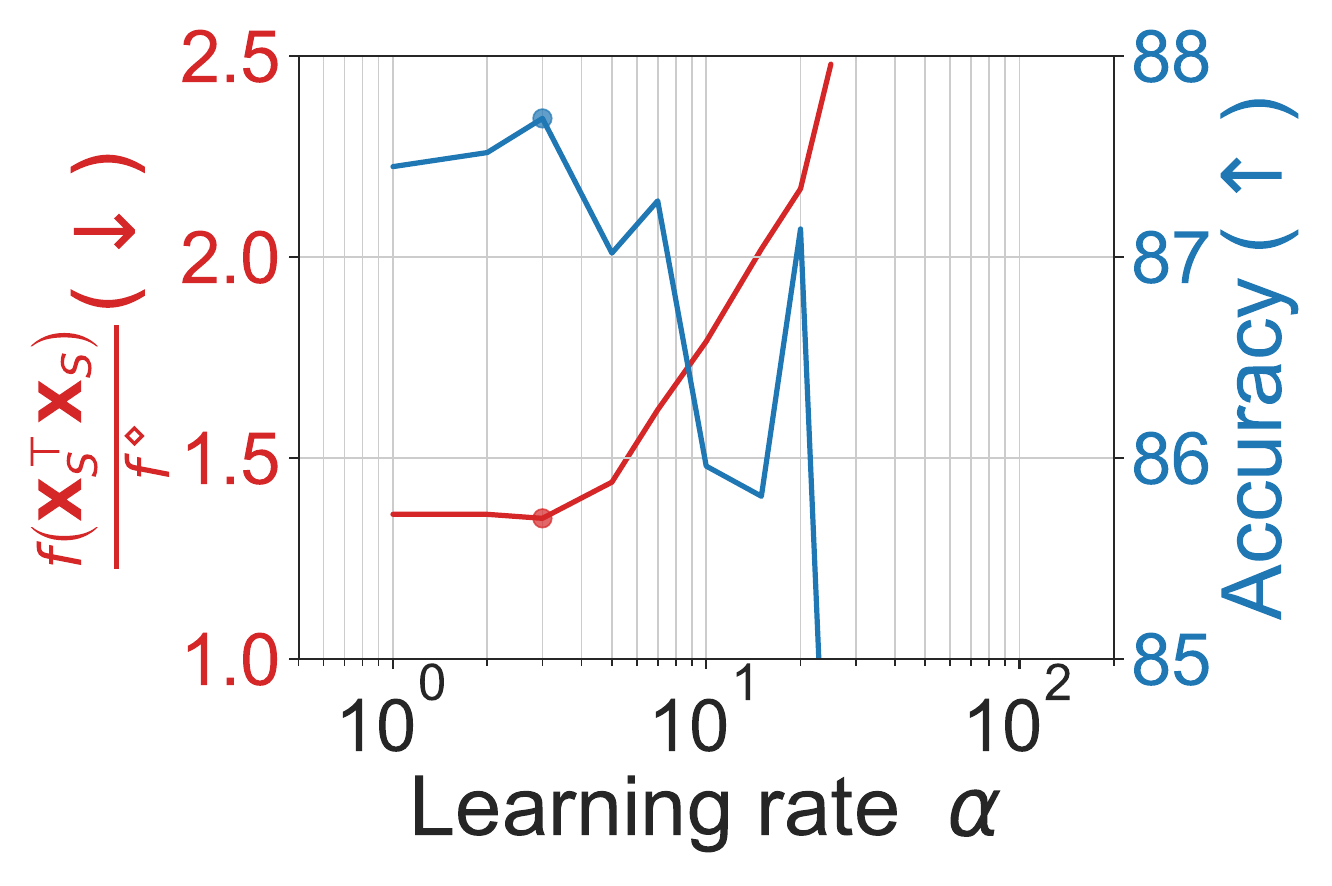}};
\node[inner sep=0pt] (b4) at (6,-10.8) {\includegraphics[width=5.5cm]{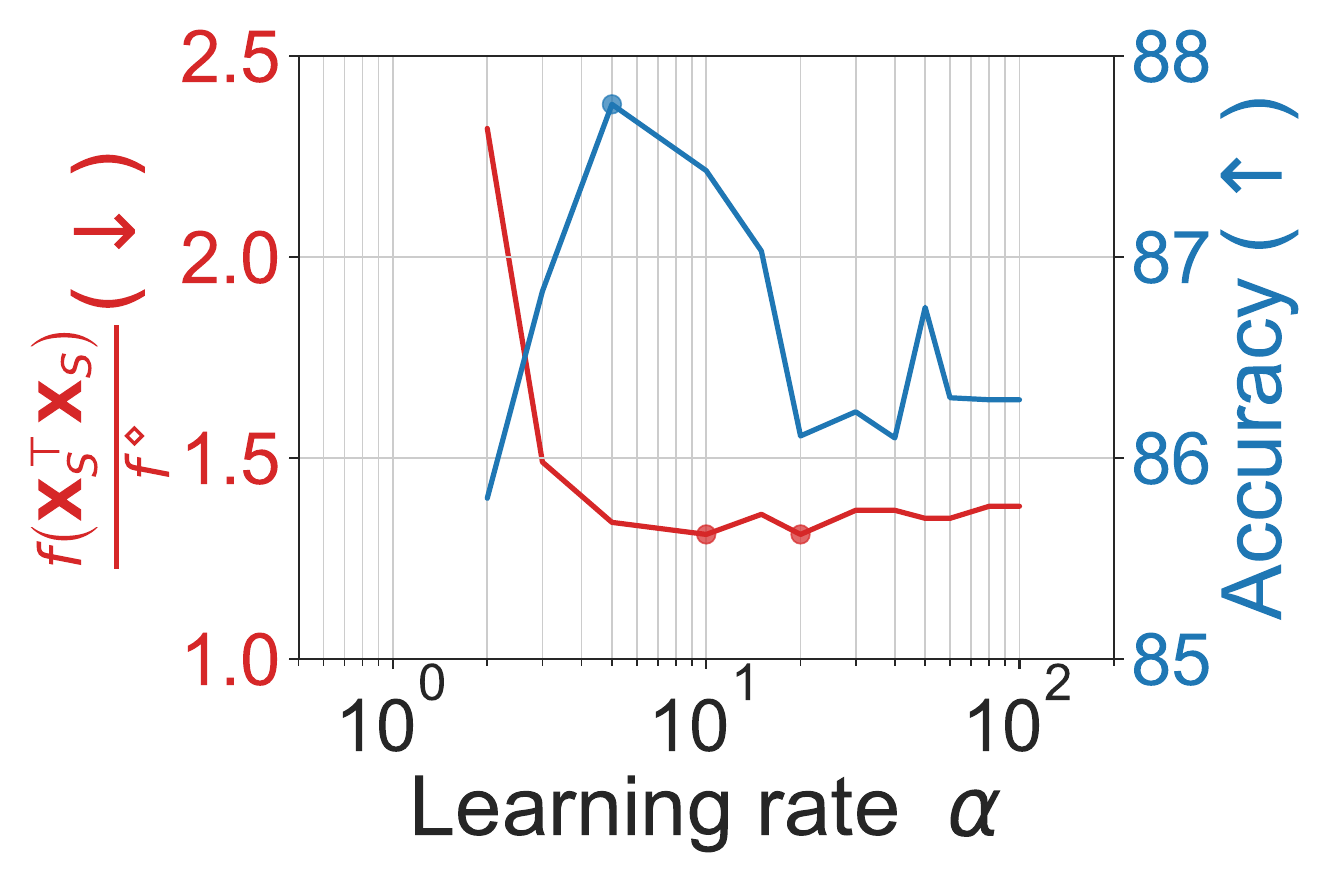}};
\node[inner sep=0pt] (a5) at (0,-14.4) {\includegraphics[width=5.5cm]{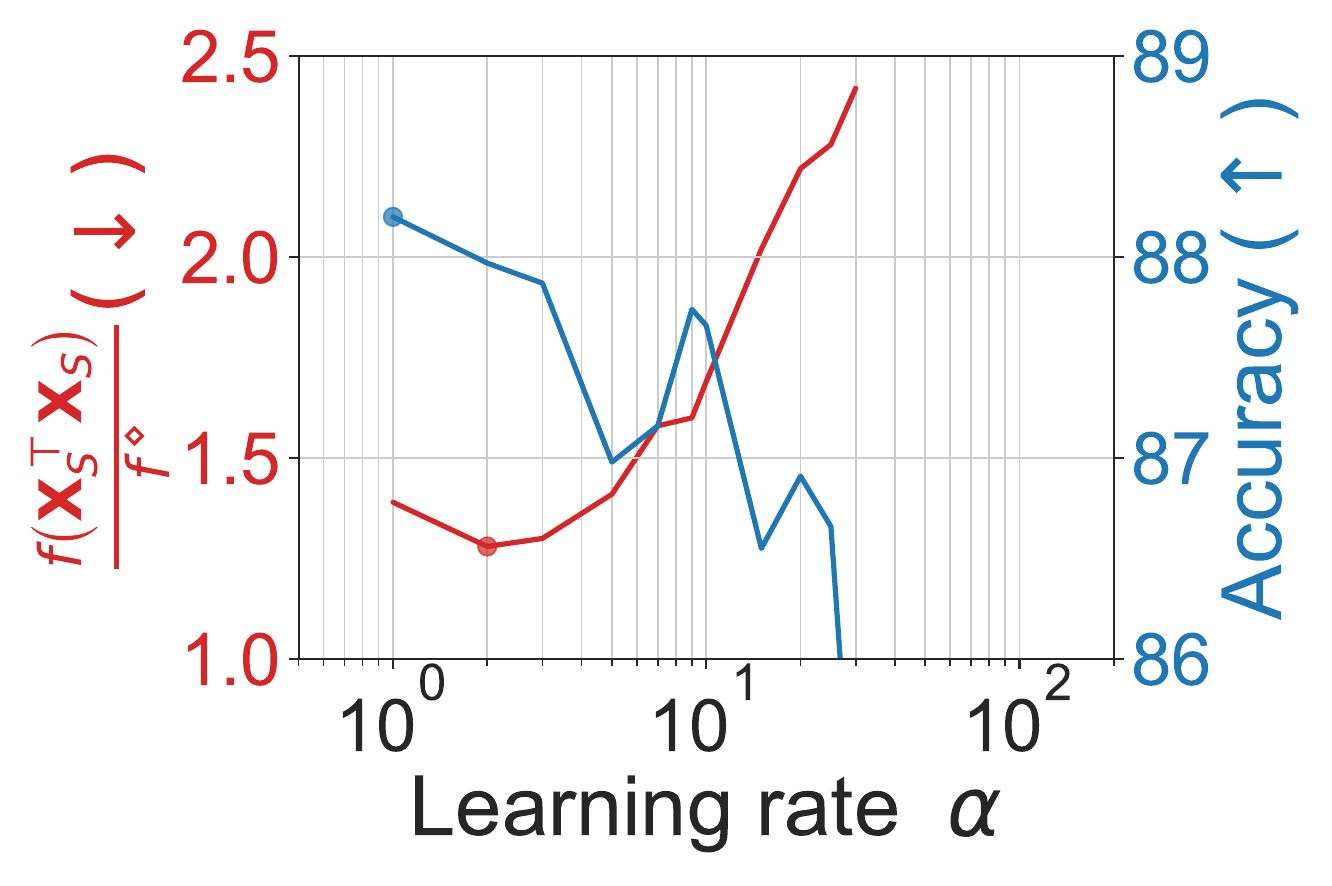}};
\node[inner sep=0pt] (b5) at (6,-14.4) {\includegraphics[width=5.5cm]{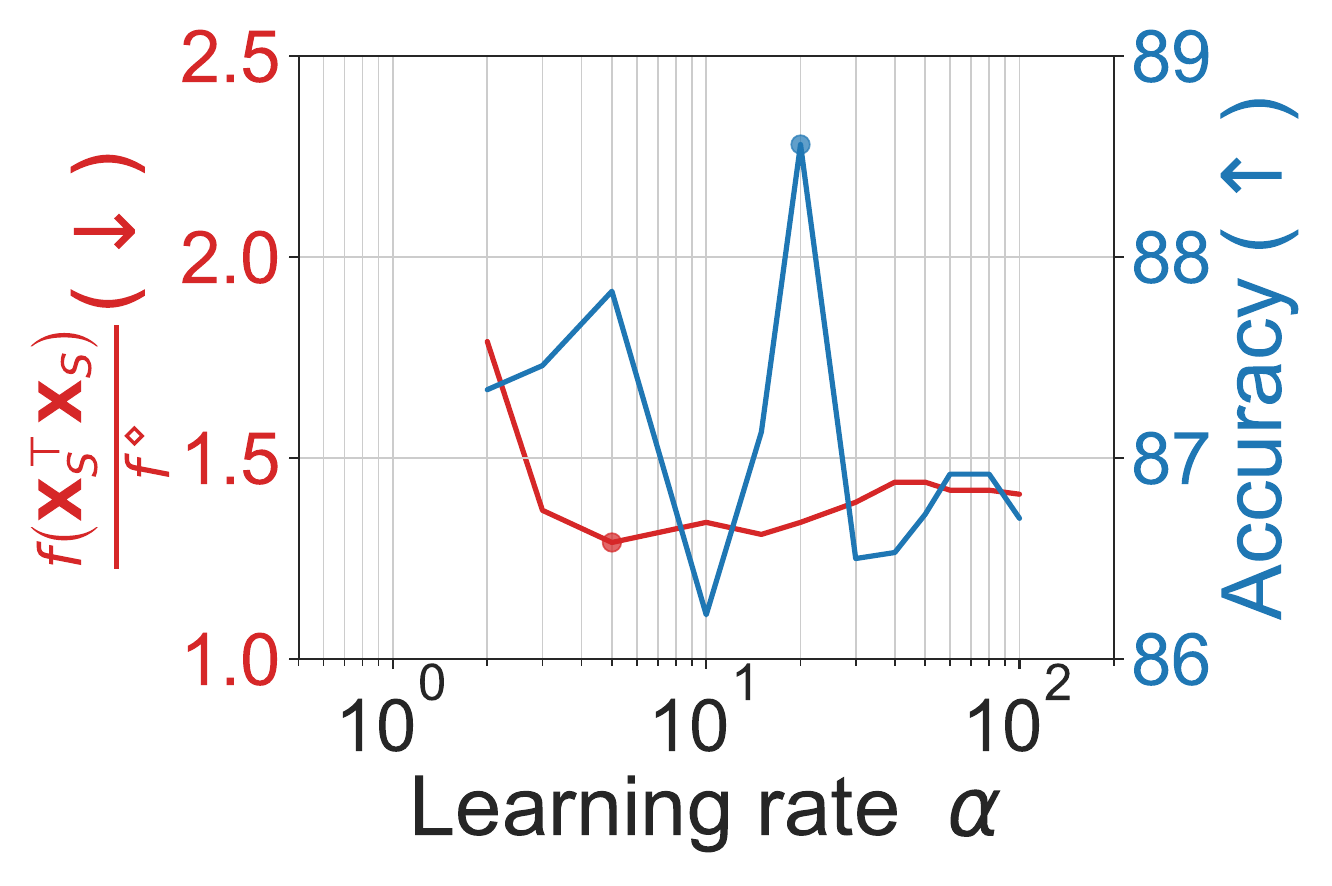}};
\node[rotate=90,anchor=north] at (-3.6, 0.2) {\color{black}Subset \rom{1}};
\node[rotate=90,anchor=north] at (-3.6, -3.4) {\color{black}Subset \rom{2}};
\node[rotate=90,anchor=north] at (-3.6, -7) {\color{black}Subset \rom{3}};
\node[rotate=90,anchor=north] at (-3.6, -10.6) {\color{black}Subset \rom{4}};
\node[rotate=90,anchor=north] at (-3.6, -14.2) {\color{black}Subset \rom{5}};
\draw[] (3,2.4) -- (3,-16.2);
\draw[] (-2.9,2.4) -- (-2.9,-16.2);
\draw (-4,-16.2) rectangle (8.8,1.8);
\draw (-4,1.8) rectangle (8.8,2.4);
\end{tikzpicture}
\caption{\color{black}Comparison of \regretent (left) and \regretlhalf (right) on five CIFAR-10 subsets. For each subset, we select $k=50$ points. Each subset is constructed by randomly sampling 1,000 samples per class. We use PCA-reduced features with dimension of 40, thus the input design pool for each subset data is $\X\in \mathbb{R}^{10,000\times 40}$.  The red lines in the plot represent relative value of the objective function $\frac{f(\X_S^\top \X_S)}{f^\diamond}$, where $\X_S$ is the selected samples and $f^\diamond$ is the optimal value of the relaxed problem~\cref{eq:lp}. The blue lines represent the logistic regression prediction accuracy. The dots on each line represent the optimal points.}
\label{fig:subset-cifar10-compare}
\end{figure}

\begin{figure}[tbp]
\centering
  \footnotesize
\begin{tikzpicture}
\node[] at (0.2,2.1) {\color{black}\regretent};
\node[] at (6.2,2.1) {\color{black}\regretlhalf};
\node[inner sep=0pt] (a1) at (0,0) {\includegraphics[width=5.5cm]{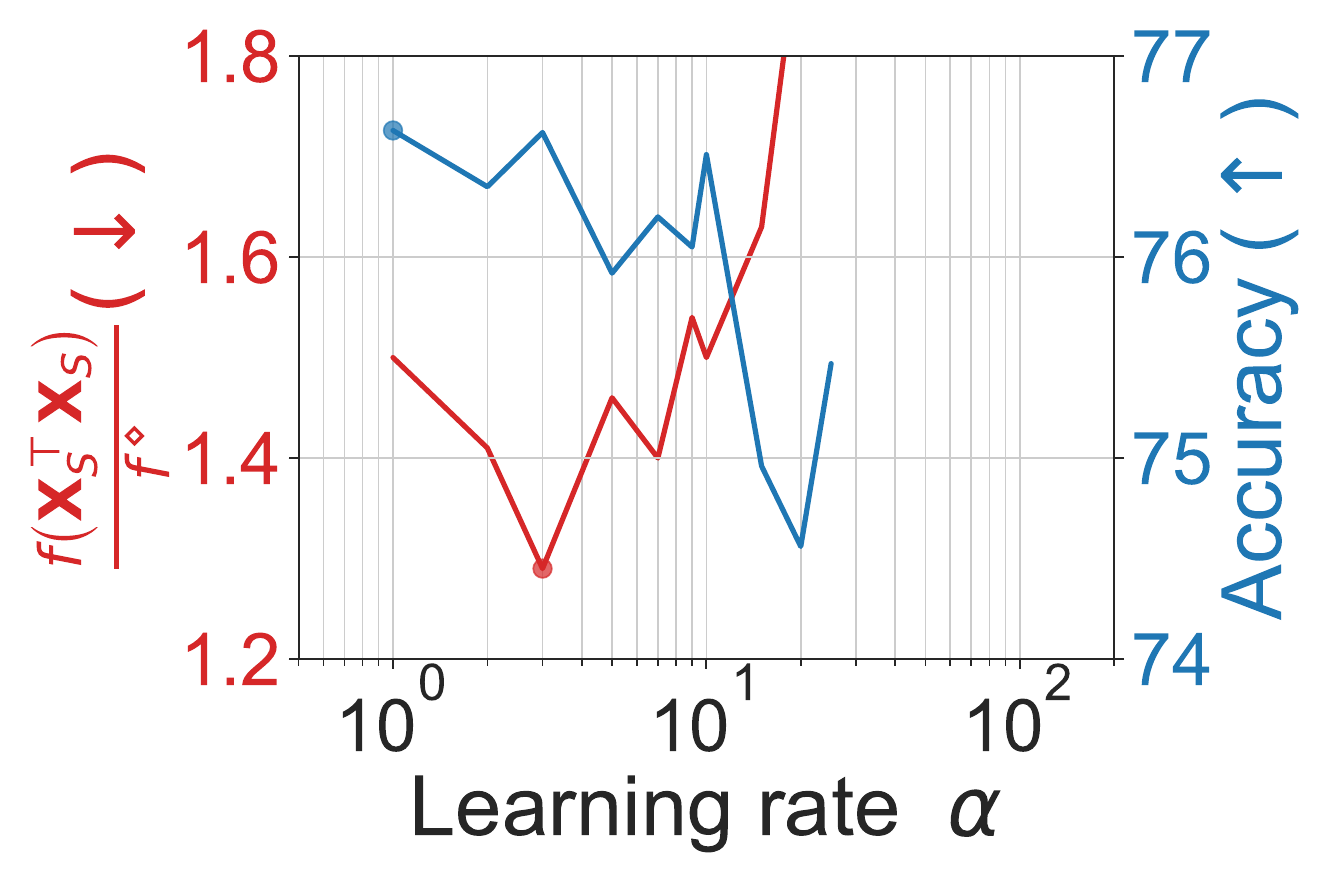}};
\node[inner sep=0pt] (b1) at (6,0) {\includegraphics[width=5.5cm]{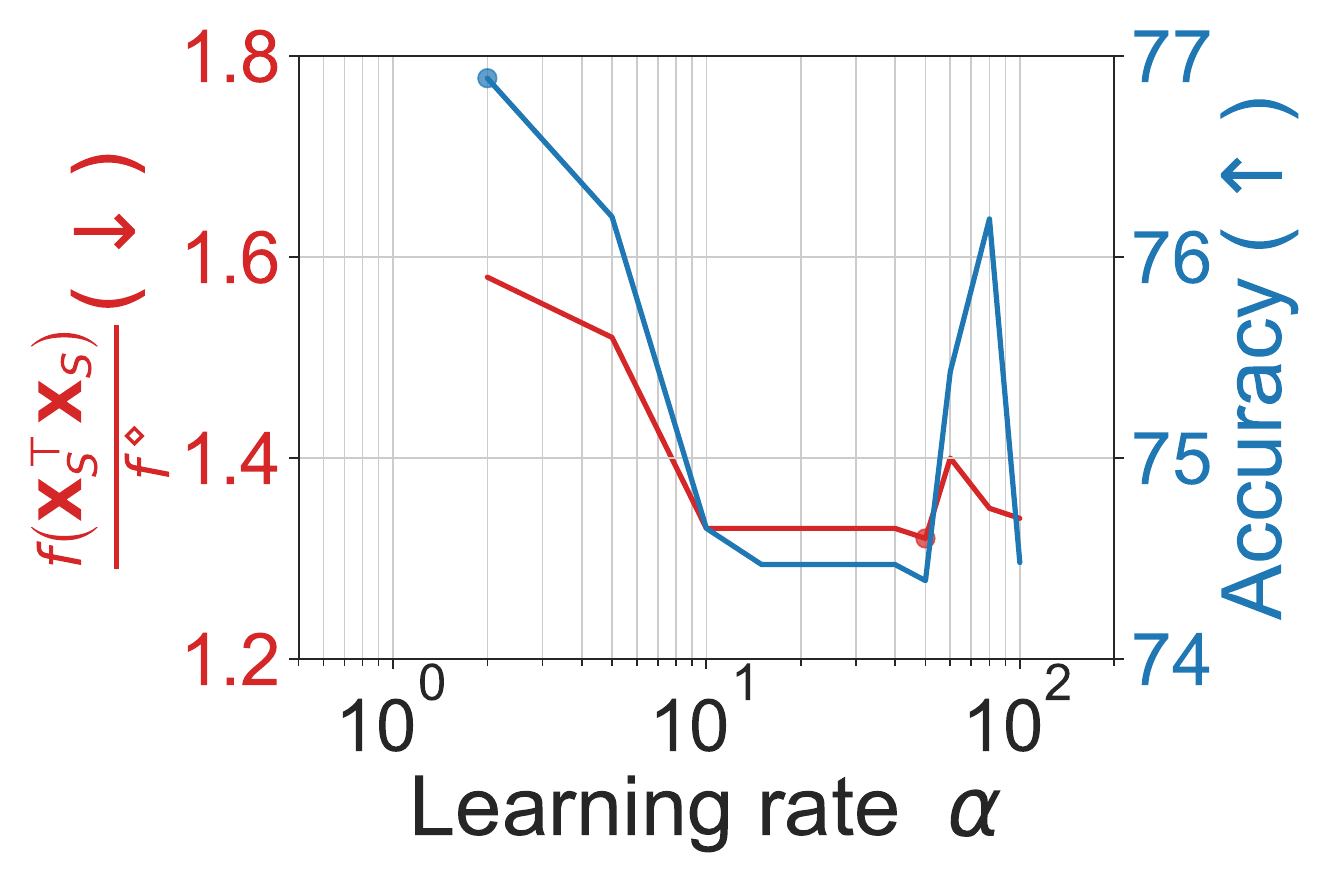}};
\node[inner sep=0pt] (a2) at (0,-3.6) {\includegraphics[width=5.5cm]{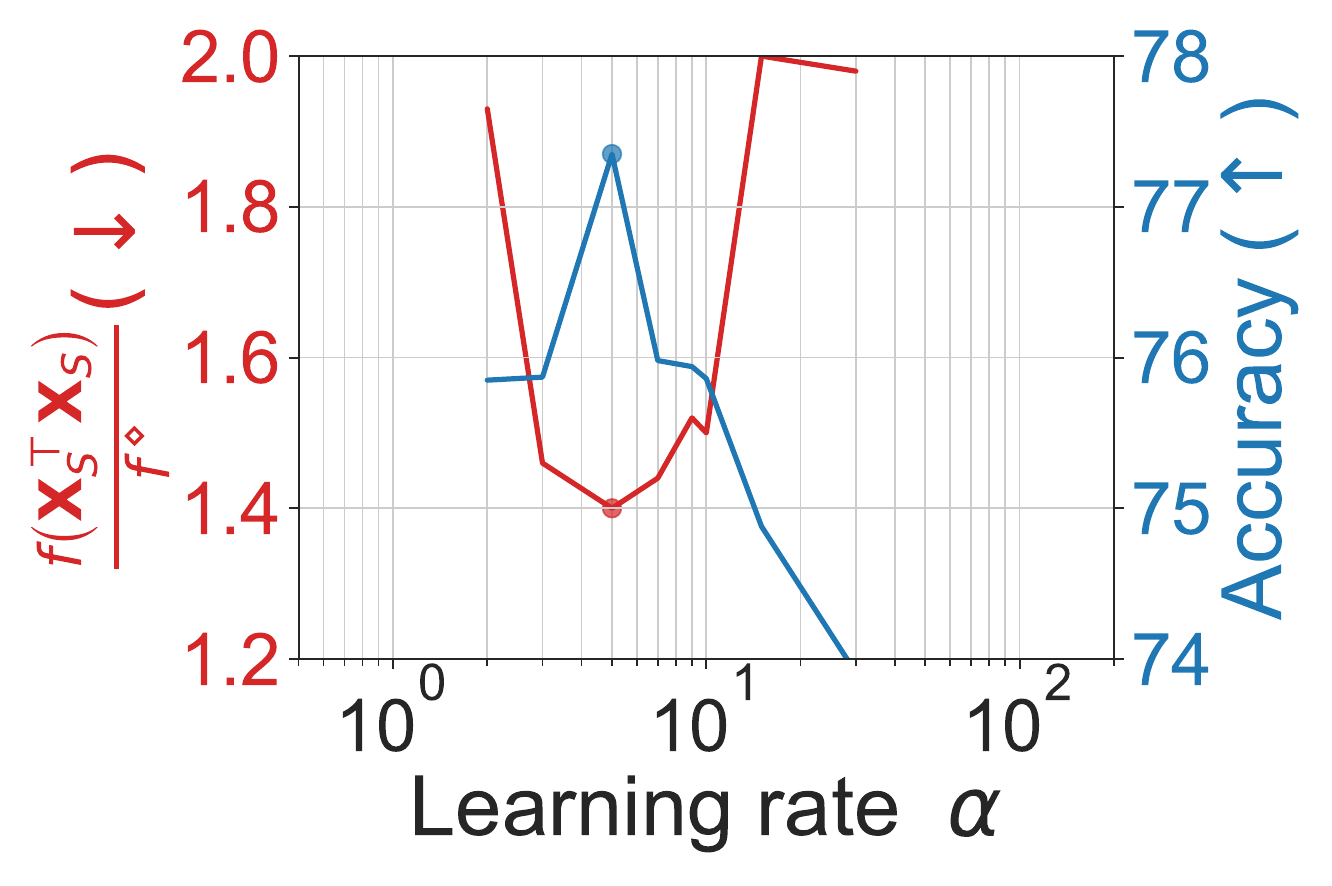}};
\node[inner sep=0pt] (b2) at (6,-3.6) {\includegraphics[width=5.5cm]{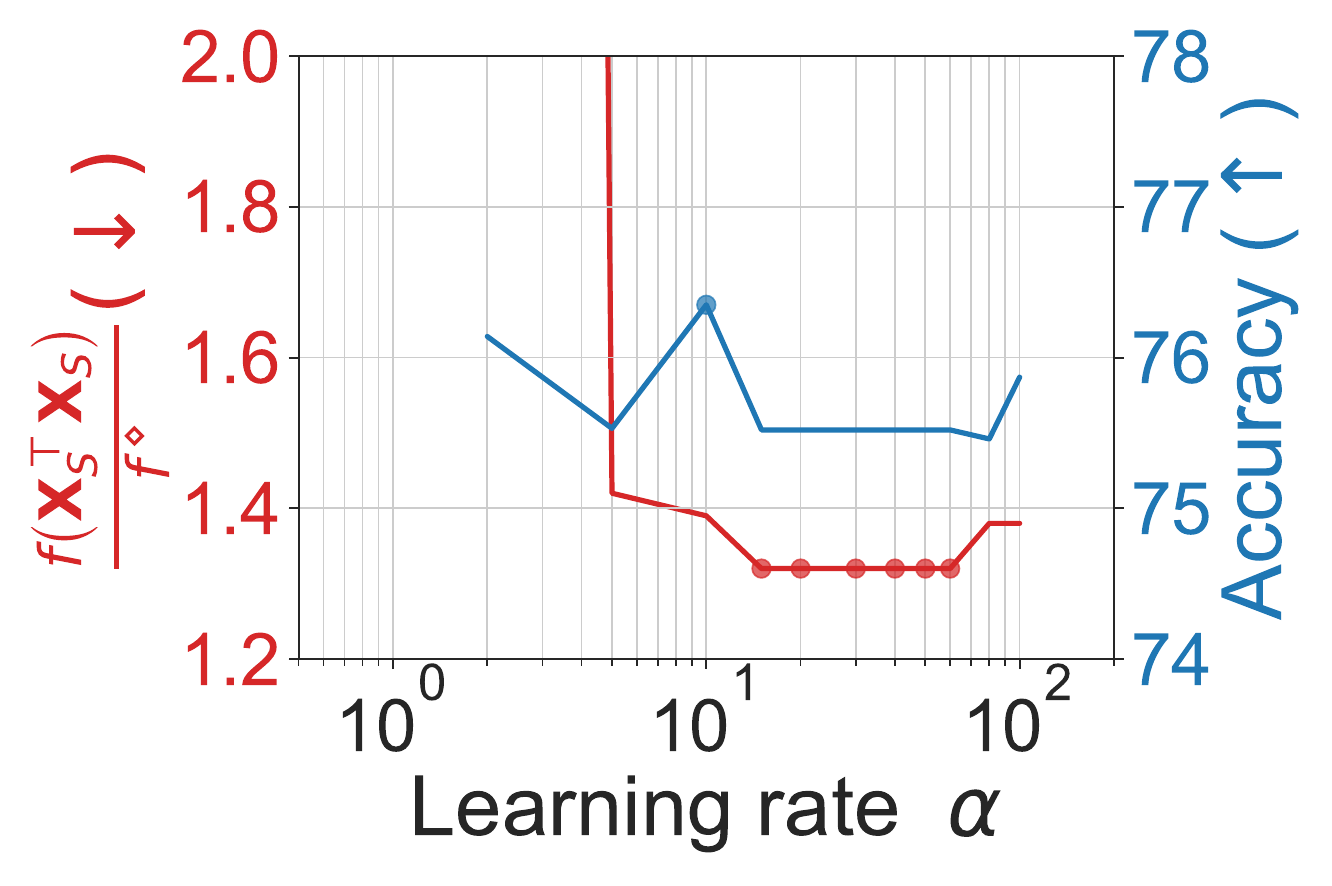}};
\node[inner sep=0pt] (a3) at (0,-7.2) {\includegraphics[width=5.5cm]{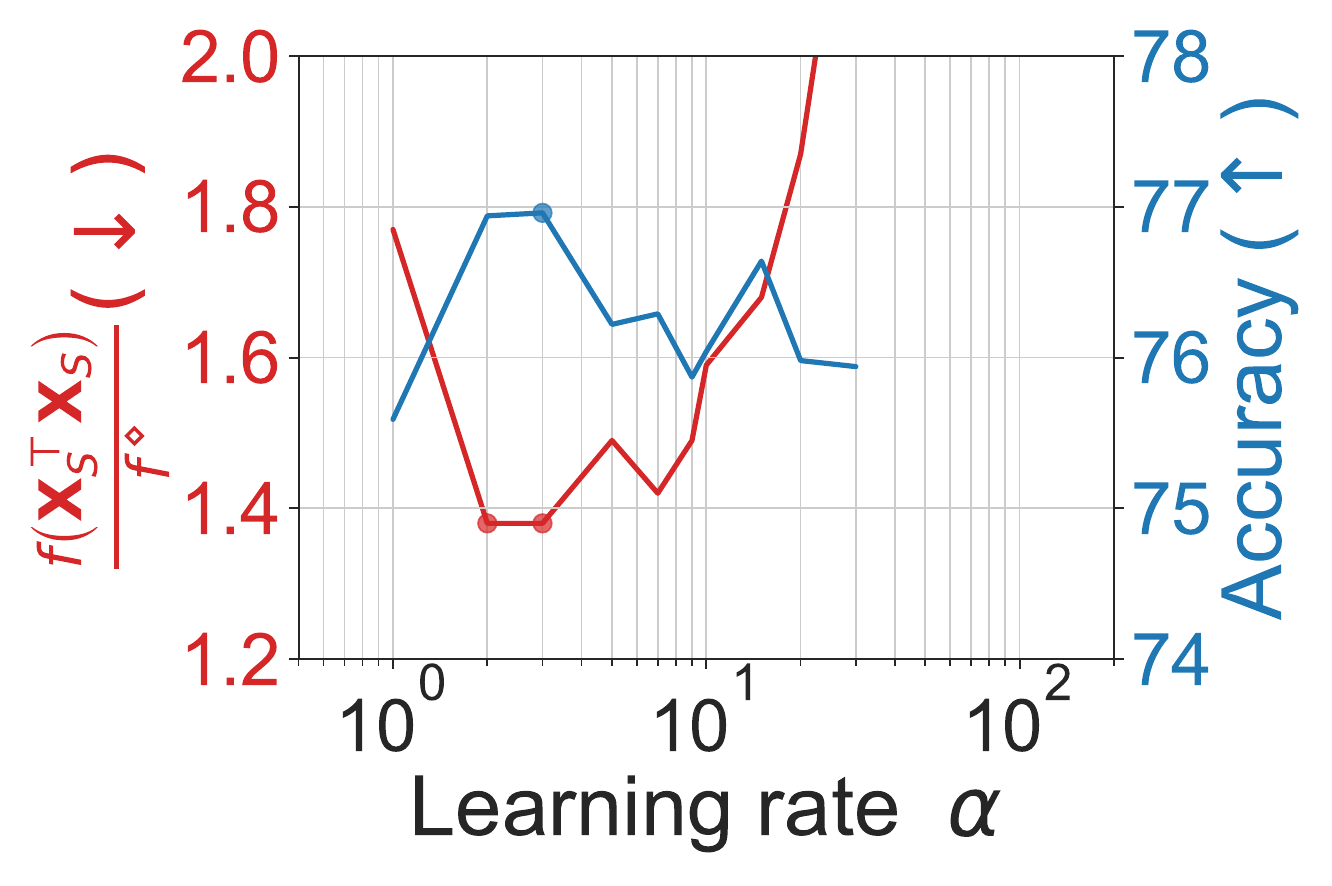}};
\node[inner sep=0pt] (b3) at (6,-7.2) {\includegraphics[width=5.5cm]{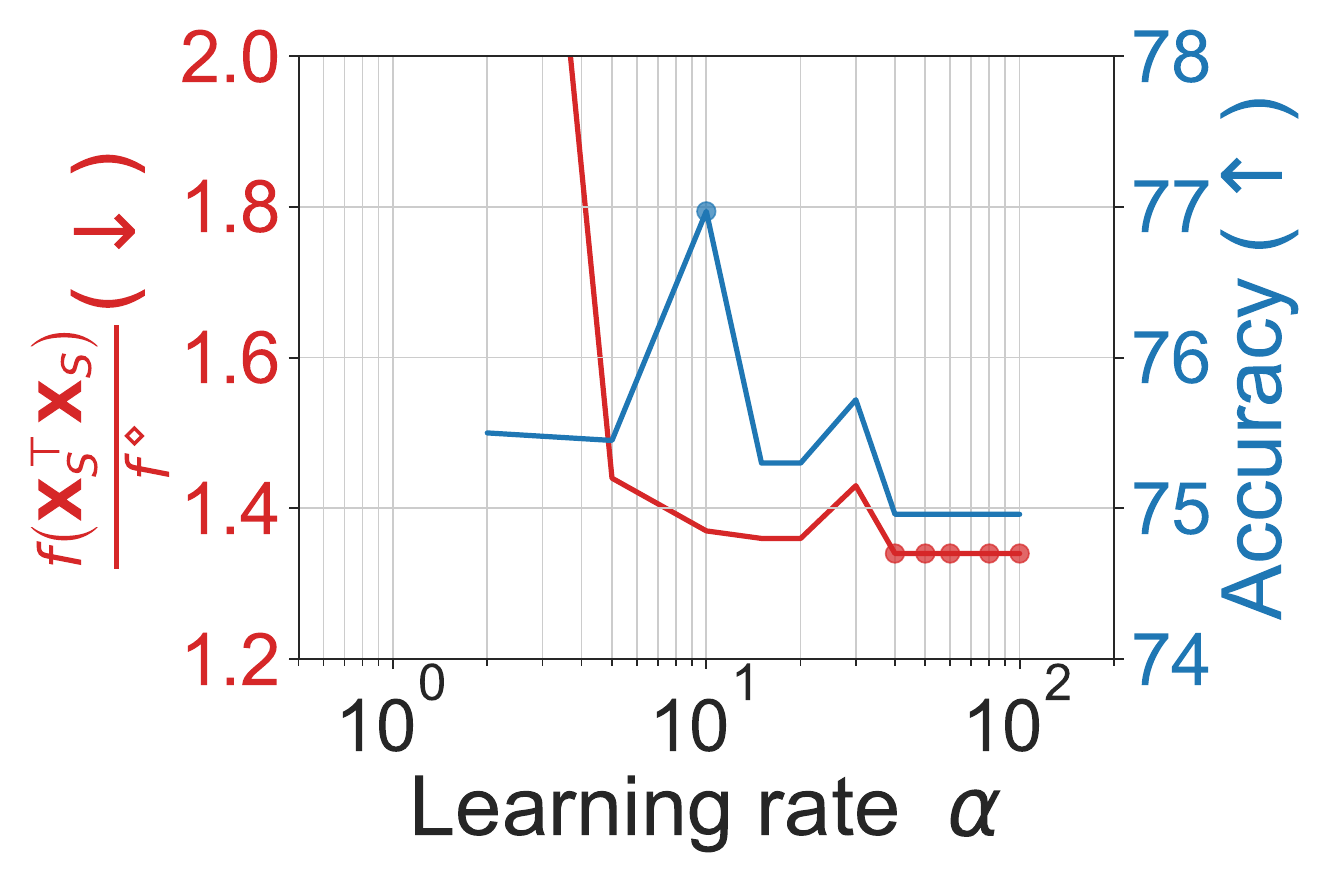}};
\node[inner sep=0pt] (a4) at (0,-10.8) {\includegraphics[width=5.5cm]{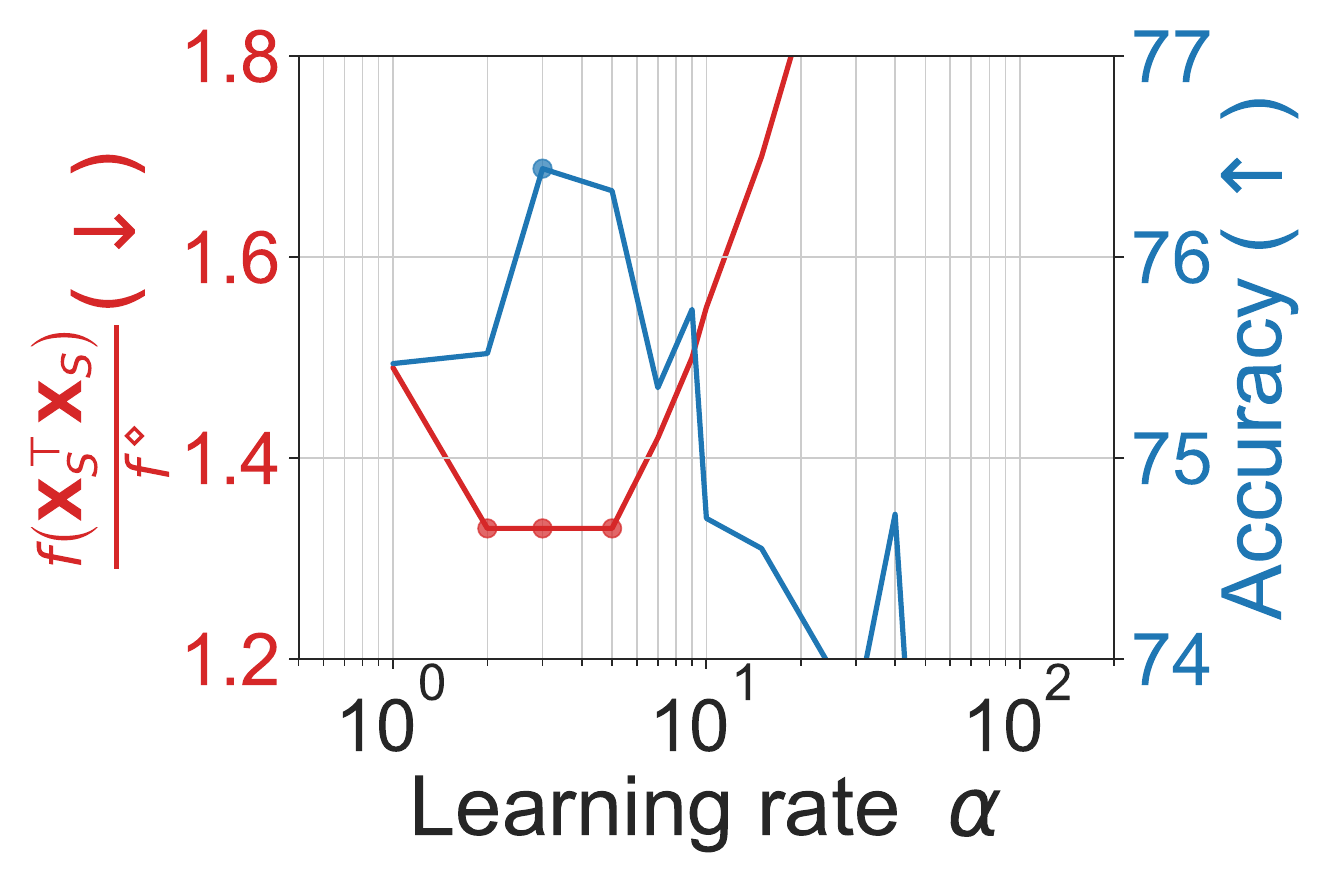}};
\node[inner sep=0pt] (b4) at (6,-10.8) {\includegraphics[width=5.5cm]{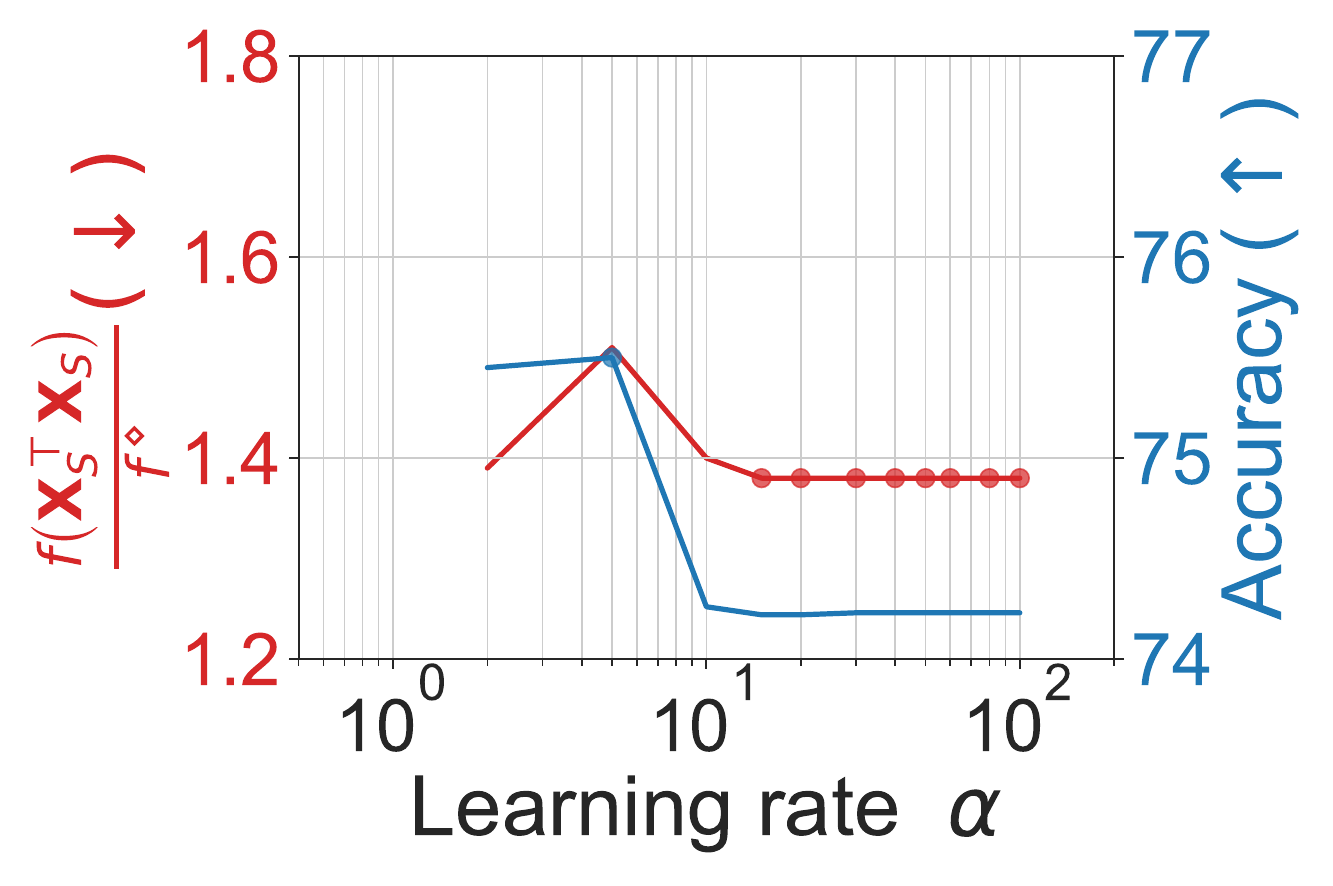}};
\node[inner sep=0pt] (a5) at (0,-14.4) {\includegraphics[width=5.5cm]{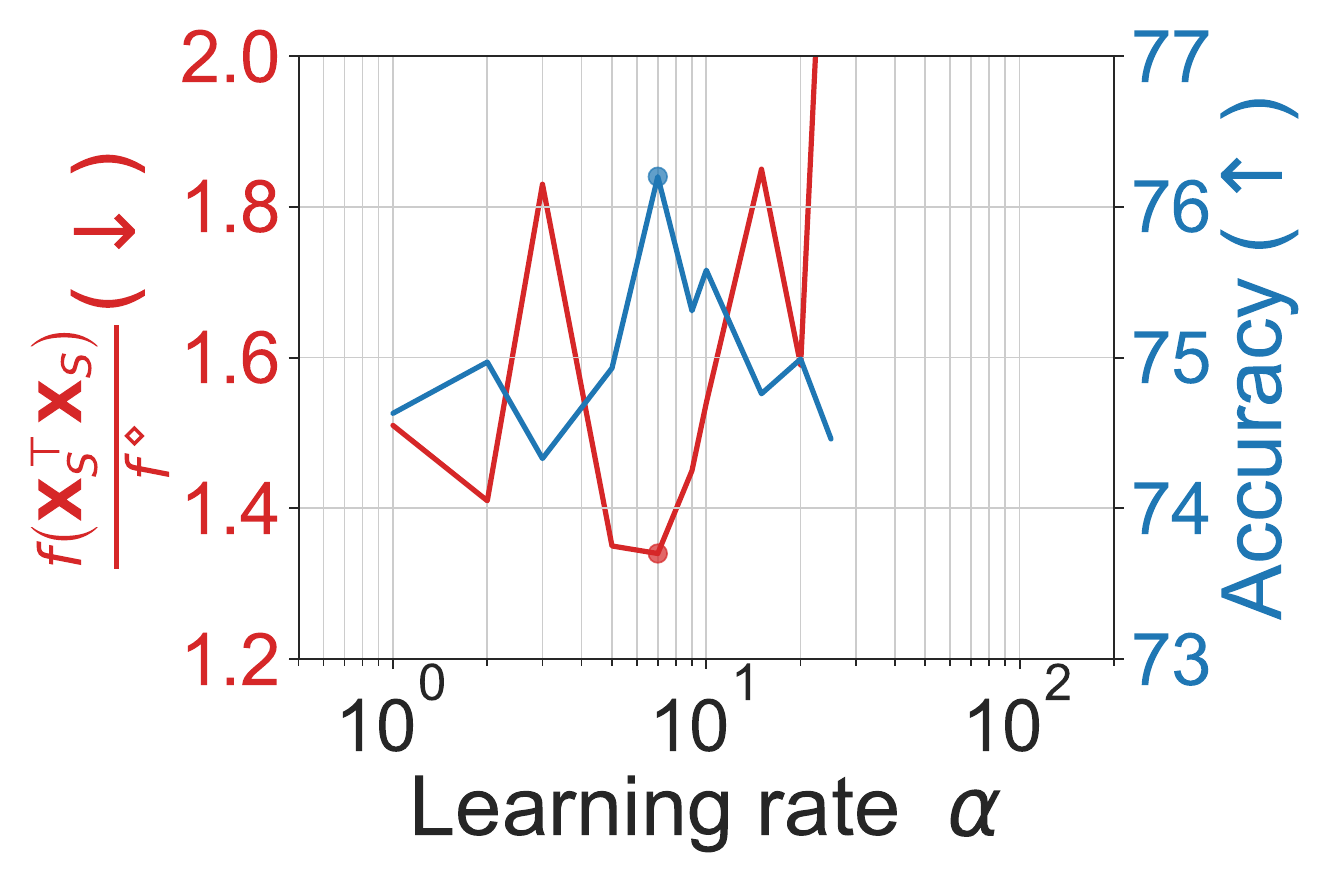}};
\node[inner sep=0pt] (b5) at (6,-14.4) {\includegraphics[width=5.5cm]{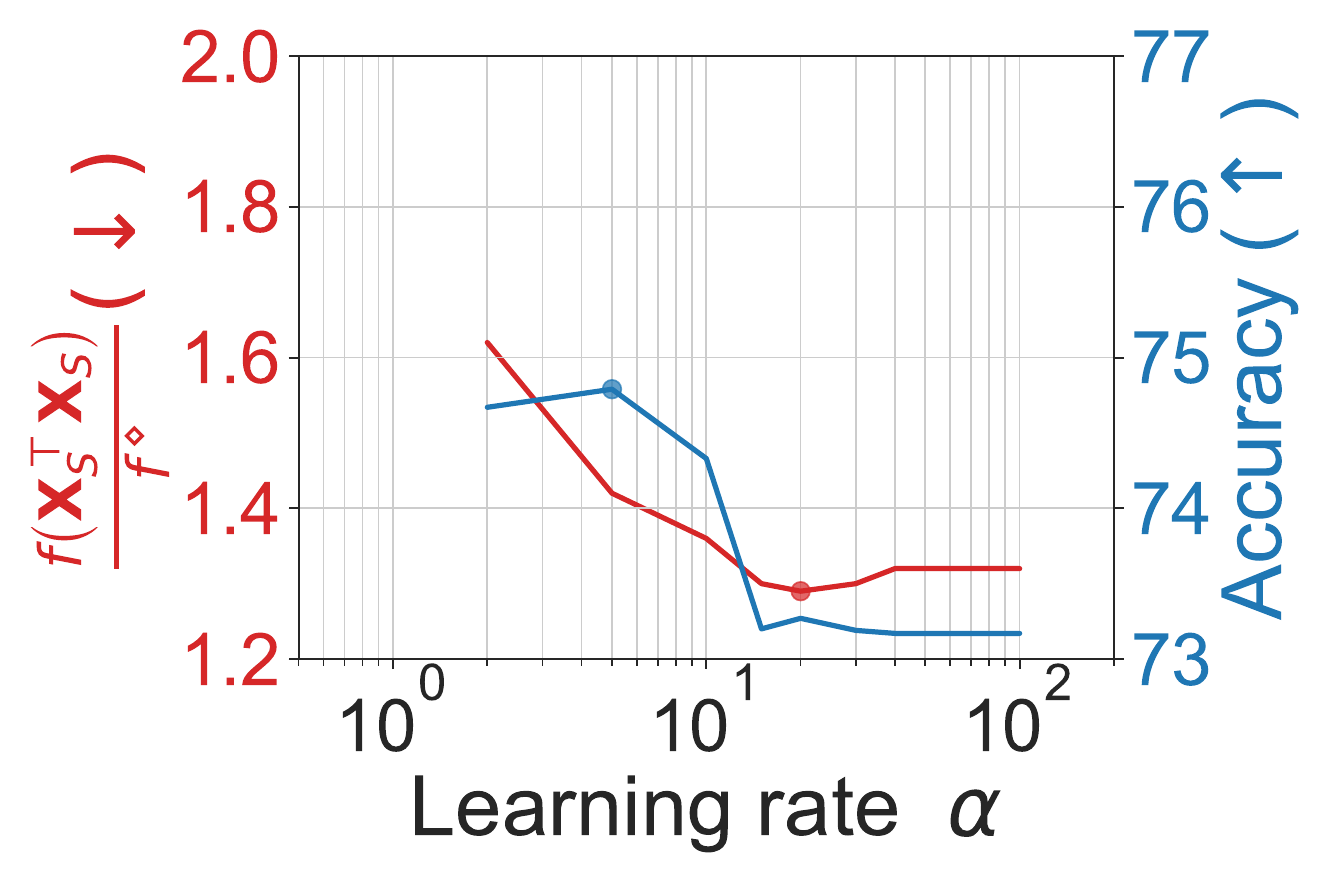}};
\node[rotate=90,anchor=north] at (-3.6, 0.2) {\color{black}Subset \rom{1}};
\node[rotate=90,anchor=north] at (-3.6, -3.4) {\color{black}Subset \rom{2}};
\node[rotate=90,anchor=north] at (-3.6, -7) {\color{black}Subset \rom{3}};
\node[rotate=90,anchor=north] at (-3.6, -10.6) {\color{black}Subset \rom{4}};
\node[rotate=90,anchor=north] at (-3.6, -14.2) {\color{black}Subset \rom{5}};
\draw[] (3,2.4) -- (3,-16.2);
\draw[] (-2.9,2.4) -- (-2.9,-16.2);
\draw (-4,-16.2) rectangle (8.8,1.8);
\draw (-4,1.8) rectangle (8.8,2.4);
\end{tikzpicture}
\caption{\color{black}Comparison of \regretent (left) and \regretlhalf (right) on five ImageNet-50 subsets. For each subset, we select $k=50$ points. Each subset is constructed by randomly sampling 400 samples per class. We use PCA-reduced features with dimension of 50, thus the input design pool for each subset data is $\X\in \mathbb{R}^{20,000\times 50}$.  The red lines in the plot represent relative value of the objective function $\frac{f(\X_S^\top \X_S)}{f^\diamond}$, where $\X_S$ is the selected samples and $f^\diamond$ is the optimal value of the relaxed problem~\cref{eq:lp}. The blue lines represent the logistic regression prediction accuracy. The dots on each line represent the optimal points.}
\label{fig:subset-imagenet-compare}
\end{figure}

\paragraph{Samples selected in CIFAR-10 test}
\Cref{fig:cifar10-images} displays the 40 images selected by each method on CIFAR-10 with 40-dimensional PCA-reduced features.

\paragraph{Samples selected in ImageNet-50 test}\Cref{table:image50-classname} lists the names of the 50 ImageNet-50 classes. Finally, \Cref{fig:imagenet-fig-1,fig:imagenet-fig-2,fig:imagenet-fig-3} present the 100 images selected by each sampling method on ImageNet-50 with 50-dimensional PCA-reduced features.

\begin{figure}[tbp]
\centering
  \footnotesize
\begin{tikzpicture}
\node[inner sep=0pt] (a1) at (0,0) {\includegraphics[width=12.5cm,height = 2.0cm]{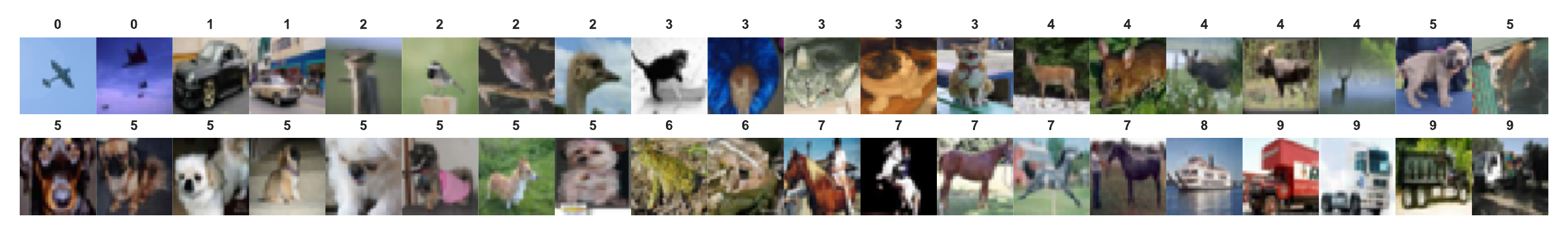}};
\node[rotate=90,anchor=north]  at (-6.6,0) {\textbf{Random}};
\node[inner sep=0pt] (a2) at (0,-1.8) {\includegraphics[width=12.5cm,height = 2.0cm]{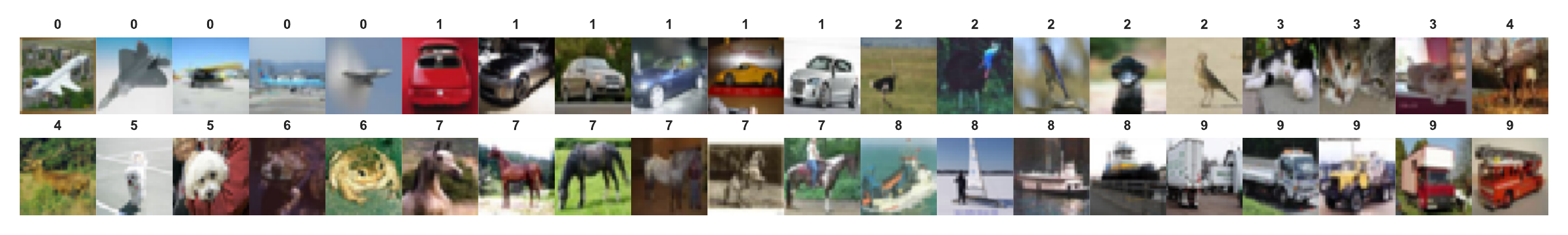}};
\node[rotate=90,anchor=north]  at (-6.6,-1.8) {\textbf{K-means}};
\node[inner sep=0pt] (a3) at (0,-3.6) {\includegraphics[width=12.5cm,height = 2.0cm]{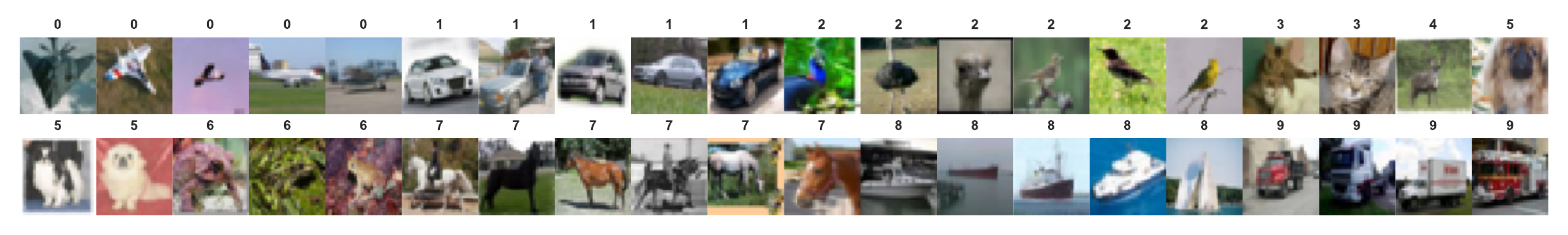}};
\node[rotate=90,anchor=north] at (-6.6,-3.6) {\textbf{RRQR}};
\node[inner sep=0pt] (a4) at (0,-5.4) {\includegraphics[width=12.5cm,height = 2.0cm]{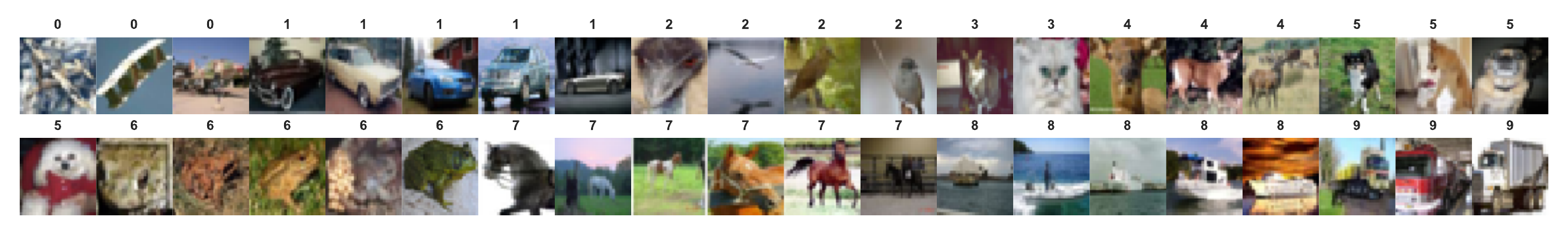}};
\node[rotate=90,anchor=north] at (-6.6,-5.4) {\textbf{MMD}};
\node[inner sep=0pt] (a5) at (0,-7.2) {\includegraphics[width=12.5cm,height = 2.0cm]{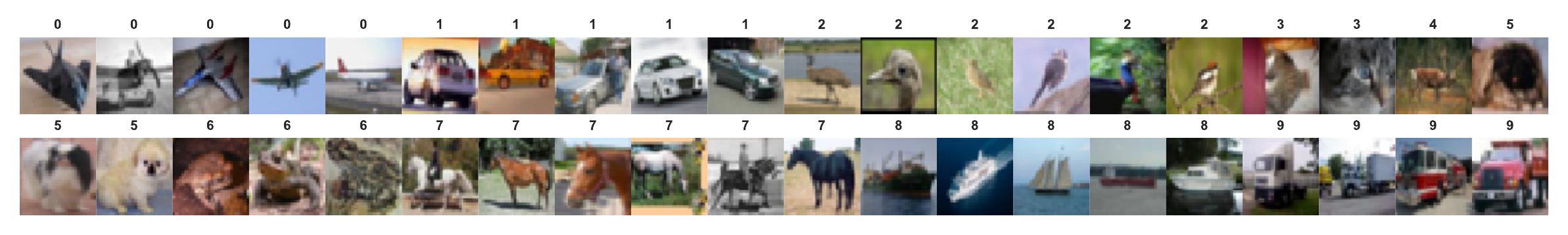}};
\node[rotate=90,anchor=north]  at (-6.6,-7.2) {\textbf{Regret-min}};
\end{tikzpicture}
\caption{40 samples selected by different sampling methods on CIFAR-10 (we use first 40 eigenvectors of the normalized graph Laplacian as features). The number above each image represents its label.}
\label{fig:cifar10-images}
\end{figure}

\begin{table}[t]
\footnotesize
\centering
\caption{Names of classes in ImageNet-50.}
\label{table:image50-classname}
\begin{tabular}{|c|C|c|C|c|C|}
    \hline
\# &class name &\# &class name &\# &class name \\ \hline
0  &combination lock  &1  &spotlight, spot  &2  &bonnet, poke bonnet\\ \hline
3  &gibbon, Hylobates lar  &4  &boxer  &5  &ice lolly, lolly, lollipop, popsicle\\ \hline
6  &Brabancon griffon  &7  &space bar  &8  &lacewing, lacewing fly\\ \hline
9  &African crocodile, Nile crocodile, Crocodylus niloticus  &10  &binder, ring-binder  &11  &dowitcher\\ \hline
12  &cup  &13  &ox  &14  &bubble\\ \hline
15  &bittern  &16  &wallaby, brush kangaroo  &17  &stingray\\ \hline
18  &hook, claw  &19  &Christmas stocking  &20  &parallel bars, bars\\ \hline
21  &canoe  &22  &home theater, home theatre  &23  &clog, geta, patten, sabot\\ \hline
24  &bathing cap, swimming cap  &25  &rock crab, Cancer irroratus  &26  &sea anemone, anemone\\ \hline
27  &measuring cup  &28  &hand blower, blow dryer, blow drier, hair dryer, hair drier  &29  &langur\\ \hline
30  &wooden spoon  &31  &pier  &32  &projectile, missile\\ \hline
33  &chime, bell, gong  &34  &airliner  &35  &dhole, Cuon alpinus\\ \hline
36  &tick  &37  &pillow  &38  &matchstick\\ \hline
39  &kimono  &40  &tripod  &41  &oxygen mask\\ \hline
42  &lorikeet  &43  &tusker  &44  &hummingbird\\ \hline
45  &gong, tam-tam  &46  &dogsled, dog sled, dog sleigh  &47  &whiptail, whiptail lizard\\ \hline
48  &bookcase  &49  &earthstar &  &\\ \hline
\end{tabular}
\end{table}

\begin{figure}[h]
\centering
  \footnotesize
\begin{tikzpicture}
\node[inner sep=0pt] (a1) at (0,0) {\includegraphics[width=8.6cm,height = 9cm]{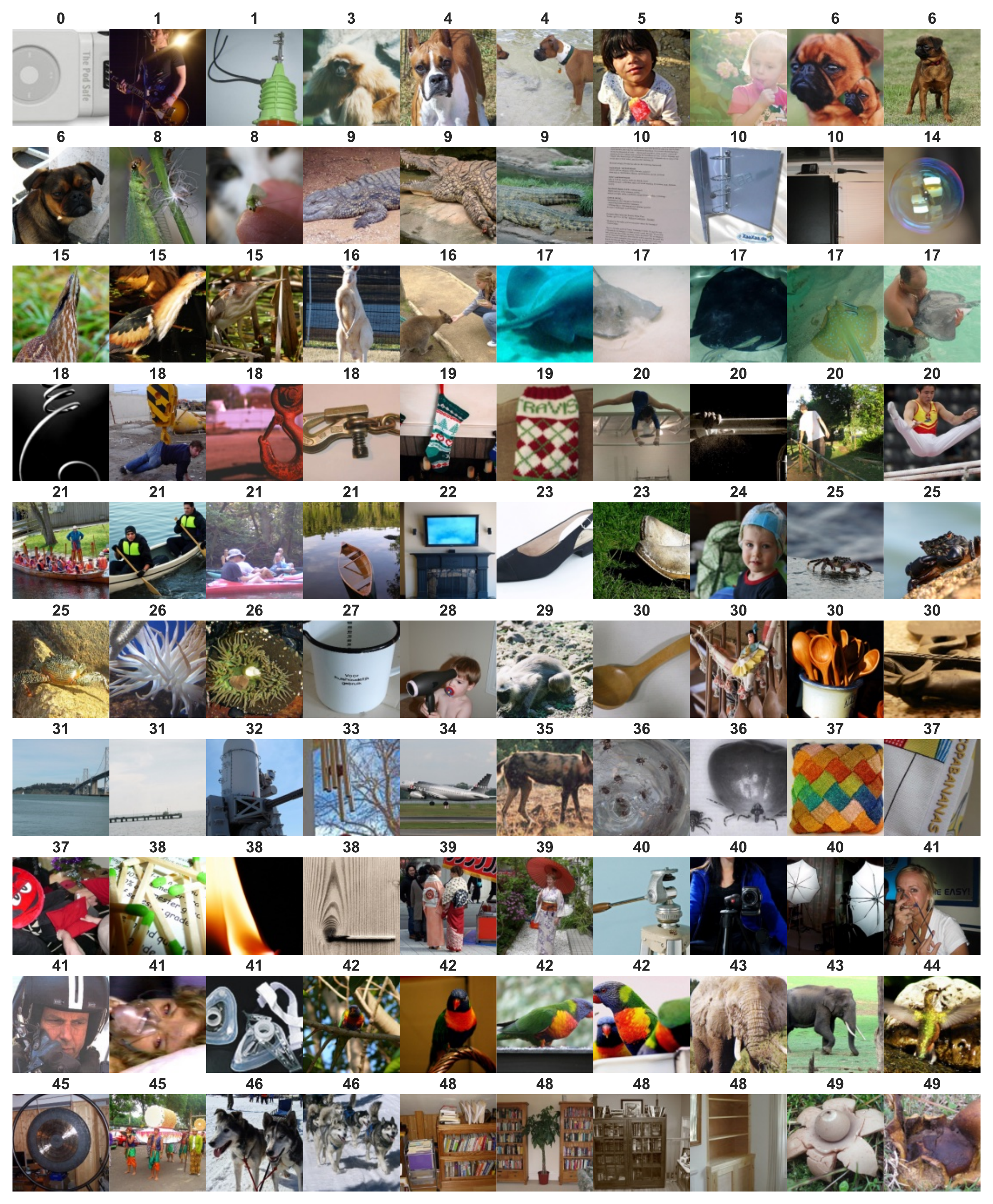}};
\node[]  at (-6,0.5) {\textbf{Random}};
\node[]  at (-6,-0) {\textbf{(\#classes:44)}};
\node[]  at (-6,-0.5) {\textbf{(Accuracy: 54.76\%)}};

\node[inner sep=0pt] (a2) at (0,-10) {\includegraphics[width=8.6cm,height = 9cm]{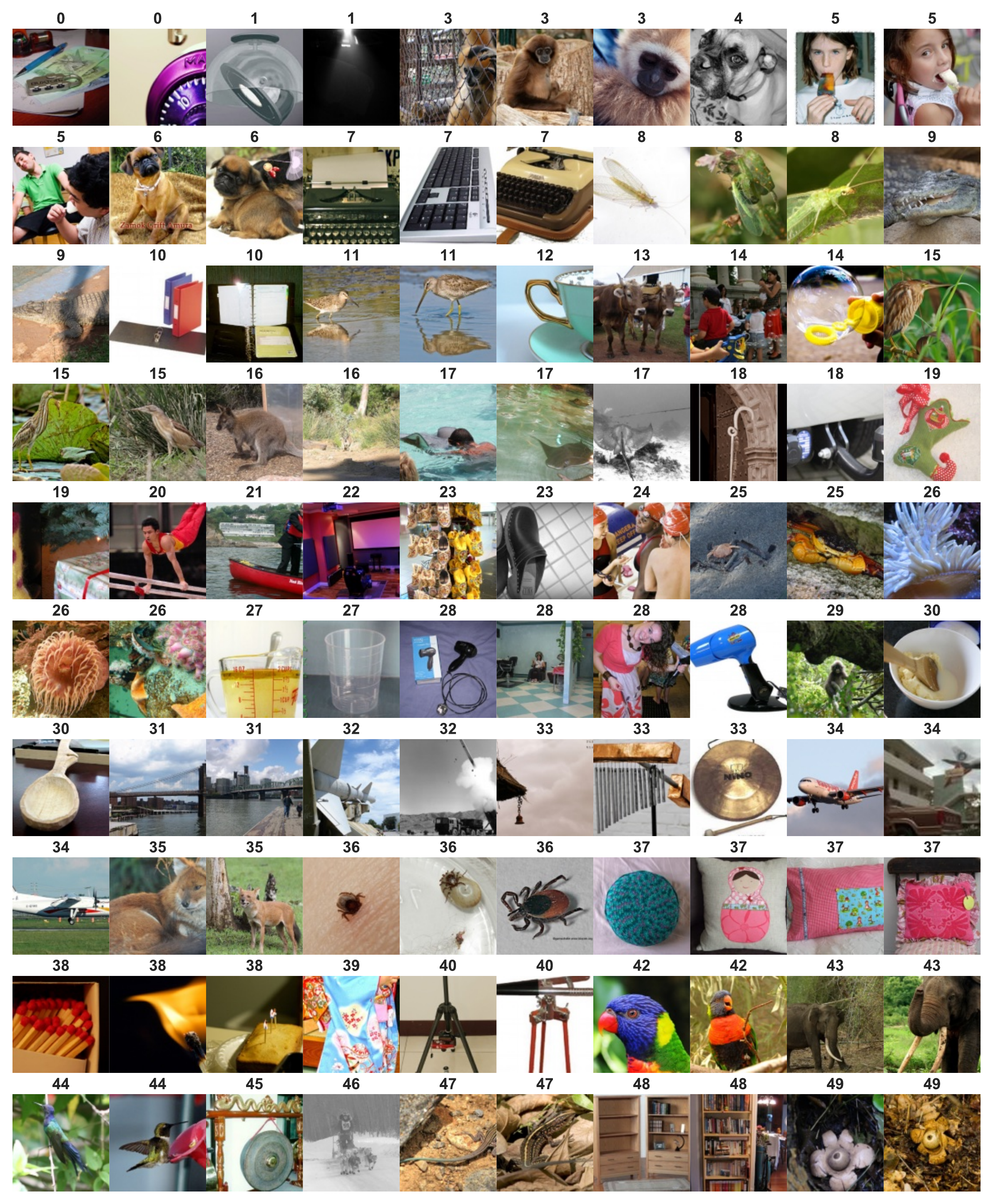}};
\node[]  at (-6,-7.5) {\textbf{K-means}};
\node[]  at (-6,-8) {\textbf{(\#classes:48)}};
\node[]  at (-6,-8.5) {\textbf{(Accuracy:68.26)}};
\end{tikzpicture}
\caption{100 ImageNet-50 samples selected by one trial of the  Random and K-means. The number of classes collected and the corresponding logistic regression prediction accuracy are reported within the brackets.}
\label{fig:imagenet-fig-1}
\end{figure}

\begin{figure}[h]
\centering
  \footnotesize
\begin{tikzpicture}
\node[inner sep=0pt] (a1) at (0,0) {\includegraphics[width=8.6cm,height = 9cm]{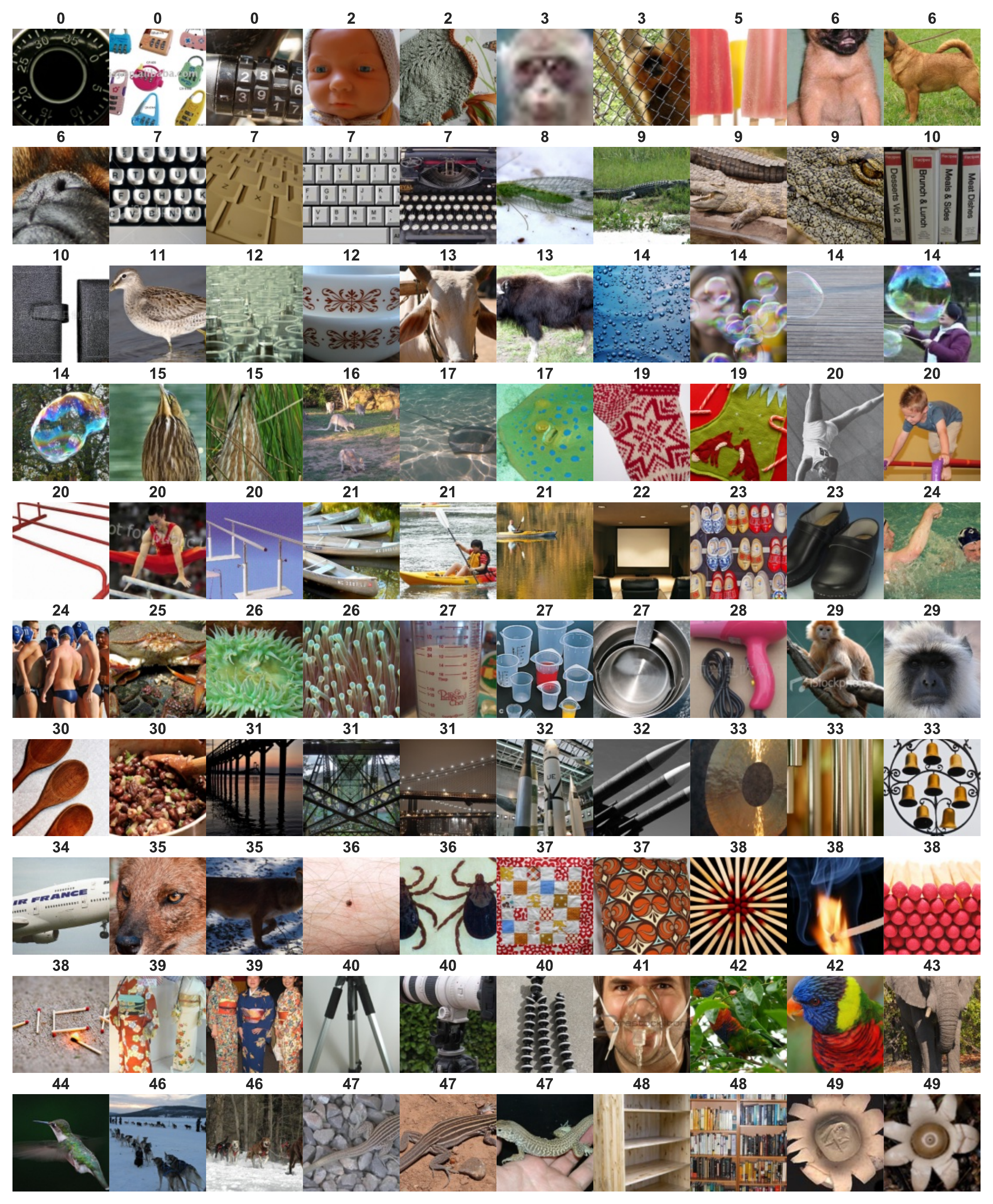}};
\node[]  at (-6,0.5) {\textbf{RRQR}};
\node[]  at (-6,-0) {\textbf{(\#classes:46)}};
\node[]  at (-6,-0.5) {\textbf{(Accuracy: 63.93\%)}};

\node[inner sep=0pt] (a2) at (0,-10) {\includegraphics[width=8.6cm,height = 9cm]{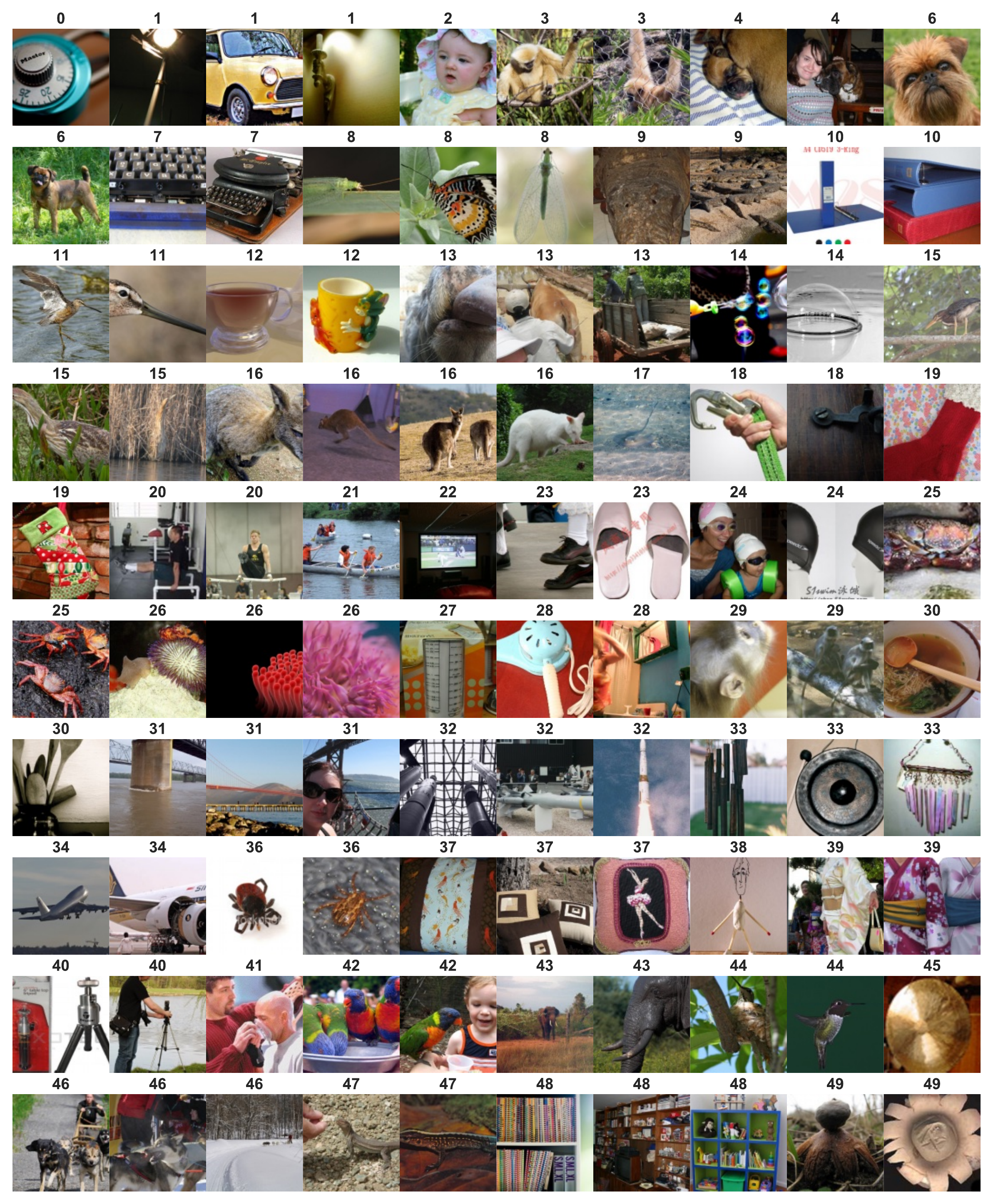}};
\node[]  at (-6,-7.5) {\textbf{MMD (dense)}};
\node[]  at (-6,-8) {\textbf{(\#classes:48)}};
\node[]  at (-6,-8.5) {\textbf{(Accuracy: 67.65\%)}};\end{tikzpicture}
\caption{100 ImageNet-50 samples selected by RRQR and MMD. The number of classes collected and the corresponding logistic regression prediction accuracy are reported within the brackets.}
\label{fig:imagenet-fig-2}
\end{figure}

\begin{figure}[h]
\centering
  \footnotesize
\begin{tikzpicture}
\node[inner sep=0pt] (a1) at (0,0) {\includegraphics[width=8.6cm,height = 9cm]{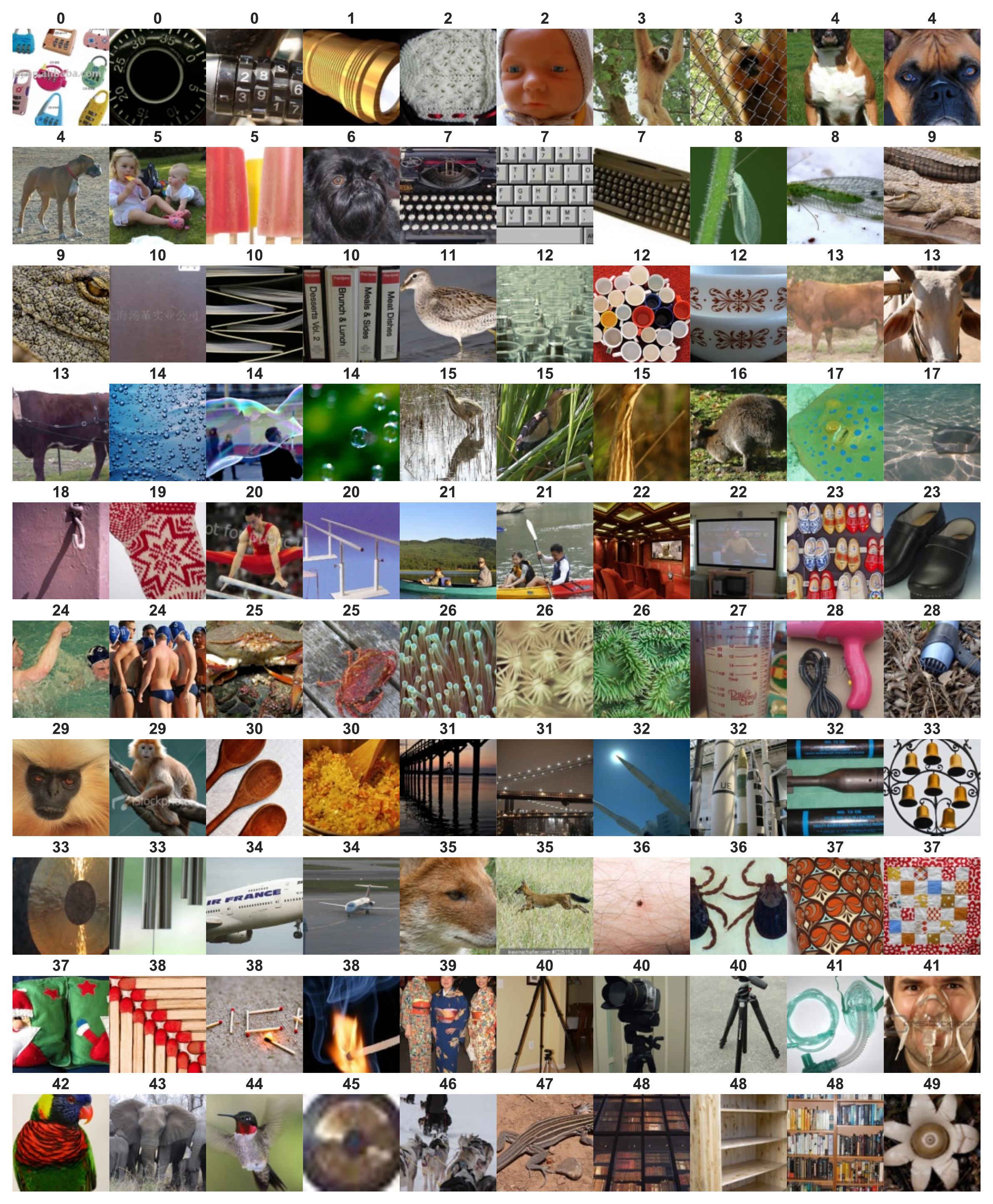}};
\node[]  at (-6.1,0) {\textbf{Regret-min (A-design)}};
\node[]  at (-6.1,-0.5) {\textbf{(\#classes: 50)}};
\node[]  at (-6.1,-1) {\textbf{(Accuracy: 70.95\%)}};
\end{tikzpicture}
\caption{100 ImageNet-50 samples selected by Regret-Min. The number of classes collected and the corresponding logistic regression prediction accuracy are reported within the brackets.}
\label{fig:imagenet-fig-3}
\end{figure}

\newpage
\bibliographystyle{siamplain}
\bibliography{references}

\end{document}